\documentclass[12pt]{article}

\usepackage[T1]{fontenc}
\usepackage[utf8]{inputenc}

\usepackage{amsmath}
\usepackage{amsfonts}
\usepackage{amssymb}
\usepackage{amsthm}

\usepackage{graphicx}
\usepackage{subcaption}
\usepackage{adjustbox}
\usepackage{multirow}
\usepackage{booktabs}
\usepackage{cellspace}

\usepackage{algorithm}
\usepackage{algorithmic}

\usepackage{natbib}
\usepackage{hyperref}

\usepackage{microtype}

\usepackage{pgfplots}
\pgfplotsset{compat=1.18}

\usepackage[most]{tcolorbox}

\usepackage{authblk}

\newtheorem{theorem}{Theorem}
\newtheorem{axiom}{Axiom}

\hypersetup{
	colorlinks=true,
	allcolors=black,
	pdfborder={0 0 0},
}

\begin{document}
	
	\title{Regime-Aware Portfolio Management via Retrieval-Augmented LLM-Guided Expert Switching}
	
\author[1]{Ahmad Asadi}
\author[1]{Reza Safabakhsh}
\affil[1]{Deep Learning Lab, Computer Engineering Department, Amirkabur University of Technology, Hafez Avenue, Tehran, Iran}	
	\maketitle
	
	\begin{abstract}
		Financial markets are inherently non-stationary, making the effectiveness of individual portfolio-management strategies highly dependent on changing market conditions. This work proposes a retrieval-augmented expert-switching framework that dynamically selects portfolio-management experts based on their historical performance under similar market situations. A dual-stream variational autoencoder represents asset-level and market-wide information, while a retrieval-based knowledge base stores historical situations and expert performance. During inference, an instruction-tuned LLM reasons over the retrieved evidence to identify the most appropriate expert rather than directly generating portfolio actions. We further establish a monotonicity property showing that adding a locally superior expert cannot degrade the switching mechanism's performance. Experiments across cryptocurrency, stock, and foreign-exchange markets show that the proposed selector achieves the highest cumulative return and Sharpe ratio among the evaluated selection strategies in all three markets. In the stock market, for example, cumulative return increases from 26\% for the best fixed expert to 34\%, while the Sharpe ratio improves from 0.74 to 0.96. Ablation results confirm the importance of both retrieval and LLM reasoning, while experiments with different expert-pool sizes demonstrate the value of complementary expertise. Overall, the findings support retrieval-grounded expert switching as an effective approach to adaptive portfolio management in non-stationary financial environments.
	\end{abstract}
	
	
	\vspace{1em}
	
	\section{Introduction}
Portfolio management in financial markets, particularly within the context of high-frequency trading, presents significant challenges due to the inherent non-stationarity of real-world data, abrupt price fluctuations driven by exogenous shocks and news events, often-irrational trader behavior, and intricate interdependencies among financial instruments \citep{cliff2020methods}. Conventional approaches typically assume a stationary data-generating process \citep{cakmak2006portfolio} and stable factor relationships over time \citep{jacobs2005portfolio}. In contrast, modern deep learning-based frameworks seek to address these complexities by scaling model capacity through a substantial increase in parameterization. However, training such large-scale models, along with ensuring their robustness and safety under extreme or anomalous market conditions, introduces additional practical constraints that hinder deployment. An alternative formalism for capturing such nonlinear dynamics is through regime-switching representations, wherein the data-generating process is characterized by distinct latent states. Nevertheless, this paradigm introduces further methodological challenges, as it requires simultaneous learning of both the regime identities and the corresponding conditional models, thereby compounding the overall training complexity \citep{liu2023regime}.

Deep learning has attracted attention in financial forecasting, with a wide range of architectures; including recurrent neural networks \citep{sharma2022portfolio}, convolutional designs \citep{popa2020convolutional}, temporal attention mechanisms \citep{asadi2025transformer}, and deep reinforcement learning frameworks \citep{taghian2022learning}, demonstrating varying degrees of success in predicting asset returns, volatility, and directional trends \citep{choudhary2025risk}. Among these, reinforcement learning (RL) offers a particularly compelling paradigm, as it naturally formalizes portfolio optimization as a sequential decision-making process. In particular, policy gradient and actor-critic methods have proven effective in navigating the intricate trade-offs between risk and return under evolving market conditions \citep{wei2025deep, jiang2024deep}. End-to-end deep learning systems are prone to overfitting the specific market regularities present in their training sets; however, they fail when confronted with unprecedented shocks or abrupt changes in volatility regimes.

To address the inherent variability of financial environments, mixture-of-experts (MoE) architectures have been proposed as a promising alternative to monolithic models \citep{wei2025deep}. In this framework, a set of specialized expert networks is maintained, each calibrated to a distinct region of the input space, thereby enabling adaptive responses to changing market conditions. In financial applications, MoE models typically employ supervised gating networks \citep{vats2024evolution} or meta-learning strategies \citep{masoudnia2014mixture} to dynamically combine expert outputs. These approaches have been shown to enhance predictive accuracy, particularly under controlled experimental settings.

Nevertheless, a critical limitation emerges when the system encounters market regimes that were not represented during training. Under such conditions, where the test-time distribution deviates substantially from the training distribution, the gating mechanism becomes unreliable, often assigning disproportionate weights to mismatched experts. This misallocation results in significant portfolio drawdowns, undermining the robustness that MoE architectures are designed to achieve. Consequently, the effectiveness of these models remains contingent upon the representativeness of the training data, and their generalization capacity to novel or extreme regimes remains an open and pressing research question.

The growing adoption of Large Language Models (LLMs) in finance has opened new avenues for integrating structured and unstructured data, such as news and reports, into trading systems. Their chain-of-thought reasoning capabilities enable semantically informed decision-making, supporting applications like sentiment analysis, automated reporting, and strategy generation \citep{kong2024large, feng2025deep}. Recent efforts have further combined LLMs with reinforcement learning or retrieval-augmented frameworks to incorporate historical context and domain knowledge at inference time \citep{liu2024revolutionising}.

We propose a situation-based expert switching framework, in which expert selection is framed as a sequential decision problem rather than an inherent property of a single predictive model. At each decision step, the system evaluates the current market state against a repository of historical situations and assesses the historical performance of each expert under analogous conditions. To render this selection process evidence-based and risk-aware, we construct a Statement of Performance (SoP) for each expert during an offline indexing phase. Each SoP encapsulates risk-adjusted profitability across a distribution of historical market contexts. During inference, these statements serve as empirical evidence, enabling the switcher to determine not only the prevailing market regime, but also which expert has demonstrated historical effectiveness within that regime.

We formalize this selection mechanism mathematically and establish a theoretical foundation through a monotonicity property: the introduction of a new expert that outperforms the existing pool within a given sub-domain guarantees non-degradation of overall system performance. This property is of practical significance, as it permits incremental expansion of the expert pool; enabling the integration of improved specialists without retraining a monolithic model from scratch. Importantly, this result distinguishes our approach from conventional strategies that pursue generalization primarily through increases in model scale. The switching mechanism is grounded in retrieval-based reasoning. Market situations are encoded via a dual-stream Variational Autoencoder, which extracts complementary representations from two sources: asset-level technical indicators and market-wide macroeconomic features. The resulting compact embeddings facilitate similarity-based retrieval of historical episodes characterized by analogous price dynamics. The retrieved evidence, combined with the expert SoPs, is supplied to the model switcher, which estimates the most appropriate expert to activate for subsequent portfolio construction.

Collectively, the proposed framework offers an alternative paradigm for adaptive portfolio management; one that emphasizes specialization across experts rather than scaling within a single model. The system begins with a predefined set of experts, incorporates a retrieval mechanism to ground expert selection in historical precedent, and integrates risk-aware performance statements to inform decisions based on realized outcomes. As new experts become available, they can be seamlessly incorporated into the pool, with the theoretical monotonicity result ensuring that such additions do not degrade performance, while the retrieval mechanism ensures their deployment under relevant conditions. Together, these components constitute a coherent and extensible framework for portfolio management in non-stationary financial environments.
	\section{Literature Review}

\subsection{Adaptive Portfolio Allocation}

Mean-variance portfolio optimization relies on assumptions of approximately Gaussian returns and stable asset correlations. Both assumptions become problematic in markets characterized by volatility clustering, structural breaks, and rapidly changing cross-asset dependencies \citep{nguyen2025advanced}. Machine learning methods relax these assumptions by learning nonlinear relationships from market observations \citep{feng2025deep}. Their empirical performance, however, varies substantially across market conditions, suggesting that the effectiveness of a portfolio model depends partly on the state of the environment in which it operates.

A common early strategy combined learned representations with conventional portfolio optimization. Predictive models extracted features from price observations, after which a separate optimizer determined portfolio weights. CNN- and recurrent-network-based predictors improved risk-adjusted performance in several such hybrid systems \citep{nguyen2025advanced}. Correlation-aware CNN-Transformer architectures subsequently incorporated cross-asset dependencies more explicitly and introduced cross-market transfer learning \citep{feng2025deep}. The allocation stage remained separate from representation learning in these models. Consequently, changes in market dynamics were handled primarily through model retraining or allocation heuristics rather than through an adaptive policy.

\subsection{Regime-Aware Reinforcement Learning}

Deep reinforcement learning removes the separate prediction-and-optimization stages by allowing an agent to map market observations directly to portfolio decisions. \citet{jiang2017cryptocurrency} showed that a CNN-based agent trained on raw price tensors could be applied across equity and cryptocurrency markets. Later architectures placed greater emphasis on temporal structure. The encoder-decoder models of \citet{taghian2022learning,taghian2023reinforcement}, for example, separate long-horizon temporal representation from short-horizon decision making. Other studies have incorporated behavioral assumptions, such as loss aversion and overconfidence, into actor-critic policies \citep{charkhestani2026behaviorally}, while \citet{alidousti2025novel} used a double-DQN formulation to reduce overestimation bias. Transformer-based reinforcement learning extends this line of work by modeling longer temporal dependencies than conventional recurrent architectures \citep{ren2025time}.

The main difficulty for a single DRL policy is its dependence on patterns that were useful during training. When the underlying market dynamics change, the same policy can continue to exploit relationships that are no longer present \citep{rezaei2025taxonomy}. Several studies therefore introduce multiple policies or explicitly model market regimes. Choudhary et al. \citep{choudhary2025risk} use multi-agent PPO with differentiated objectives based on log-return, Sharpe ratio \citep{sharpe2005journal}, and drawdown, producing more stable allocations than a single-agent formulation. Regime-switching systems instead associate different DRL agents with detected market conditions and report improvements in Sharpe and Sortino ratios for commodity and futures markets \citep{wang2025risk,sortino1994performance}. \citet{zhang2025regimefolio} follow a related strategy by defining volatility regimes and combining regime-specific ensemble forecasts with adaptive mean-variance allocation.

These studies differ mainly in how regime information reaches the trading policy. Some methods infer the regime implicitly through the learned state representation, whereas others introduce an explicit controller that determines which policy should be active. Explicit switching provides a natural mechanism for handling heterogeneous market conditions, but it also makes the quality of the switching signal critical. A regime classifier that reacts slowly can retain an inappropriate policy after a structural change; one that reacts too quickly can cause unnecessary switching between experts.

\subsection{Mixture-of-Experts and Expert Routing}

Mixture-of-experts (MoE) architectures address policy heterogeneity by maintaining several specialists and learning a routing mechanism that determines which specialist receives each state. In portfolio applications, Transformer-based spatio-temporal representations have been combined with learned gating networks to select among DRL experts, with improvements reported in cumulative and risk-adjusted returns \citep{wei2025deep}. Other systems use hypernetworks to modify expert parameters according to detected regime information, improving regime identification and predictive performance \citep{sun2025adaptive}. \citet{gu2025mixture} likewise use an MoE formulation to account for non-stationarity and stochastic volatility.

The routing mechanism is the key distinction among these models. Learned gates generally infer expert suitability from the current representation, whereas performance-based mechanisms can use recent outcomes as an indirect measure of suitability. Recent performance is attractive because it is inexpensive and directly related to the trading objective, but it is also noisy. A short sequence of favorable or unfavorable returns may reflect transient randomness rather than a persistent change in regime. During abrupt transitions, a router based primarily on recent performance can therefore favor an expert whose historical behavior is poorly matched to the newly emerging conditions.

Historical comparison offers a different basis for routing. Instead of asking which expert has performed best most recently, the router can ask which previously observed market states resemble the current state and examine expert behavior in those states. Such a mechanism shifts the routing problem from short-horizon performance estimation toward retrieval of comparable market episodes. The distinction is particularly relevant when the current regime has not appeared frequently in the recent training window but has recognizable historical counterparts.

\subsection{LLMs for Financial Decision Support}

Large language models have primarily entered financial systems through the processing of information that is difficult to represent directly in numerical market features. \citet{huang2025leveraging} combine ChatGPT-based stock selection with portfolio optimization, while \citet{yin2026complex} investigate chain-of-thought prompting with ChatGPT 4.0 for financial forecasting and portfolio optimization. Other studies incorporate LLM-derived sentiment, macroeconomic information, and textual summaries into DRL state representations, reporting improvements relative to quantitative baselines \citep{liu2024revolutionising,unnikrishnan2024financial}. Formulaic signals and textual descriptions generated by LLMs have similarly been investigated as higher-level features for prediction and interpretation.

A separate line of research assigns the LLM a coordination role. TradExpert uses several specialized LLM experts with different knowledge domains and combines their outputs through a general expert for prediction or ranking \citep{ding2024tradexpert}. \citet{liu2025llm} replace conventional neural routing with an LLM-based selector, using contextual reasoning and broader financial knowledge to distinguish among experts. TradingAgents extends the multi-agent formulation by assigning separate agents to analysis, risk management, and trading before synthesizing their outputs \citep{xiao2025tradingagents}. Similar multi-agent designs have also been explored for cryptocurrency portfolio management.

The empirical evidence places an important constraint on these uses of LLMs. Large-scale benchmarks show that LLM-based trading agents do not consistently outperform buy-and-hold strategies over multi-month evaluation periods \citep{chen2025stockbench}. The result suggests that language-based reasoning alone is insufficient for reliable portfolio control. LLMs appear more suitable for tasks involving information synthesis, contextual interpretation, and coordination than for generating precise portfolio actions at every decision step. For a portfolio system with quantitative experts, this distinction makes the LLM's role narrower but also easier to evaluate: the language model can contribute contextual information or assess relationships among retrieved market episodes without being responsible for the final portfolio weights.

\subsection{Research Gaps and Proposed Contributions}

The literature leaves three related issues unresolved. First, DRL and MoE systems commonly learn expert selection from the current state representation, a learned gate, or recent performance. These signals do not necessarily indicate whether an expert has previously succeeded under market conditions similar to those observed at the current decision point. Second, explicit regime-aware systems depend on regime definitions or classifiers whose errors can directly affect policy selection. Third, LLM-based financial systems have demonstrated value in information processing and coordination, but their use as direct trading agents remains inconsistent.

The proposed method places these three observations in a single routing mechanism. For a current market state, historical observations are retrieved according to their similarity to the present state. Expert performance in the retrieved episodes then provides the basis for selecting the portfolio policy. The retrieved episodes also provide an empirical reference for assessing how well the candidate experts behaved under comparable conditions. The LLM operates on this contextual information instead of producing portfolio weights directly; its role is to encode semantic information and assist with reasoning about regime similarity.

The resulting design treats expert selection as a historical matching problem. The central question is not simply which expert has performed well recently, but which expert has demonstrated suitable behavior in market states resembling the one currently observed. This distinction connects regime-aware reinforcement learning with MoE routing while giving the LLM a bounded decision-support role. It also provides a direct basis for evaluating expert selection through the retrieved historical episodes rather than relying solely on the internal state of a learned gating network.

	\section{Proposed Method}

\subsection{Overview}

The proposed system maintains a pool of specialized portfolio experts and selects among them at inference time using historical market states as reference points. The procedure has an offline and an online component. During offline processing, historical market windows are converted into latent state representations, each expert is evaluated over a fixed forward horizon, and the resulting state--performance pairs are stored in a vector database. During inference, the current market window is embedded in the same latent space, similar historical windows are retrieved, and the corresponding expert outcomes are used to rank the available policies.

The design separates the expensive construction of the historical performance index from the online selection step. Expert evaluation and indexing are performed once for each historical window. Online inference then requires only state encoding, nearest-neighbor retrieval, and expert scoring. The retrieved observations also provide an explicit record of why a particular expert is favored: its selection can be traced to the behavior of that expert under historically similar market conditions.

\subsection{System Architecture and Data Flow}
\label{subsec:architecture}

Figure~\ref{fig:arch} shows the complete data flow. The lower pipeline constructs the historical index, whereas the upper pipeline processes the current market state and selects an expert.

The market state is represented through two feature streams. The first contains asset-level technical information, including returns, momentum, volatility, and volume characteristics for each asset. The second describes market-wide conditions through aggregate trend, volatility, and liquidity measures. Separate encoders transform these streams into latent representations, which are subsequently combined into a single vector $h_t$. Using two streams preserves information about individual assets while retaining conditions that affect the market as a whole.

The offline pipeline associates each historical representation with the realized performance of every expert over a fixed forward horizon. The resulting records contain both numerical performance measures and a Statement of Performance (SoP), which summarizes the corresponding market episode and the observed behavior of each expert. These records form the evidence base used during online routing.

At time $t$, the current state is processed by the same encoders and represented by $h_t$. Approximate nearest-neighbor search identifies historical states close to $h_t$. The performance records attached to these neighbors provide estimates of how the experts have behaved under comparable conditions. The LLM receives these retrieved records together with the current market description and is used for contextual assessment and uncertainty calibration. It does not generate portfolio weights. The selected expert remains responsible for producing the portfolio allocation.

\begin{figure*}[ht!]
	\centering
	\includegraphics[width=\linewidth]{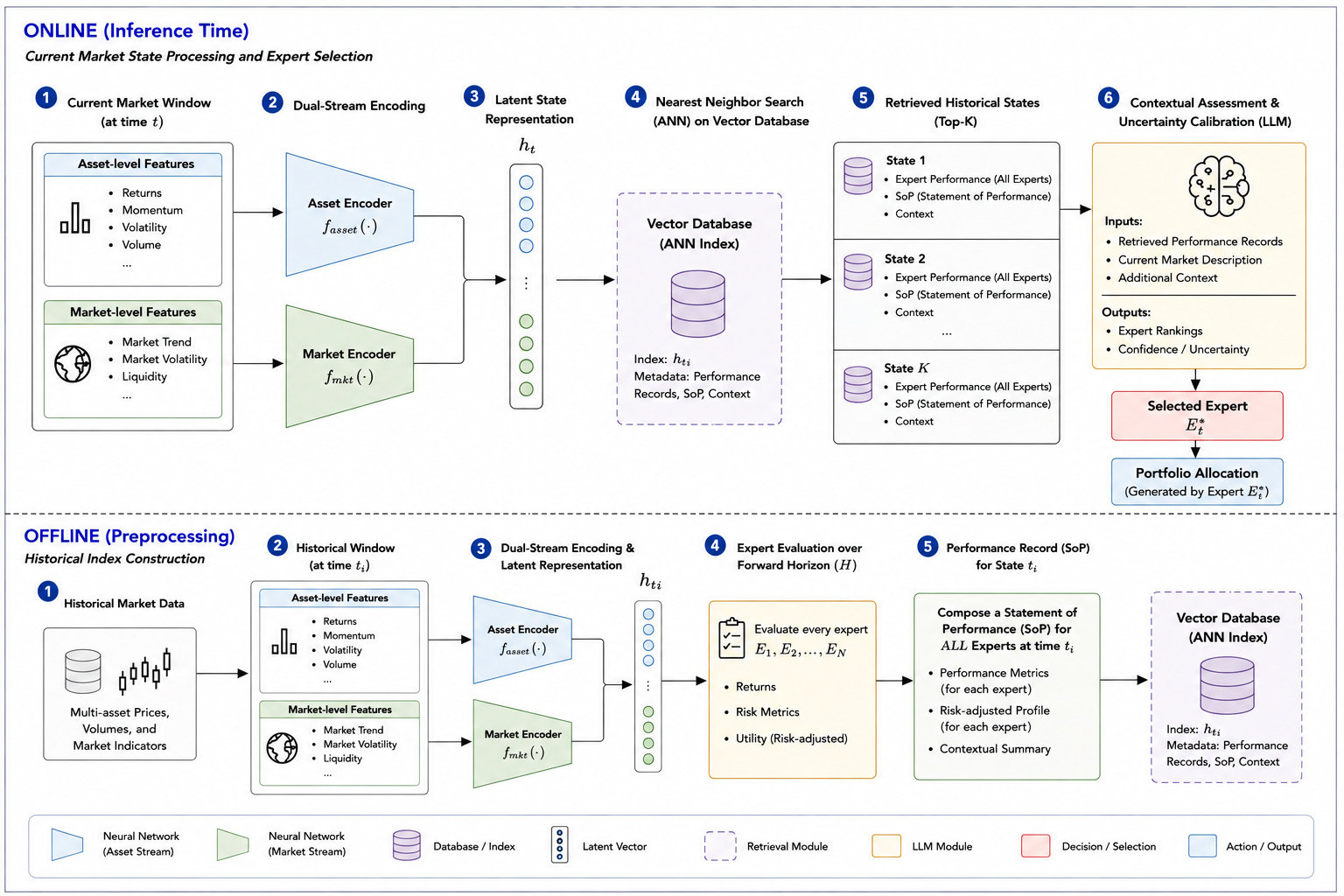}
	\caption{Overall architecture of the proposed retrieval-augmented adaptive portfolio management framework. The lower pipeline constructs the historical index by encoding market states and evaluating expert performance. The upper pipeline retrieves comparable historical states, combines their performance records with contextual information, and selects the expert used for portfolio allocation.}
	\label{fig:arch}
\end{figure*}

\subsection{Theoretical Framework: Expert Switching with Historical Risk Evaluation}
\label{sec:theoretical_framework}

The switching problem can be formulated over a distribution of latent market states. Let $\mathcal{S}$ denote the space of possible market states and $\mathbb{P}$ a probability measure over $\mathcal{S}$. A state may represent a combination of macroeconomic conditions, volatility characteristics, liquidity conditions, and structural market changes. The latent state is not directly observed at decision time; instead, the system constructs an observable representation from recent market data and uses historical observations to identify similar states.

Let the available experts be
$
\mathcal{E}_N={E_1,\ldots,E_N}.
$
Each expert $E_i$ maps an observed market representation to a portfolio weight vector. Its realized risk-adjusted utility in state $s$ is denoted by $U_i(s)$. This quantity is unknown when the portfolio decision is made and is estimated from historical observations.

The value of a switching policy $\pi_N:\mathcal{S}\rightarrow{1,\ldots,N}$ is

\begin{equation}
	V(\pi_N)=
	\int_{\mathcal{S}}
	U_{\pi_N(s)}(s),d\mathbb{P}(s).
\end{equation}

The purpose of adding a new specialist is therefore not simply to increase the number of available policies. The additional expert is useful only when its inclusion creates states in which the switching rule can identify a policy with at least as much utility as the best existing expert.

\subsubsection{Historical Evidence and Switching Rule}

For a current state $s$, let $\mathcal{H}_t(s)$ denote the historical episodes retrieved as similar to that state. The quantitative component of the router derives an expert-specific estimate from the outcomes observed in these episodes. The LLM receives the same evidence and produces a contextual risk/return assessment:

\begin{equation}
	\rho_i(s)=
	\mathcal{M}\left(\mathcal{H}_t(s),E_i\right).
\end{equation}

Here, $\rho_i(s)$ represents the assessment used by the switching procedure; it should not be interpreted as a direct observation of the unobservable quantity $U_i(s)$. The historical performance records provide the numerical basis for the estimate, while the LLM interprets the retrieved context and contributes to its calibration.

The active expert is selected according to the resulting scores. When a new specialist $E_{N+1}$ is introduced, the corresponding policy $\pi_{N+1}$ either retains the incumbent expert or activates the new specialist according to the switching margin specified below.

\subsubsection{Conditions for Non-Decreasing Performance}

The monotonicity result relies on three assumptions concerning estimation error, switching conservatism, and specialist behavior.

\begin{axiom}[Bounded Estimation Error / Empirical Alignment]
	The assessment $\rho_i(s)$ is uniformly close to the true utility. There exists $\epsilon>0$ such that
	\begin{equation}
		\sup_{s\in\mathcal{S}}
		|\rho_i(s)-U_i(s)|
		\leq\epsilon,
		\qquad
		\forall i\in{1,\ldots,N+1}.
	\end{equation}
	The practical implementation approximates this condition by grounding assessments in observed historical returns, drawdowns, and related performance measures rather than allowing unrestricted extrapolation.
\end{axiom}

\begin{axiom}[Conservative Switching Logic]
	At state $s$, the switcher evaluates all experts in the pool
	$\mathcal{E}_{N}={E_1,\ldots,E_{N}}$. Let the incumbent expert be
	$\pi_N(s)$ and let
	\begin{equation}
		e^*(s)=	\arg\max_{i\in{1,\ldots,N}}\rho_i(s)
	\end{equation}
	denote the expert with the highest estimated utility. The switcher activates
	$e^*(s)$ only when its estimated advantage over the incumbent exceeds the
	switching threshold $\tau\geq 2 \epsilon$:
	\begin{equation}
		\pi_{N}(s)=
		\begin{cases}
			e^*(s),
			&
			\rho_{e^*(s)}(s)-\rho_{\pi_N(s)}(s)>\tau,
			\\
			\pi_{N}(s) & \text{otherwise}.
		\end{cases}
	\end{equation}
	The threshold prevents the switcher from replacing the incumbent when the
	estimated advantage of the best available expert is within the expected
	estimation error.
\end{axiom}

\begin{axiom}[Specialist Dominance and Tail Protection]
	The new expert $E_{N+1}$ is intended for a subset $S_R\subset\mathcal{S}$ with $\mathbb{P}(S_R)>0$. Within this region, the specialist has a utility advantage of at least $\delta$, where
	$
	\delta>\tau+2\epsilon.
	$
	Thus,
	\begin{equation}
		U_{N+1}(s)
		\geq
		U_{\pi_N(s)}(s)+\delta,
		\qquad
		s\in S_R.
	\end{equation}
	Outside the specialist region, its utility is assumed unlimited.
\end{axiom}

\subsubsection{Operational Interpretation of the Conditions}

The three conditions translate into concrete constraints on the routing procedure. Empirical alignment limits the range of values supplied to the LLM; conservative switching requires a candidate expert to demonstrate a sufficiently large advantage before replacing the incumbent; and tail protection requires explicit examination of periods in which the specialist is expected to provide its main benefit.

Table~\ref{tab:llm_prompts} summarizes this correspondence. The prompt instructions are not themselves proofs of the mathematical assumptions. They are implementation heuristics intended to make the LLM's behavior more consistent with those assumptions. The prompt was refined using a validation set, as described in Section~\ref{subsec:prompt_calibration}.

\begin{table}[htbp]
	\centering
	\caption{Mapping of the switching assumptions to LLM evaluation constraints.}
	\label{tab:llm_prompts}
	\small
	
	\begin{tabular}{p{2cm}p{2.5cm}p{7cm}}
		\toprule
		\textbf{Condition} & 
		\textbf{Mathematical requirement} & 
		\textbf{Operational constraint} \\
		\midrule
		
		Axiom 1 &
		$\sup_s|\rho_i-U_i|\leq\epsilon$ &
		Restrict quantitative assessments to the empirical range of retrieved return and drawdown observations. Missing regimes are handled using historical baseline statistics.
		\\[0.2cm]
		\midrule
		
		Axiom 2 &
		$\tau\geq2\epsilon$ &
		Require a candidate expert to exceed the incumbent by a margin larger than the estimated uncertainty before recommending a switch.
		\\[0.2cm]
		\midrule
		
		Axiom 3 &
		$\delta>\tau+2\epsilon$ &
		Evaluate specialist performance in extreme conditions.
		\\
		
		\bottomrule
	\end{tabular}
	
\end{table}

\subsubsection{Monotonicity Result}

\begin{theorem}[Monotonic Expert-Pool Performance]
	Under Axioms 1--3, adding a specialist $E_{N+1}$ to the expert pool does not decrease the expected utility:
	$
	V(\pi_{N+1})\geq V(\pi_N).
	$
	Furthermore, the specialist produces a strict utility improvement on its target region $S_R$. If specialists are subsequently added so that their target regions cover the relevant state space, the resulting sequence of policy values is non-decreasing and converges to a limiting value bounded above by the optimal achievable utility.
\end{theorem}

\begin{proof}
	Consider first a state $s\notin S_R$. If the switching condition is not satisfied, the incumbent remains active and
	$
	\pi_{N+1}(s)=\pi_N(s),
	$
	so the utility is unchanged.
	
	If the switch occurs, then
	\[
	\rho_{N+1}(s)-\rho_{\pi_N(s)}(s)>\tau.
	\]
	By Axiom 1,
	\begin{align}
		U_{N+1}(s)&\geq\rho_{N+1}(s)-\epsilon,\\
		U_{\pi_N(s)}(s)&\leq\rho_{\pi_N(s)}(s)+\epsilon.
	\end{align}
	Consequently,
	\begin{equation}
		U_{N+1}(s)-U_{\pi_N(s)}(s) > \tau-2\epsilon.
	\end{equation}
	Since $\tau\geq2\epsilon$, the utility difference is non-negative. A switch outside the specialist's target region therefore cannot reduce utility under the stated error bound.
	
	Now consider $s\in S_R$. Axiom 3 gives
	\[
	U_{N+1}(s)-U_{\pi_N(s)}(s)\geq\delta.
	\]
	Using Axiom 1,
	\begin{equation}
		\rho_{N+1}(s)-\rho_{\pi_N(s)}(s) \geq \delta-2\epsilon.
	\end{equation}
	Because $\delta>\tau+2\epsilon$, it follows that
	\[
	\delta-2\epsilon>\tau.
	\]
	The specialist therefore satisfies the switching condition throughout $S_R$. Its contribution to the expected utility is at least
	\begin{equation}
		\Delta V \geq \int_{S_R}\delta\,d\mathbb{P}(s) = \delta\mathbb{P}(S_R)>0.
	\end{equation}
	
	Combining the two regions gives
	\[
	V(\pi_{N+1})\geq V(\pi_N).
	\]
	For a sequence of specialist additions, the same argument gives
	\[
	V(\pi_1)\leq V(\pi_2)\leq\cdots\leq V(\pi_K).
	\]
	The sequence is bounded above by the best achievable utility
	\[
	V^*= \int_{\mathcal S} \sup_E U_E(s)\,d\mathbb P(s),
	\]
	and therefore has a limit. If the added specialists cover the relevant state space and satisfy the stated dominance conditions, this limiting value approaches $V^*$.
	
\end{proof}

\subsection{Prompt Design and Uncertainty Calibration}
\label{subsec:prompt_calibration}

The practical implementation was designed around the limitations observed during validation. Early prompt versions permitted the LLM to infer returns outside the range represented in the retrieved historical records. These responses were frequently overconfident. The final prompt constrains numerical assessments to empirical observations and requires the model to fall back to historical baseline statistics when the retrieved evidence is sparse.

Five prompt iterations were evaluated on a validation set. The mean absolute estimation error decreased from approximately $18\%$ in the earlier versions to $11\%$ after the empirical-range constraints were introduced. The final prompt implements three specific restrictions:

\begin{itemize}
	\item \textbf{Empirical grounding:} risk and return assessments are derived from the historical MDD, realized information ratios, and other performance observations supplied in the retrieved context.
	\item \textbf{Conservative margining:} a newly introduced expert receives a risk penalty before comparison with the incumbent, and the candidate is not preferred unless its estimated advantage exceeds the specified noise margin.
	\item \textbf{Tail-event analysis:} the evaluator examines extreme drawdowns, liquidity squeezes, and regime transitions separately from ordinary observations. A specialist's advantage must be localized to its intended regime without creating a severe deterioration elsewhere.
\end{itemize}

The resulting system prompt is shown in Prompt Specification~\ref{spec:risk_prompt}~. Its output is restricted to a structured JSON representation so that the LLM assessment can be incorporated into the numerical routing procedure without requiring free-form text parsing.

\begin{tcolorbox}[
	enhanced,
	breakable,
	colback=gray!3,
	colframe=black!60,
	boxrule=0.5pt,
	arc=1pt,
	left=6pt,
	right=6pt,
	top=6pt,
	bottom=6pt,
	title={System prompt used by the LLM risk evaluator. The instructions constrain numerical assessments to retrieved empirical evidence and require explicit treatment of switching margins and tail conditions.},
	fonttitle=\bfseries
	]
	\small
	\textbf{ROLE.}
	You are a financial risk analyzer operating within a dynamic
	Mixture-of-Experts (MoE) portfolio system.
	
	\medskip
	\textbf{EMPIRICAL ALIGNMENT.}
	\begin{itemize}
		\item Base risk-adjusted assessments on the observed historical
		Maximum Drawdown (MDD) and realized Information Ratios provided
		in the context.
		\item Do not extrapolate returns beyond the empirical observations.
		\item If evidence for a regime is sparse, use the historical
		baseline mean.
	\end{itemize}
	
	\medskip
	\textbf{CONSERVATIVE MARGINING.}
	\begin{itemize}
		\item When evaluating a new candidate expert $E^{(N+1)}$, apply
		the specified risk penalty buffer $\tau$.
		\item Do not prefer the candidate unless its estimated advantage
		over the incumbent exceeds the estimated noise level by the
		required margin.
	\end{itemize}
	
	\medskip
	\textbf{TAIL-EVENT SEPARATION.}
	\begin{itemize}
		\item Examine extreme drawdowns, liquidity squeezes, and regime shifts.
		\item If a candidate is a specialist, identify the regime in which
		its advantage is observed.
		\item Verify that performance outside this regime does not fall below
		the weakest baseline expert.
	\end{itemize}
	
	\medskip
	\textbf{OUTPUT FORMAT.}
	The evaluator must return a structured record:
	\begin{verbatim}
		{
			"expert_id": "E_N+1",
			"calculated_rho": <float>,
			"confidence_bound_epsilon": <float>,
			"regime_separation_margin_tau": <float>,
			"tail_optimal_flag": <boolean>
		}
	\end{verbatim}
	
\end{tcolorbox}
\label{spec:risk_prompt}

\subsection{Market State Decomposition and Feature Modeling}

At each time step $t$, the market state is represented by two rolling-window feature tensors of length $L$. The first describes asset-specific technical behavior:
$
S_t^{\mathrm{tech}}\in \mathbb{R}^{L\times A\times F_1},
$
where $A$ is the number of assets and $F_1$ is the number of asset-level features. The feature set includes log-returns, RSI, MACD, ATR, realized volatility, and volume changes.

The second representation describes market-wide conditions:
$
S_t^{\mathrm{mkt}}\in \mathbb{R}^{L\times F_2}.
$
It contains aggregate trend indicators, volatility-regime measures, and liquidity variables.

The two representations retain different types of information. Asset-level features describe relative behavior among individual instruments, whereas market-wide features describe conditions shared across the portfolio. Keeping the streams separate before encoding avoids forcing both types of information into a single input representation. The resulting latent representation is used for historical state matching, where preserving both sources of information is important for identifying comparable episodes \citep{g2024hedge,asadi2025transformer}.

All features are standardized using rolling z-score normalization. The rolling procedure limits the effect of scale changes and reduces the influence of heteroskedasticity on the state representation.

\subsection{Dual-Stream Transformer-VAE Embedding}
\label{subsec:embedding}

Each feature stream is encoded with a pretrained Transformer-VAE from \citet{asadi2025transformer}. For
$x\in{\mathrm{tech},\mathrm{mkt}}$, the encoder maps $S_t^x$ to the Gaussian posterior

\begin{equation}
	q_{\phi_x}(z_x\mid S_t^x) = \mathcal{N} \left(\mu_x(S_t^x), \Sigma_x(S_t^x)\right),
\end{equation}

where the mean and diagonal covariance are generated by the temporal Transformer encoder. A latent representation is sampled using

\begin{equation}
	z_x = \mu_x+\epsilon\odot\sigma_x, \qquad \epsilon\sim\mathcal{N}(0,I).
\end{equation}

During pretraining, the sampled representation is passed through the decoder to reconstruct the input. After pretraining, the decoder is removed and the encoder is retained for market-state representation.

The two latent vectors are concatenated:

\begin{equation}
	h_t = [z_t^{\mathrm{tech}}\mid z_t^{\mathrm{mkt}}] \in\mathbb{R}^{D}.
\end{equation}

The same representation function is used in both phases of the system. Historical states are indexed using $h_\tau$, and the current state is represented by $h_t$ before nearest-neighbor retrieval. Using the same encoder in both cases ensures that the distance used during retrieval is defined in a common latent space.

\subsection{Offline Indexing}

Algorithm~\ref{alg:indexing} describes the construction of the historical index. For each historical time step $\tau$, the system constructs the two market-state representations and computes $h_\tau$. Every expert is then evaluated from that state over the forward horizon $H$. The resulting record contains cumulative return, Sharpe ratio, and maximum drawdown. A Statement of Performance (SoP) is generated from the same episode and records the market context, observed expert behavior, and relevant uncertainty factors.

An index entry is therefore represented as

\begin{equation}
	\mathcal{D}_\tau = \left(h_\tau, {R_\tau^{e}}, \mathrm{SoP}_\tau^{e} \right), \qquad e\in{1,\ldots,E}.
\end{equation}

\begin{algorithm}[t]
	\caption{Offline Regime Indexing and Expert Performance Grounding}
	\label{alg:indexing}
	\begin{algorithmic}[1]
		\REQUIRE Historical market data, expert set $\mathcal{E}={E^{(e)}}*{e=1}^{E}$, horizon $H$
		\FOR{each historical time step $\tau$}
		\STATE Construct $S*\tau^{\mathrm{tech}}$, $S_\tau^{\mathrm{mkt}}$
		\STATE Compute embedding $h_\tau$ using the dual-stream T-VAE
		\FOR{each expert $E^{(e)}$}
		\STATE Execute expert policy from $S_\tau$
		\STATE Evaluate realized performance $R_\tau^{(e)}$ over horizon $H$
		\ENDFOR
		\STATE Generate Statement of Performance $\mathrm{SoP}*\tau$
		\STATE Store $(h*\tau,{R_\tau^{(e)}},{\mathrm{SoP}_\tau^{(e)}})$
		\ENDFOR
	\end{algorithmic}
\end{algorithm}

\subsection{Online Retrieval and Portfolio Decision}

Algorithm~\ref{alg:inference} summarizes the online procedure. The current market window is first transformed into $S_t^{\mathrm{tech}}$ and $S_t^{\mathrm{mkt}}$ and encoded as $h_t$. The system then retrieves the $K$ historical states closest to $h_t$:

\begin{equation}
	\mathcal{N}_t = \mathrm{KNN} \left(h_t, {h_\tau}, K\right).
\end{equation}

The performance of each expert is estimated from the retrieved episodes. Let $\mathrm{sim}(h_t,h_\tau)$ denote the similarity between the current and retrieved states. The normalized retrieval weight is

\begin{equation}
	w_\tau = \frac{\mathrm{sim}(h_t,h_\tau)}{\sum_{j\in\mathcal{N}_t} \mathrm{sim}(h_t,h_j)},
\end{equation}

and the similarity-weighted performance estimate for expert $e$ is

\begin{equation}
	\hat{R}_t^{(e)} = \sum_{\tau\in\mathcal{N}*t} w*\tau R_\tau^{(e)}.
\end{equation}

The retrieved SoPs, performance summaries, and current market description are also supplied to the LLM. The model decides on selecting the best expert model among the existing ones based on the current market situation and the risk-return behavior profile of the experts embedded into retrieved SoPs.

\begin{algorithm}[t]
	\caption{Online Retrieval and Portfolio Decision-Making}
	\label{alg:inference}
	\begin{algorithmic}[1]
		\REQUIRE Current market data, vector database, expert set ${E^{(e)}}$
		\STATE Construct $S_t^{\mathrm{tech}}$, $S_t^{\mathrm{mkt}}$
		\STATE Compute embedding $h_t$ using the dual-stream T-VAE
		\STATE Retrieve $K$ nearest historical regimes $\mathcal{N}_t$
		\FOR{each expert $E^{(e)}$}
		\STATE Compute similarity-weighted performance $\hat{R}_t^{(e)}$
		\STATE Compute weighted uncertainty $\sigma_t^{(e)}$
		\STATE Query the LLM using retrieved SoPs and market context to get $e^*$
		\ENDFOR
		\STATE Generate portfolio weights $w_t$ using $E^{(e^*)}$
		\RETURN $w_t$
	\end{algorithmic}
\end{algorithm}

	\section{Experimental Results}
\label{sec:experiments}

We evaluate the proposed retrieval-augmented expert-switching framework independently on three financial market classes: cryptocurrency, stocks, and foreign exchange. The effect of important components in the proposed framework are investigated and the performance of the proposed model is compared with the baseline portfolio management methods.

Each experiment is conducted once for each market using the complete 30-symbol panel of that market. Consequently, the reported values are point estimates for a single experimental run rather than averages over multiple random seeds. This design makes the market-level comparison explicit and avoids conflating differences between market classes with differences between repeated runs. 

\subsection{Experimental Setup}
\label{subsec:setup}

\paragraph{Dataset}
\label{sec:dataset}

We construct a multi-market daily OHLCV dataset covering three market classes: cryptocurrencies, stocks, and foreign-exchange pairs. For each market, top 30 assets regarding their liquidity are selected along with the data for the market indices. Since, the cryptocurrencies market has no standard market index, we computed two indices to illustrate overall market situation. 

The cryptocurrency panel contains the assets listed in Table~\ref{tbl:data}. The stock panel and foreign-exchange panel are summarized in Tables~\ref{tbl:data2} and~\ref{tbl:data3}, respectively. Individual availability windows differ because the instruments were listed or became available at different dates. The resulting panels therefore contain unequal numbers of historical observations per symbol.

All experiments use a rolling input window of $L=22$ trading days and a holding horizon of $H=5$ days. The same input and decision horizons are used across all three markets so that the switching mechanism is evaluated under a common experimental protocol rather than being separately tuned to each market.

\begin{table}[htbp]
	\centering
	\caption{Crypto market data summary.}
	\scriptsize
	\resizebox{\linewidth}{!}{%
			\begin{tabular}{llll}
			\toprule
			Market & Symbol & Range & Samples\\
			\midrule
			Crypto & VET\_USD & 2018-08-03-2026-05-31 & 2859\\
			Crypto & BCH\_USD & 2017-11-09-2026-05-31 & 3126\\
			Crypto & ETC\_USD & 2017-11-09-2026-05-31 & 3126\\
			Crypto & UNI\_USD & 2019-10-21-2025-04-17 & 2006\\
			Crypto & SHIB\_USD & 2020-08-01-2026-05-31 & 2089\\
			Crypto & LTC\_USD & 2014-09-17-2026-05-31 & 4275\\
			Crypto & ICP\_USD & 2021-05-10-2026-05-31 & 1848\\
			Crypto & ARB\_USD & 2017-11-09-2026-05-31 & 3108\\
			Crypto & ADA\_USD & 2017-11-09-2026-05-31 & 3126\\
			Crypto & SOL\_USD & 2020-04-10-2026-05-31 & 2243\\
			Crypto & FIL\_USD & 2017-12-13-2026-05-31 & 3092\\
			Crypto & BTC\_USD & 2014-09-17-2026-05-31 & 4275\\
			Crypto & TON\_USD & 2020-08-29-2026-05-31 & 2102\\
			Crypto & EGLD\_USD & 2020-09-04-2026-05-31 & 2096\\
			Crypto & MATIC\_USD & 2019-04-28-2025-03-24 & 2158\\
			Crypto & DOT\_USD & 2020-08-20-2026-05-31 & 2111\\
			Crypto & XLM\_USD & 2017-11-09-2026-05-31 & 3126\\
			Crypto & ETH\_USD & 2017-11-09-2026-05-31 & 3126\\
			Crypto & RNDR\_USD & 2020-06-11-2024-07-21 & 1502\\
			Crypto & MKR\_USD & 2017-11-20-2026-05-31 & 3115\\
			Crypto & HBAR\_USD & 2019-09-17-2026-05-31 & 2449\\
			Crypto & AVAX\_USD & 2020-07-13-2026-05-31 & 2080\\
			Crypto & ATOM\_USD & 2019-03-14-2026-05-31 & 2636\\
			Crypto & DOGE\_USD & 2017-11-09-2026-05-31 & 3126\\
			Crypto & LINK\_USD & 2017-11-09-2026-05-31 & 3126\\
			Crypto & TRX\_USD & 2017-11-09-2026-05-31 & 3126\\
			Crypto & NEAR\_USD & 2020-10-14-2026-05-31 & 2056\\
			Crypto & APT\_USD & 2021-11-20-2025-06-24 & 1311\\
			Crypto & XRP\_USD & 2017-11-09-2026-05-31 & 3126\\
			Crypto & BNB\_USD & 2017-11-09-2026-05-31 & 3126\\
			\bottomrule
	\end{tabular}}
	\label{tbl:data}
\end{table}

\begin{table}[htbp]
	\centering
	\caption{Stock market data summary.}
	\scriptsize
	\resizebox{\linewidth}{!}{%
		\begin{tabular}{llll}
			\toprule
			Market & Symbol & Range & Samples\\
			\midrule
			Stock & TMO & 2000-01-03-2026-05-29 & 6641\\
			Stock & MCD & 2000-01-03-2026-05-29 & 6641\\
			Stock & NVDA & 2000-01-03-2026-05-29 & 6641\\
			Stock & DIS & 2000-01-03-2026-05-29 & 6641\\
			Stock & TSLA & 2010-06-29-2026-05-29 & 4004\\
			Stock & CVX & 2000-01-03-2026-05-29 & 6641\\
			Stock & XOM & 2000-01-03-2026-05-29 & 6641\\
			Stock & ACN & 2001-07-19-2026-05-29 & 6252\\
			Stock & V & 2008-03-19-2026-05-29 & 4578\\
			Stock & WMT & 2000-01-03-2026-05-29 & 6641\\
			Stock & KO & 2000-01-03-2026-05-29 & 6641\\
			Stock & LIN & 2000-01-03-2026-05-29 & 6641\\
			Stock & HD & 2000-01-03-2026-05-29 & 6641\\
			Stock & GOOGL & 2004-08-19-2026-05-29 & 5479\\
			Stock & CRM & 2004-06-23-2026-05-29 & 5519\\
			Stock & UNH & 2000-01-03-2026-05-29 & 6641\\
			Stock & BAC & 2000-01-03-2026-05-29 & 6641\\
			Stock & COST & 2000-01-03-2026-05-29 & 6641\\
			Stock & MA & 2006-05-25-2026-05-29 & 5034\\
			Stock & PG & 2000-01-03-2026-05-29 & 6641\\
			Stock & AAPL & 2000-01-03-2026-05-29 & 6641\\
			Stock & JPM & 2000-01-03-2026-05-29 & 6641\\
			Stock & JNJ & 2000-01-03-2026-05-29 & 6641\\
			Stock & MSFT & 2000-01-03-2026-05-29 & 6641\\
			Stock & MRK & 2000-01-03-2026-05-29 & 6641\\
			Stock & META & 2012-05-18-2026-05-29 & 3527\\
			Stock & AMZN & 2000-01-03-2026-05-29 & 6641\\
			Stock & ABBV & 2013-01-02-2026-05-29 & 3372\\
			Stock & BRK-B & 2000-01-03-2026-05-29 & 6641\\
			Stock & PEP & 2000-01-03-2026-05-29 & 6641\\
			\bottomrule
	\end{tabular}}
	\label{tbl:data2}
\end{table}

\begin{table}[htbp]
	\centering
	\caption{Foreign-exchange market data summary.}
	\scriptsize
	\resizebox{\linewidth}{!}{%
		\begin{tabular}{llll}
			\toprule
			Market & Symbol & Range & Samples\\
			\midrule
			Forex & USDCAD & 2003-09-17-2026-05-29 & 5904\\
			Forex & EURAUD & 2003-12-01-2026-05-29 & 5853\\
			Forex & EURCHF & 2003-01-23-2026-05-29 & 6059\\
			Forex & GBPCHF & 2003-10-01-2026-05-29 & 5894\\
			Forex & CHFJPY & 2003-12-01-2026-05-29 & 5854\\
			Forex & NZDCHF & 2003-12-01-2026-05-29 & 5825\\
			Forex & AUDCAD & 2003-12-01-2026-05-29 & 5856\\
			Forex & NZDCAD & 2003-12-01-2026-05-29 & 5852\\
			Forex & GBPUSD & 2003-12-01-2026-05-29 & 5848\\
			Forex & AUDUSD & 2006-05-16-2026-05-29 & 5212\\
			Forex & GBPAUD & 2003-06-03-2026-05-29 & 5982\\
			Forex & USDJPY & 2000-01-03-2026-05-29 & 6856\\
			Forex & GBPJPY & 2003-12-01-2026-05-29 & 5852\\
			Forex & NZDJPY & 2003-12-01-2026-05-29 & 5852\\
			Forex & AUDNZD & 2003-12-01-2026-05-29 & 5855\\
			Forex & EURGBP & 2000-01-03-2026-05-29 & 6876\\
			Forex & USDCHF & 2003-09-17-2026-05-29 & 5902\\
			Forex & EURCAD & 2003-12-01-2026-05-29 & 5839\\
			Forex & AUDJPY & 2003-12-01-2026-05-29 & 5855\\
			Forex & GBPCAD & 2003-12-01-2026-05-29 & 5853\\
			Forex & USDTRY & 2005-01-03-2026-05-29 & 5565\\
			Forex & AUDCHF & 2003-12-01-2026-05-29 & 5827\\
			Forex & EURTRY & 2005-01-03-2026-05-29 & 5567\\
			Forex & EURSEK & 2000-01-03-2026-05-29 & 6853\\
			Forex & GBPNZD & 2003-06-03-2026-05-29 & 5981\\
			Forex & EURUSD & 2003-12-01-2026-05-29 & 5836\\
			Forex & NZDUSD & 2003-12-01-2026-05-29 & 5837\\
			Forex & EURJPY & 2003-01-23-2026-05-29 & 6061\\
			Forex & EURNZD & 2003-12-01-2026-05-29 & 5853\\
			Forex & CADJPY & 2004-08-20-2026-05-29 & 5665\\
			\bottomrule
	\end{tabular}}
	\label{tbl:data3}
\end{table}

\paragraph{Market-level visualizations}

Figure~\ref{fig:ohlc_crypto} illustrates the OHLC price histories of all 30 cryptocurrency symbols. The corresponding visualizations for stocks and foreign exchange are shown in Figures~\ref{fig:ohlc_stock} and~\ref{fig:ohlc_forex}. These plots are included to make the cross-market heterogeneity explicit and to document the different price scales, trend structures, and volatility regimes encountered by the feature extractor and retrieval mechanism.

\begin{figure*}[htbp]
	\centering
	\scriptsize
	\begin{subfigure}[b]{0.155\textwidth}\includegraphics[width=\linewidth]{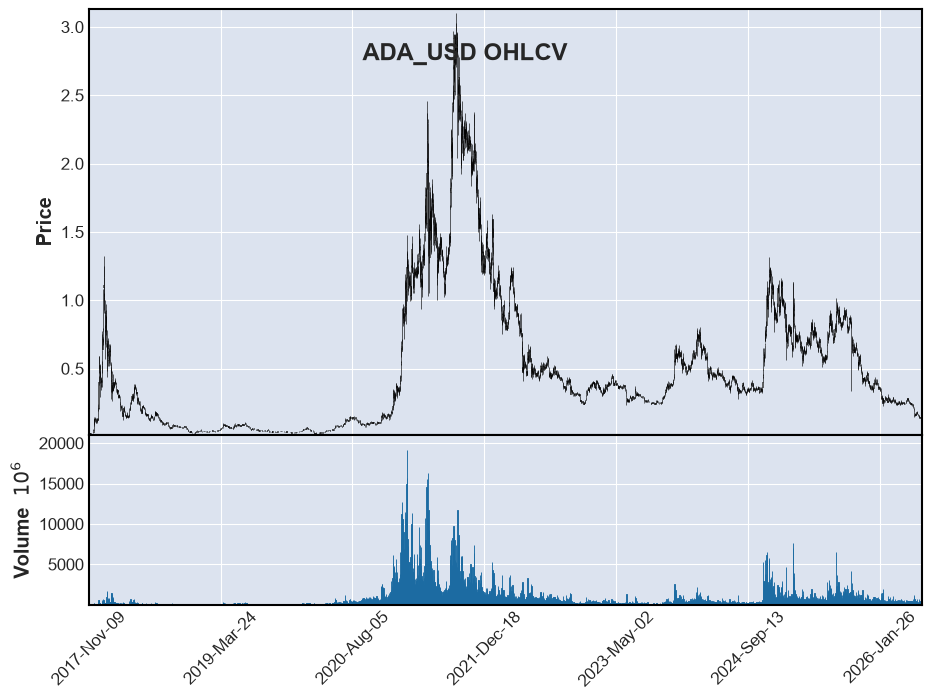}\caption{ADA}\end{subfigure}\hfill
	\begin{subfigure}[b]{0.155\textwidth}\includegraphics[width=\linewidth]{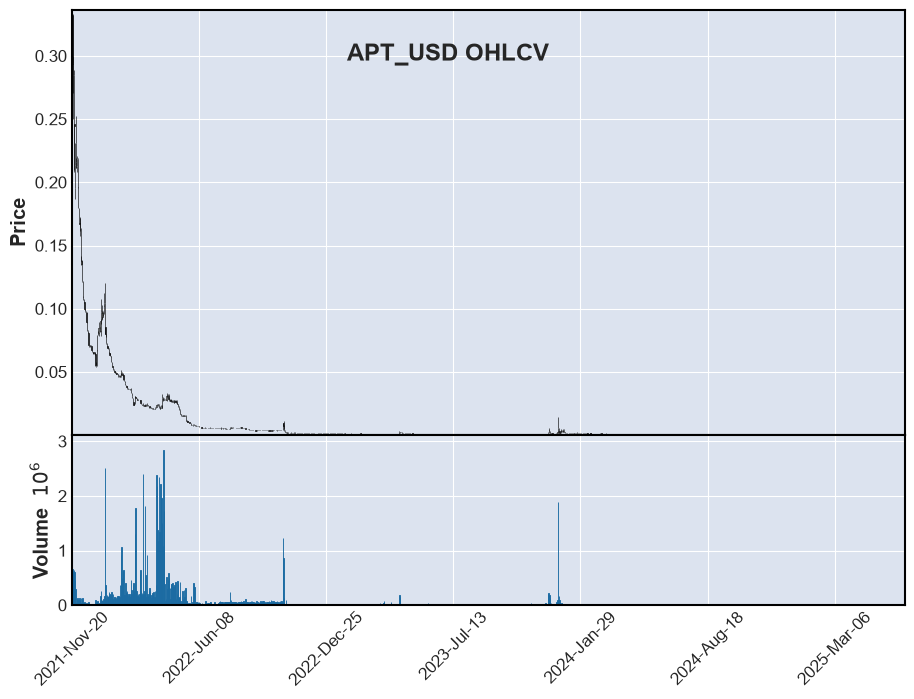}\caption{APT}\end{subfigure}\hfill
	\begin{subfigure}[b]{0.155\textwidth}\includegraphics[width=\linewidth]{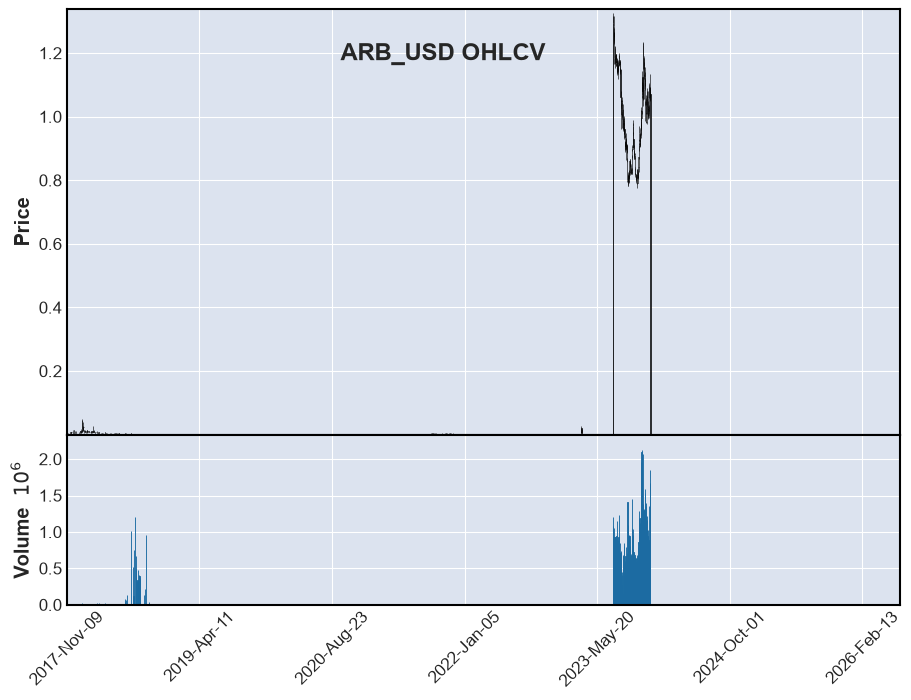}\caption{ARB}\end{subfigure}\hfill
	\begin{subfigure}[b]{0.155\textwidth}\includegraphics[width=\linewidth]{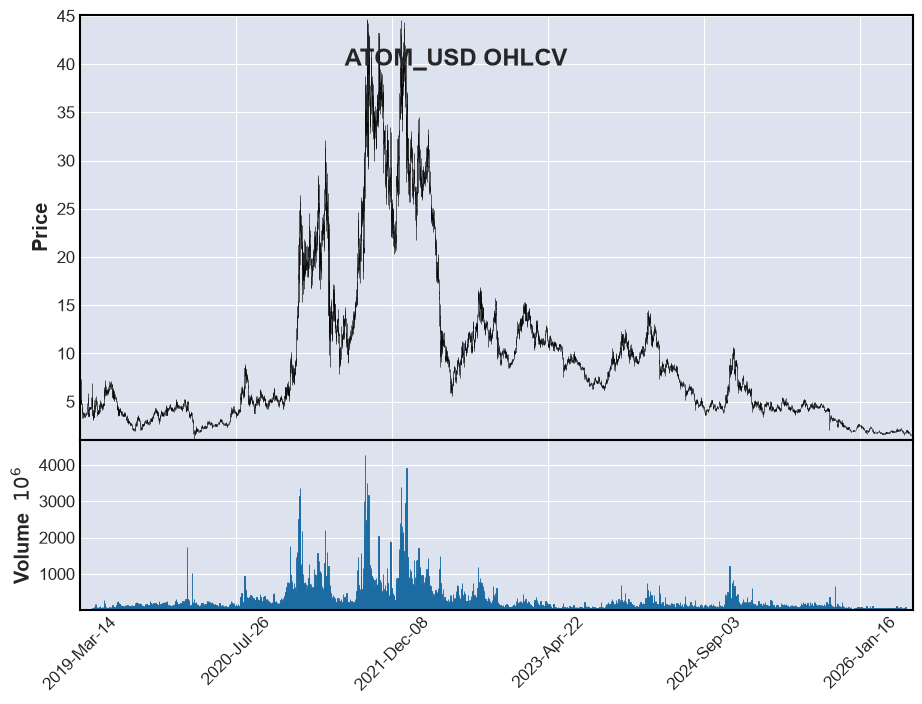}\caption{ATOM}\end{subfigure}\hfill
	\begin{subfigure}[b]{0.155\textwidth}\includegraphics[width=\linewidth]{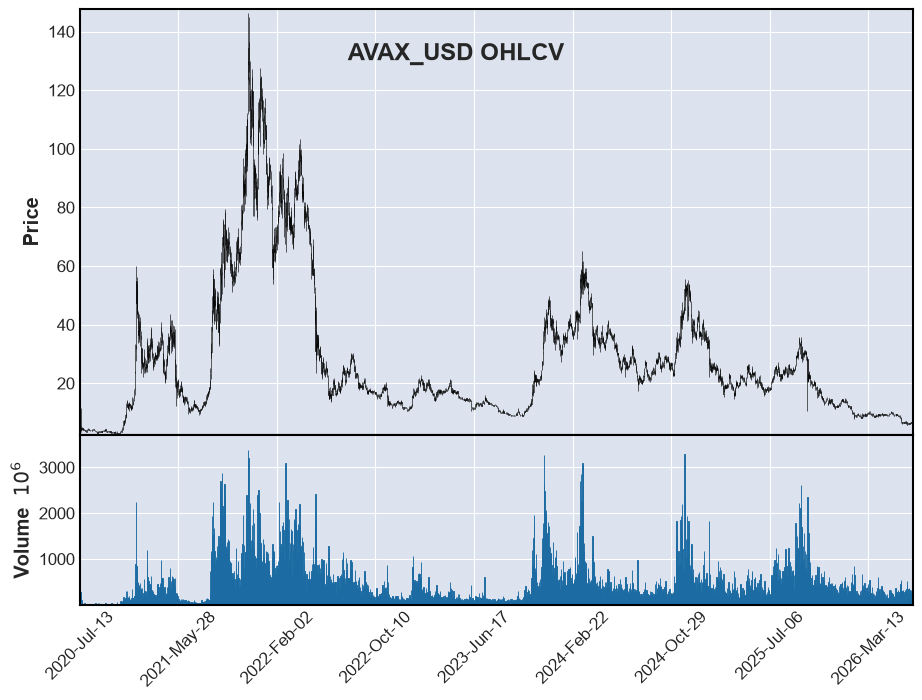}\caption{AVAX}\end{subfigure}\hfill
	\begin{subfigure}[b]{0.155\textwidth}\includegraphics[width=\linewidth]{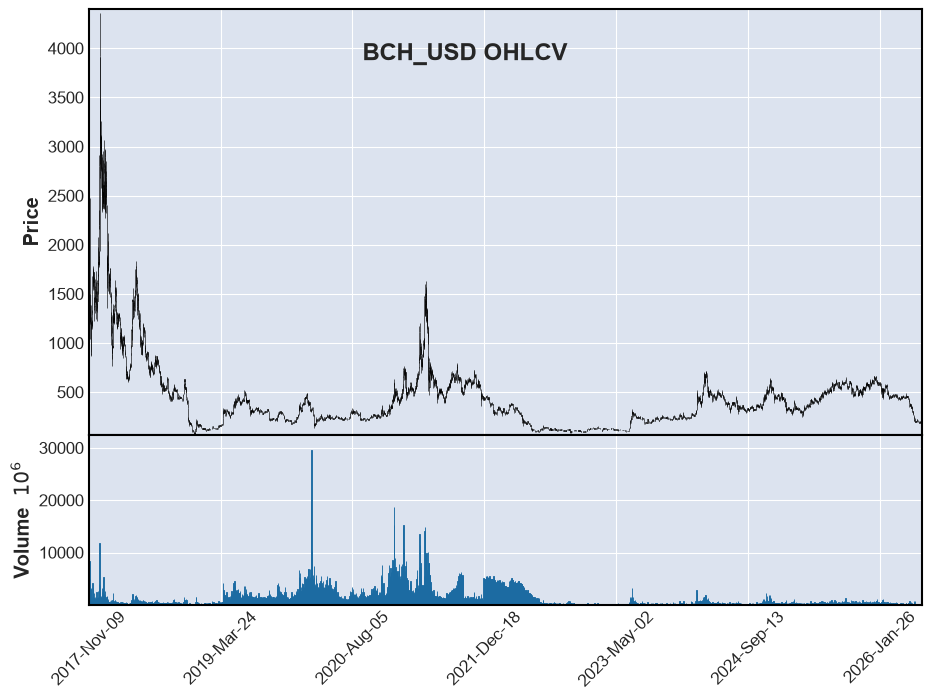}\caption{BCH}\end{subfigure}
	
	\begin{subfigure}[b]{0.155\textwidth}\includegraphics[width=\linewidth]{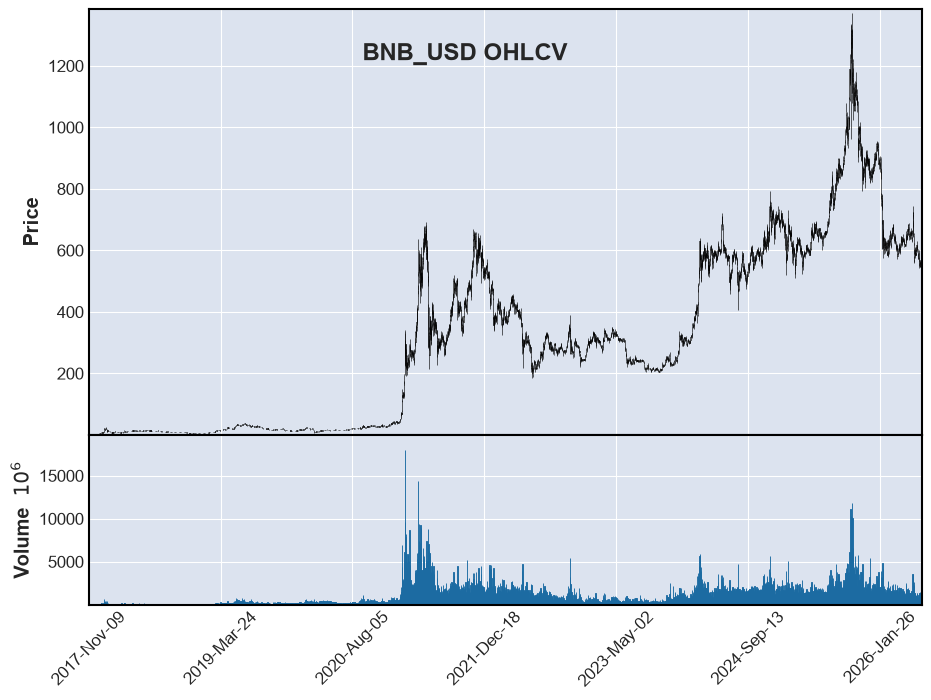}\caption{BNB}\end{subfigure}\hfill
	\begin{subfigure}[b]{0.155\textwidth}\includegraphics[width=\linewidth]{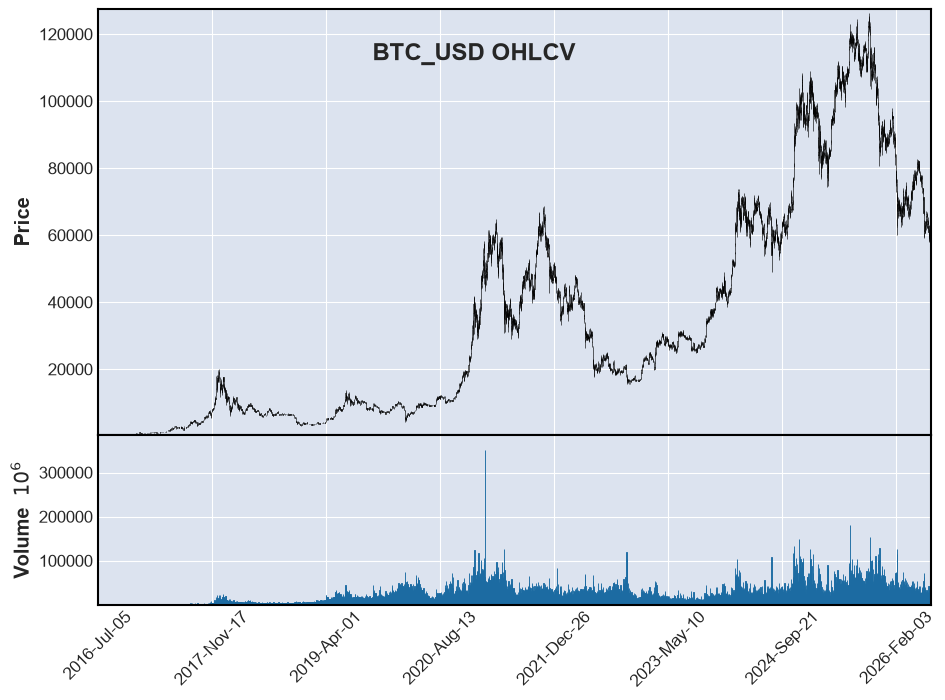}\caption{BTC}\end{subfigure}\hfill
	\begin{subfigure}[b]{0.155\textwidth}\includegraphics[width=\linewidth]{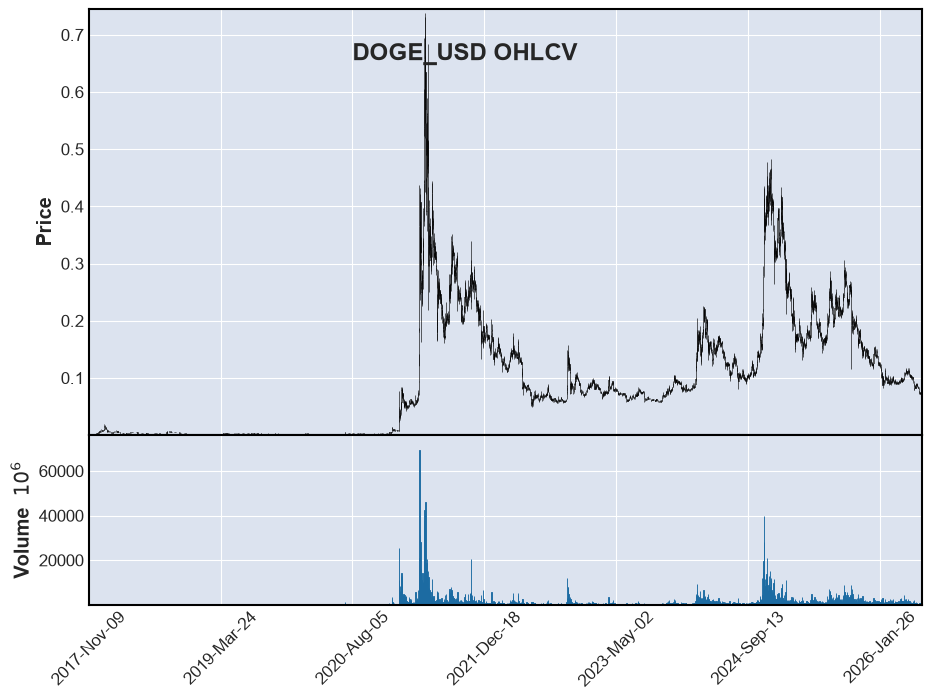}\caption{DOGE}\end{subfigure}\hfill
	\begin{subfigure}[b]{0.155\textwidth}\includegraphics[width=\linewidth]{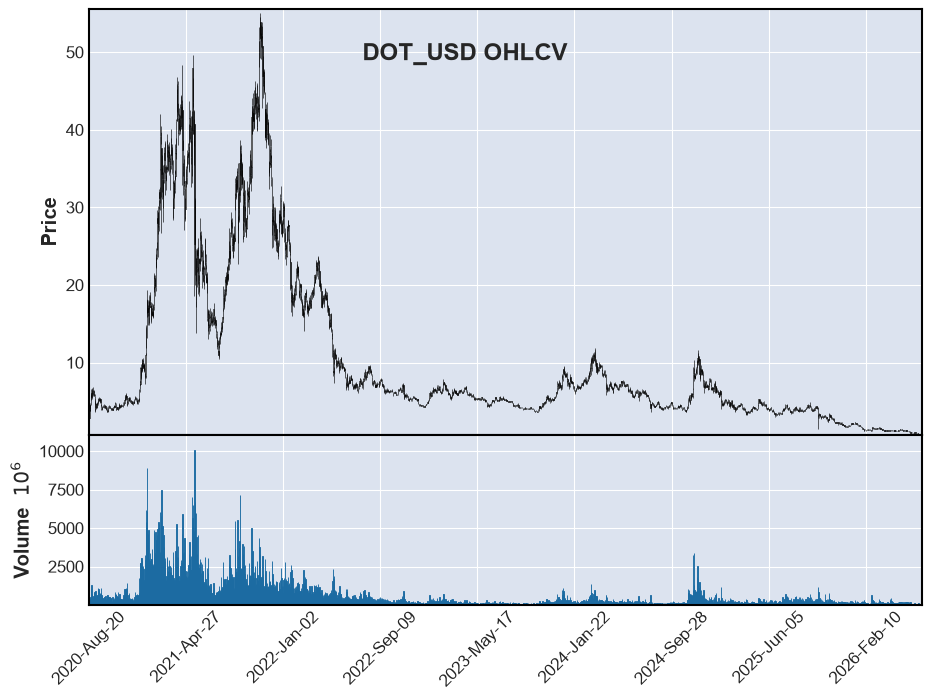}\caption{DOT}\end{subfigure}\hfill
	\begin{subfigure}[b]{0.155\textwidth}\includegraphics[width=\linewidth]{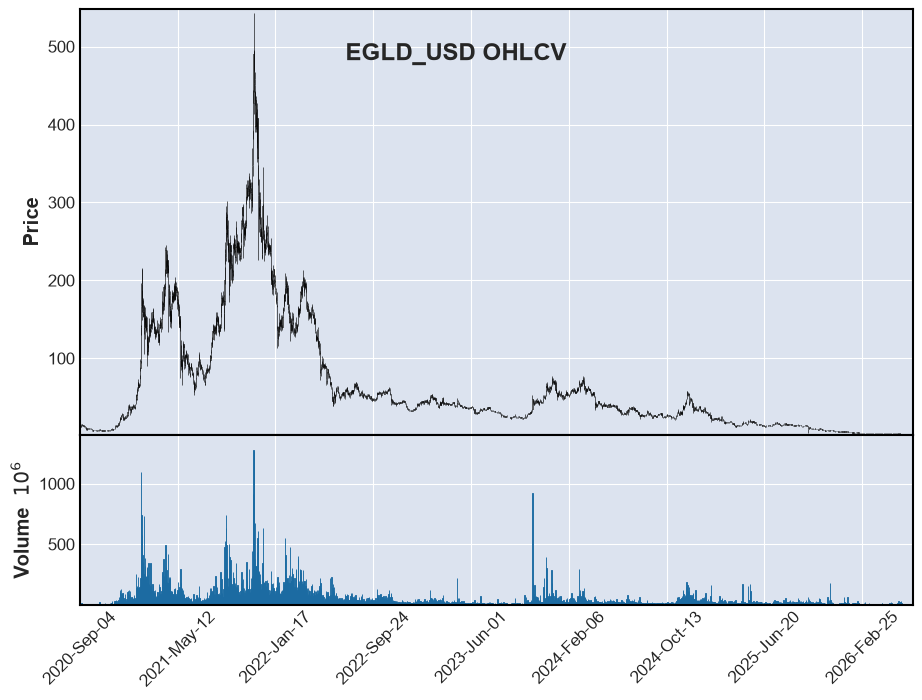}\caption{EGLD}\end{subfigure}\hfill
	\begin{subfigure}[b]{0.155\textwidth}\includegraphics[width=\linewidth]{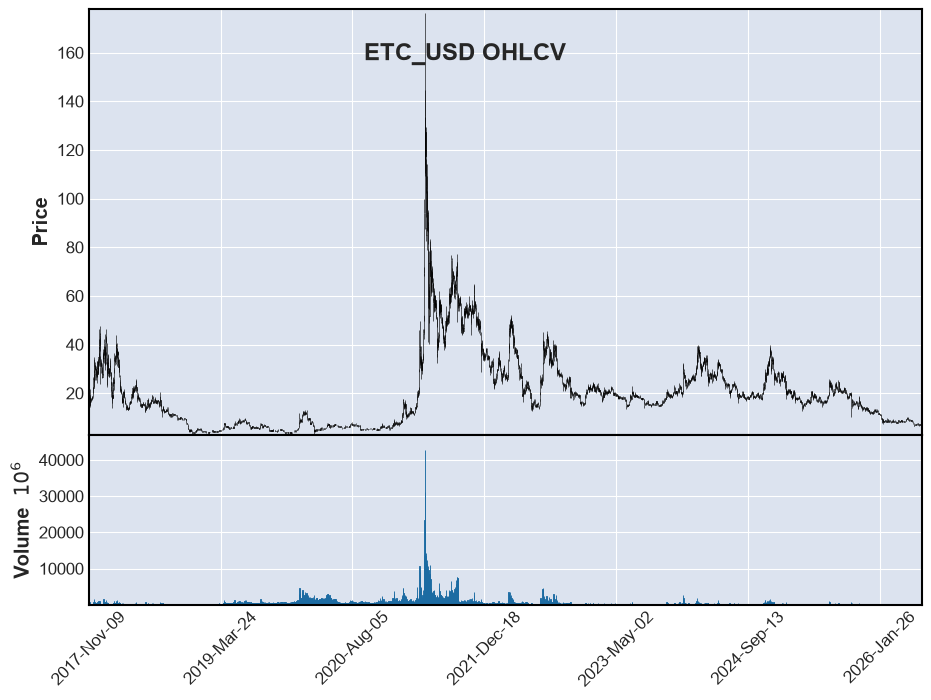}\caption{ETC}\end{subfigure}
	
	\begin{subfigure}[b]{0.155\textwidth}\includegraphics[width=\linewidth]{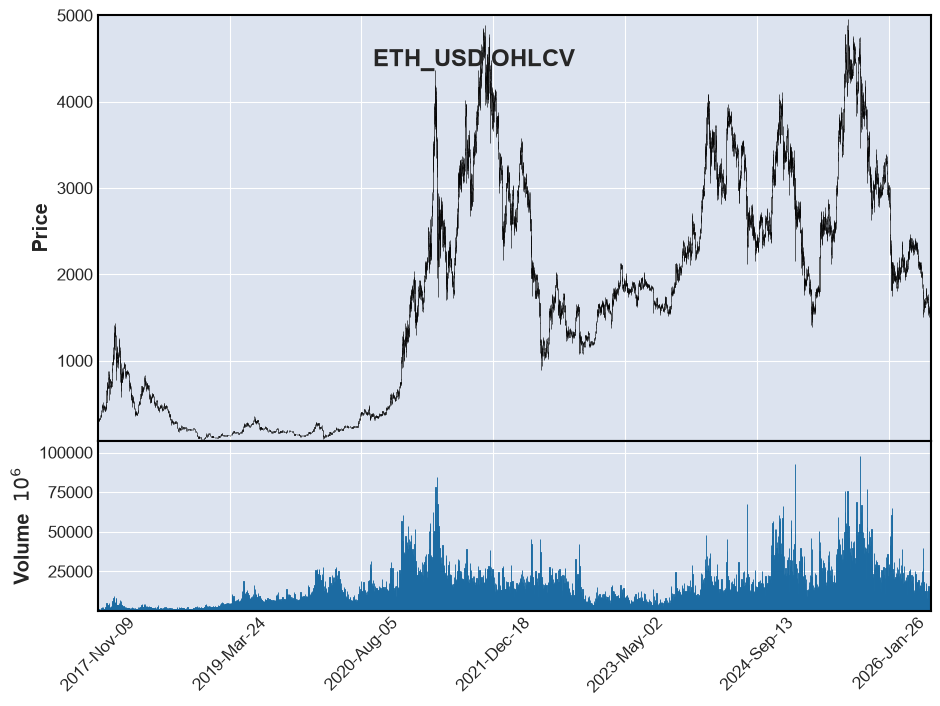}\caption{ETH}\end{subfigure}\hfill
	\begin{subfigure}[b]{0.155\textwidth}\includegraphics[width=\linewidth]{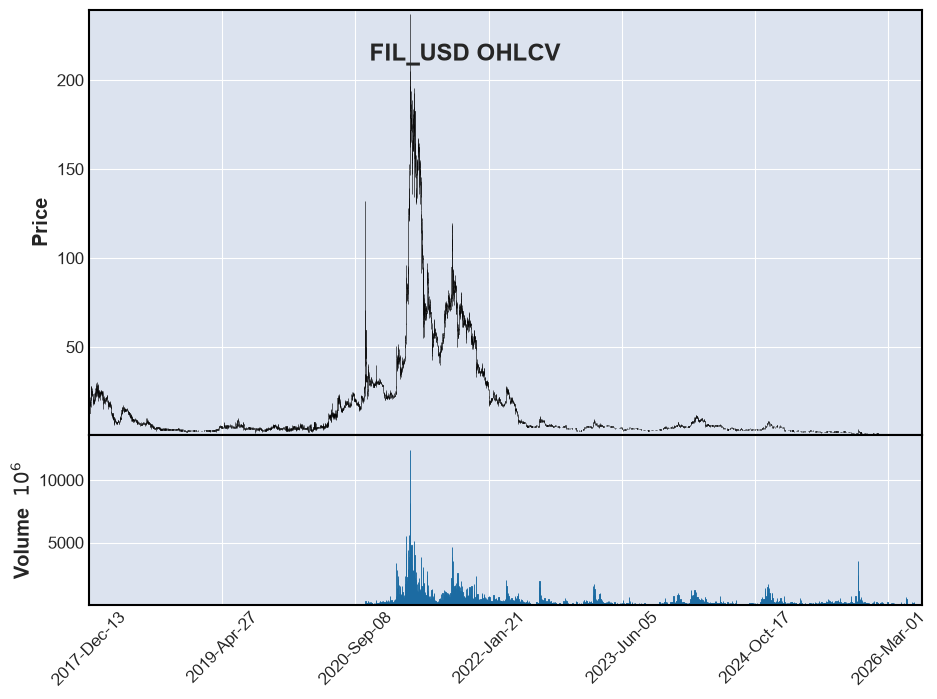}\caption{FIL}\end{subfigure}\hfill
	\begin{subfigure}[b]{0.155\textwidth}\includegraphics[width=\linewidth]{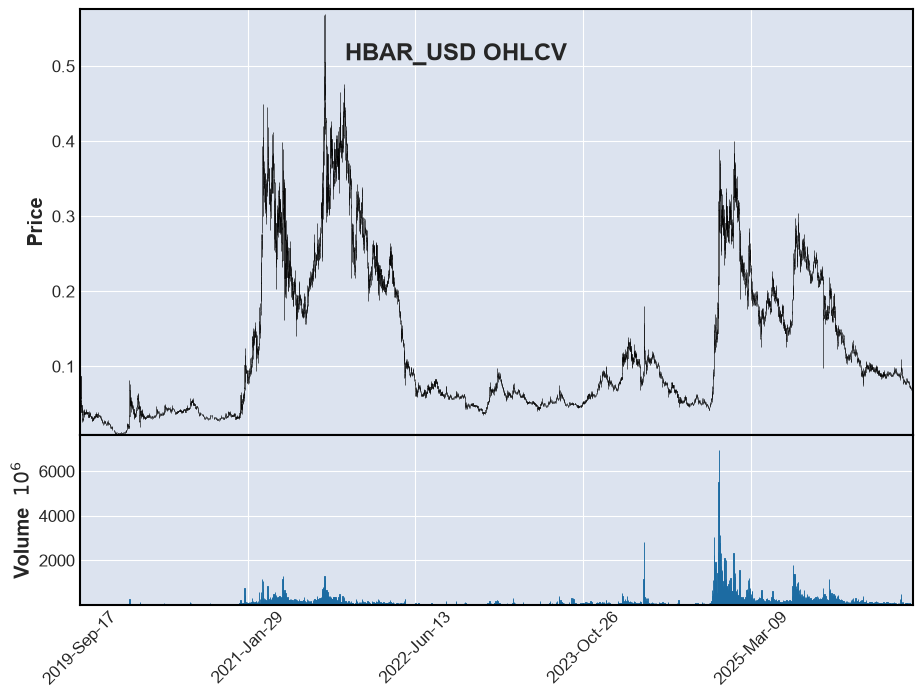}\caption{HBAR}\end{subfigure}\hfill
	\begin{subfigure}[b]{0.155\textwidth}\includegraphics[width=\linewidth]{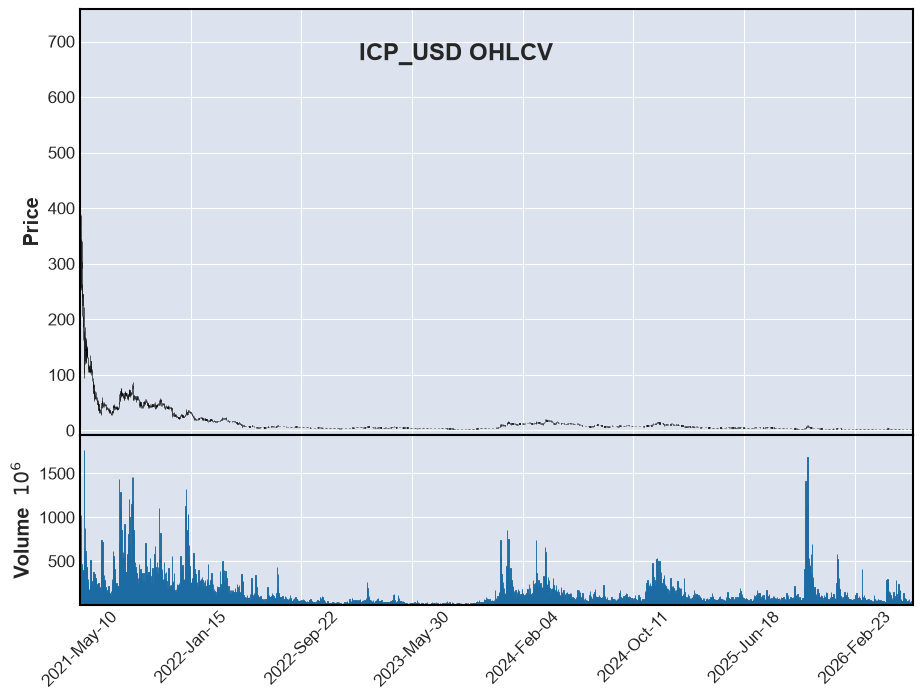}\caption{ICP}\end{subfigure}\hfill
	\begin{subfigure}[b]{0.155\textwidth}\includegraphics[width=\linewidth]{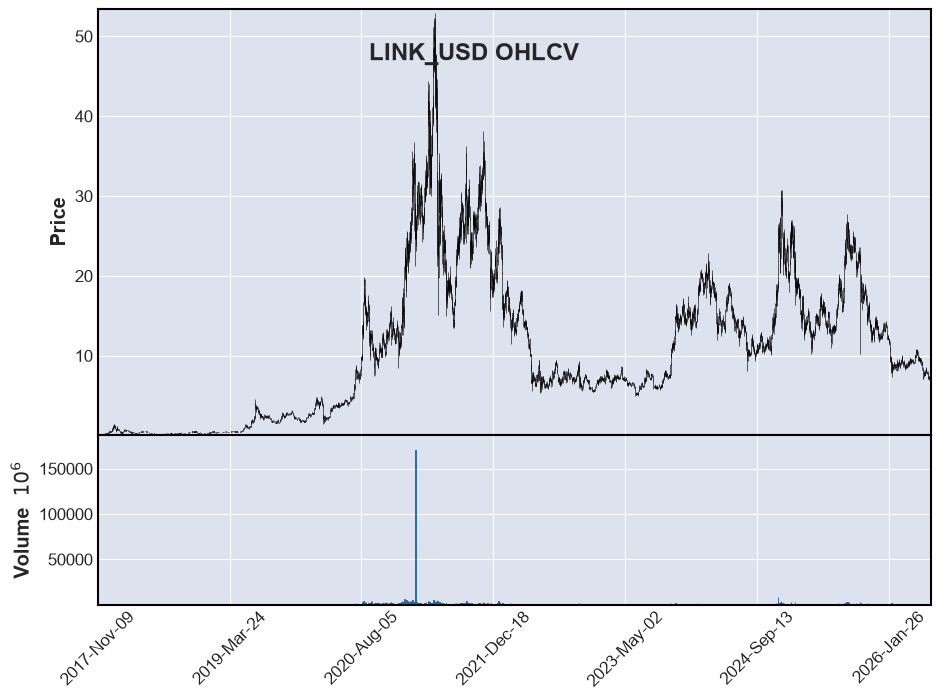}\caption{LINK}\end{subfigure}\hfill
	\begin{subfigure}[b]{0.155\textwidth}\includegraphics[width=\linewidth]{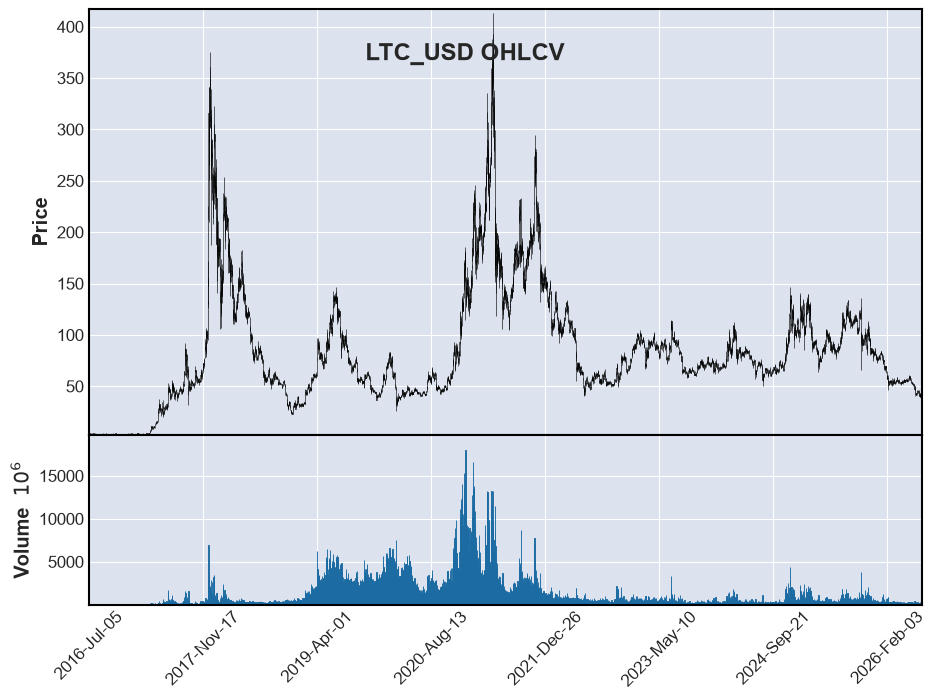}\caption{LTC}\end{subfigure}
	
	\begin{subfigure}[b]{0.155\textwidth}\includegraphics[width=\linewidth]{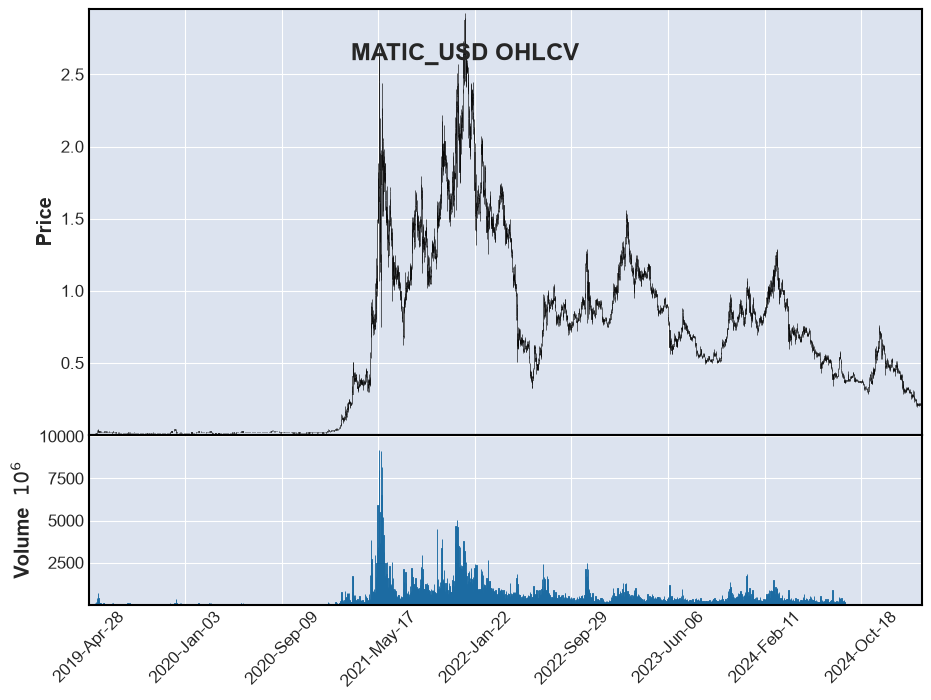}\caption{MATIC}\end{subfigure}\hfill
	\begin{subfigure}[b]{0.155\textwidth}\includegraphics[width=\linewidth]{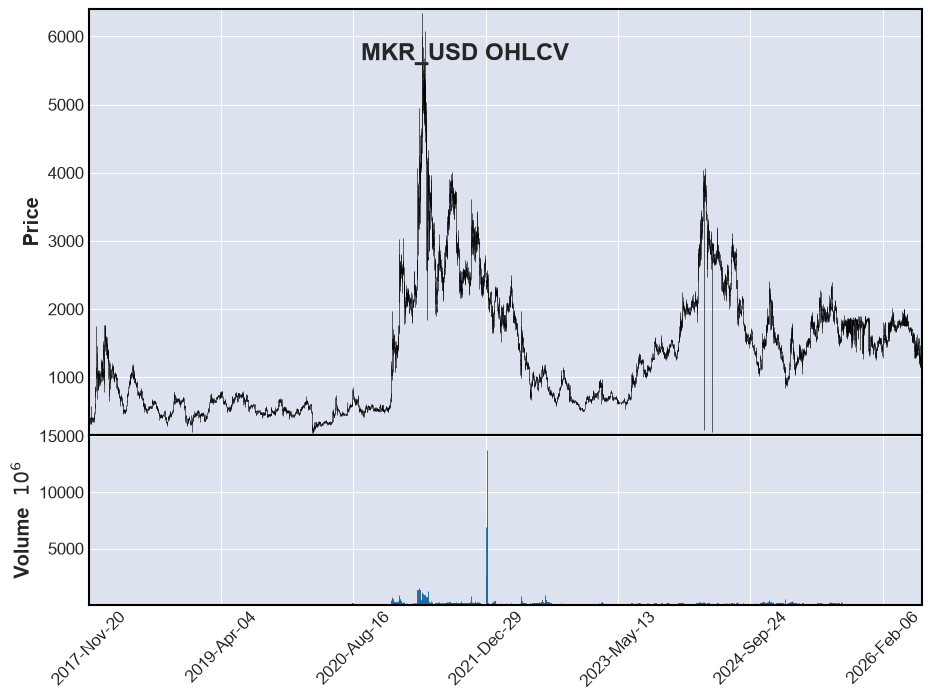}\caption{MKR}\end{subfigure}\hfill
	\begin{subfigure}[b]{0.155\textwidth}\includegraphics[width=\linewidth]{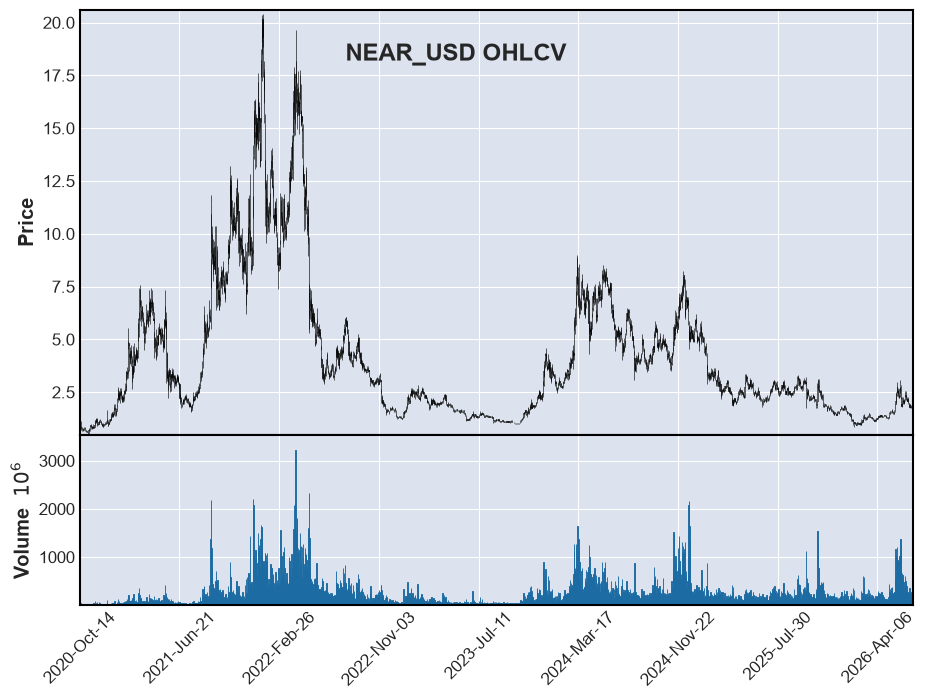}\caption{NEAR}\end{subfigure}\hfill
	\begin{subfigure}[b]{0.155\textwidth}\includegraphics[width=\linewidth]{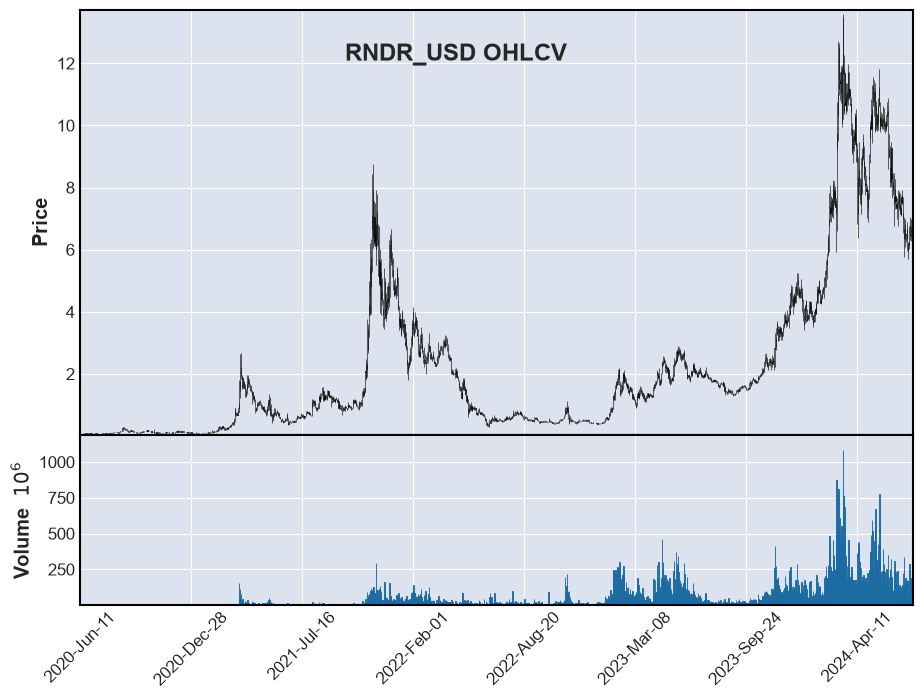}\caption{RNDR}\end{subfigure}\hfill
	\begin{subfigure}[b]{0.155\textwidth}\includegraphics[width=\linewidth]{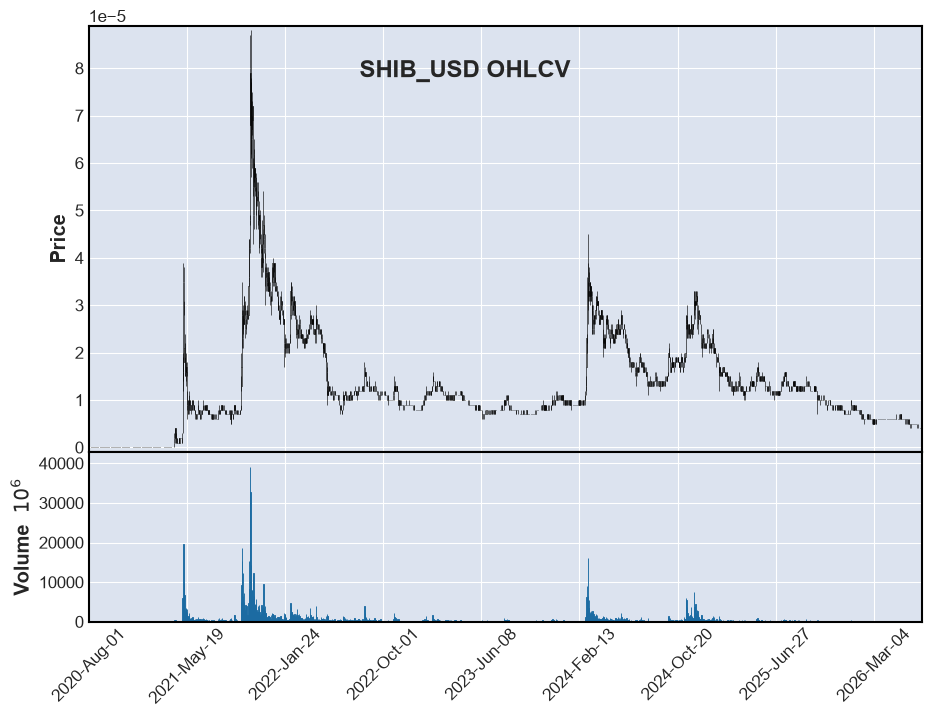}\caption{SHIB}\end{subfigure}\hfill
	\begin{subfigure}[b]{0.155\textwidth}\includegraphics[width=\linewidth]{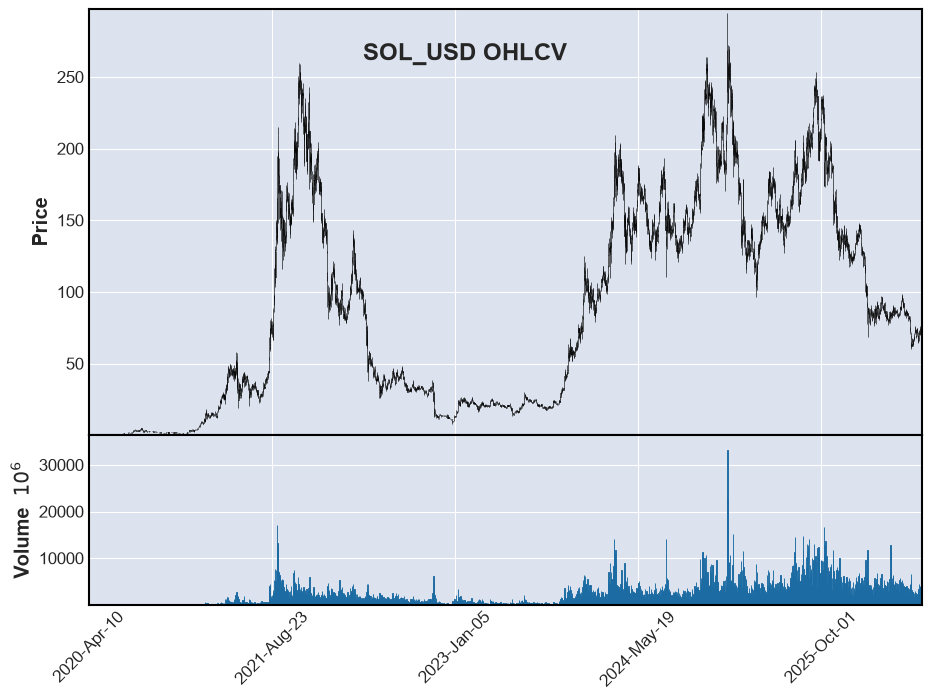}\caption{SOL}\end{subfigure}
	
	\begin{subfigure}[b]{0.155\textwidth}\includegraphics[width=\linewidth]{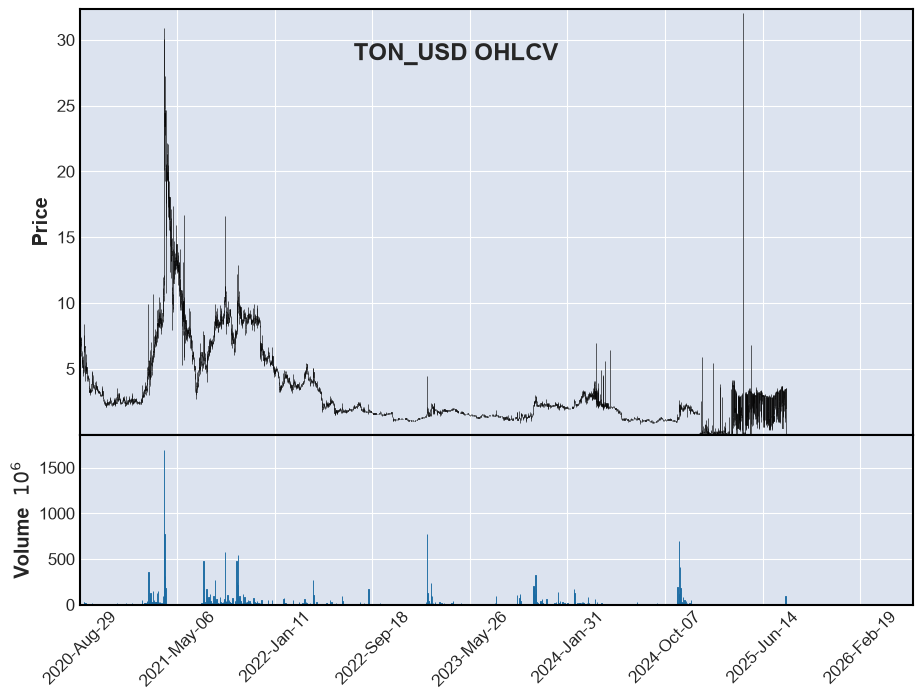}\caption{TON}\end{subfigure}\hfill
	\begin{subfigure}[b]{0.155\textwidth}\includegraphics[width=\linewidth]{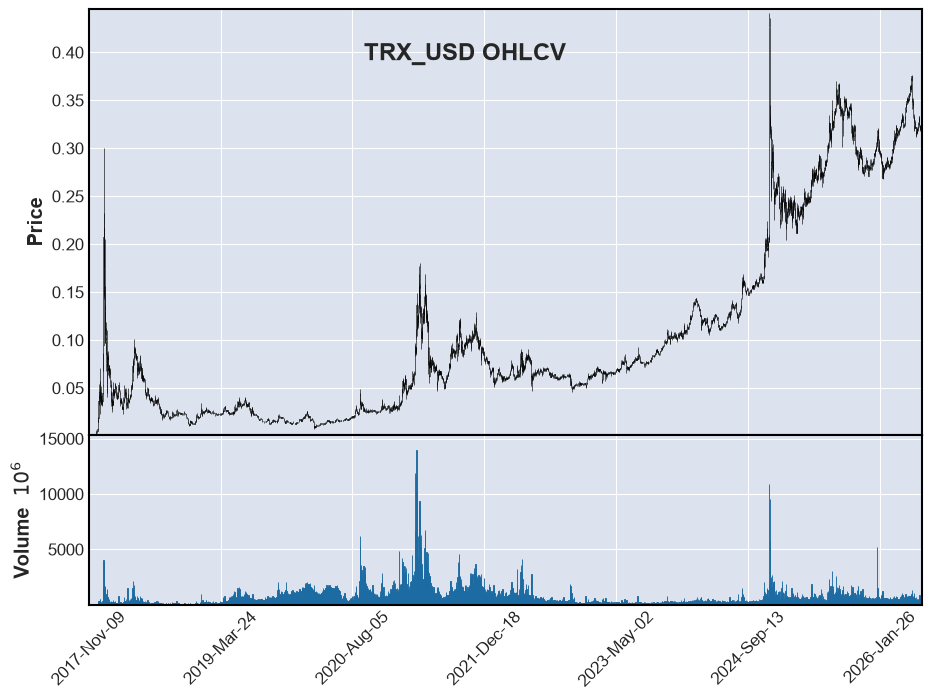}\caption{TRX}\end{subfigure}\hfill
	\begin{subfigure}[b]{0.155\textwidth}\includegraphics[width=\linewidth]{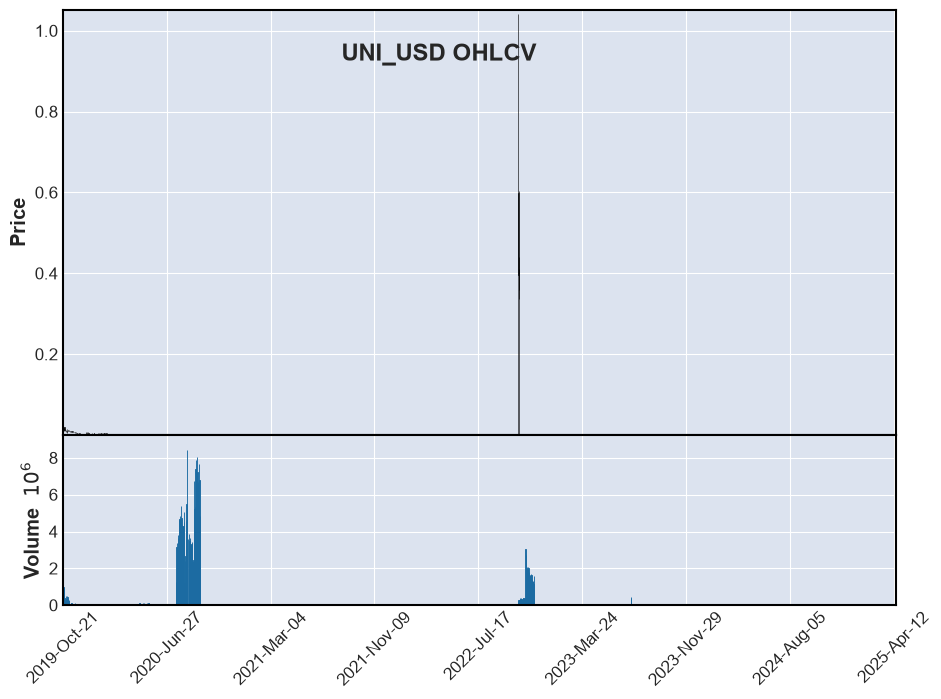}\caption{UNI}\end{subfigure}\hfill
	\begin{subfigure}[b]{0.155\textwidth}\includegraphics[width=\linewidth]{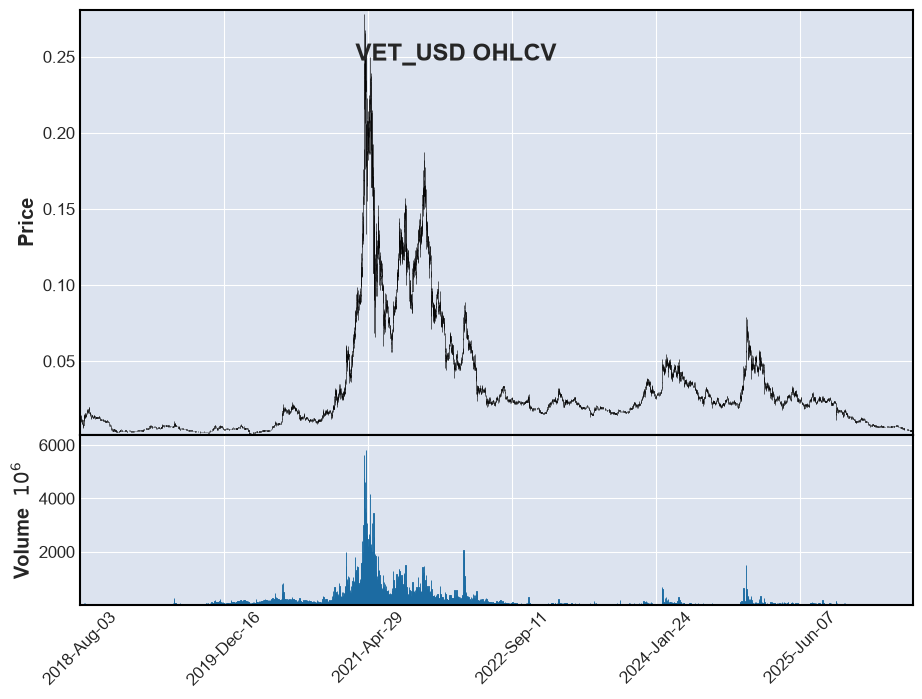}\caption{VET}\end{subfigure}\hfill
	\begin{subfigure}[b]{0.155\textwidth}\includegraphics[width=\linewidth]{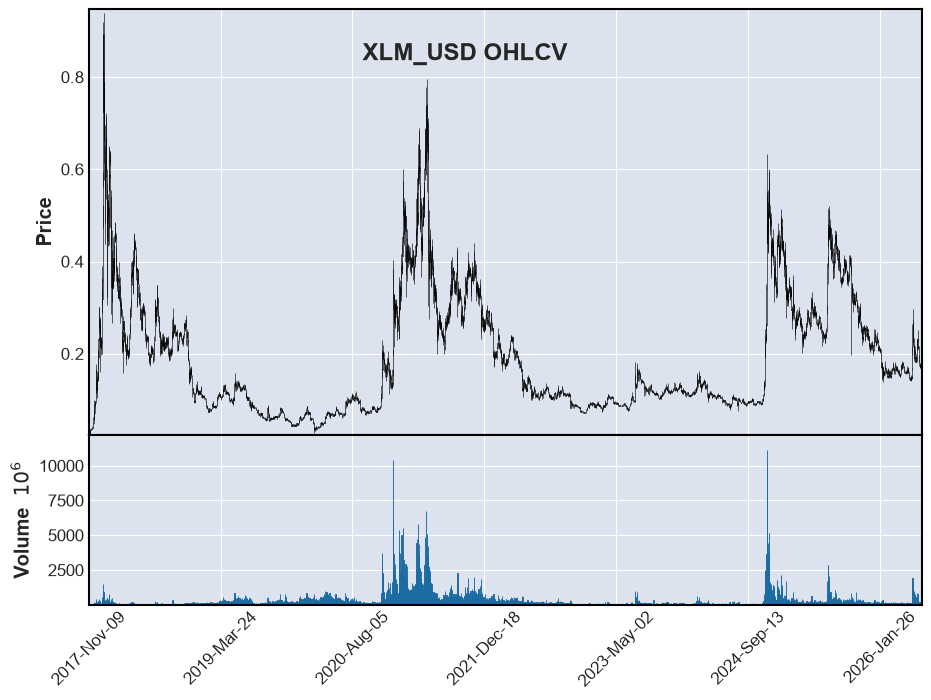}\caption{XLM}\end{subfigure}\hfill
	\begin{subfigure}[b]{0.155\textwidth}\includegraphics[width=\linewidth]{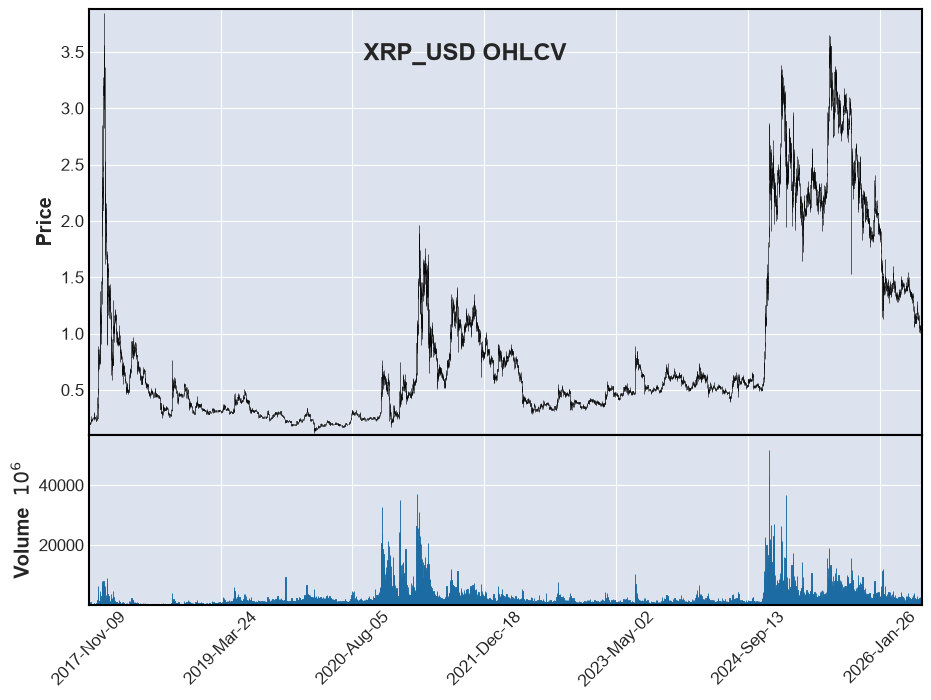}\caption{XRP}\end{subfigure}
	
	\caption{Daily OHLC price histories for the 30 cryptocurrency symbols used in the experiments.}
	\label{fig:ohlc_crypto}
\end{figure*}

\begin{figure*}[htbp]
	\centering
	\scriptsize
	\begin{subfigure}[b]{0.155\textwidth}\includegraphics[width=\linewidth]{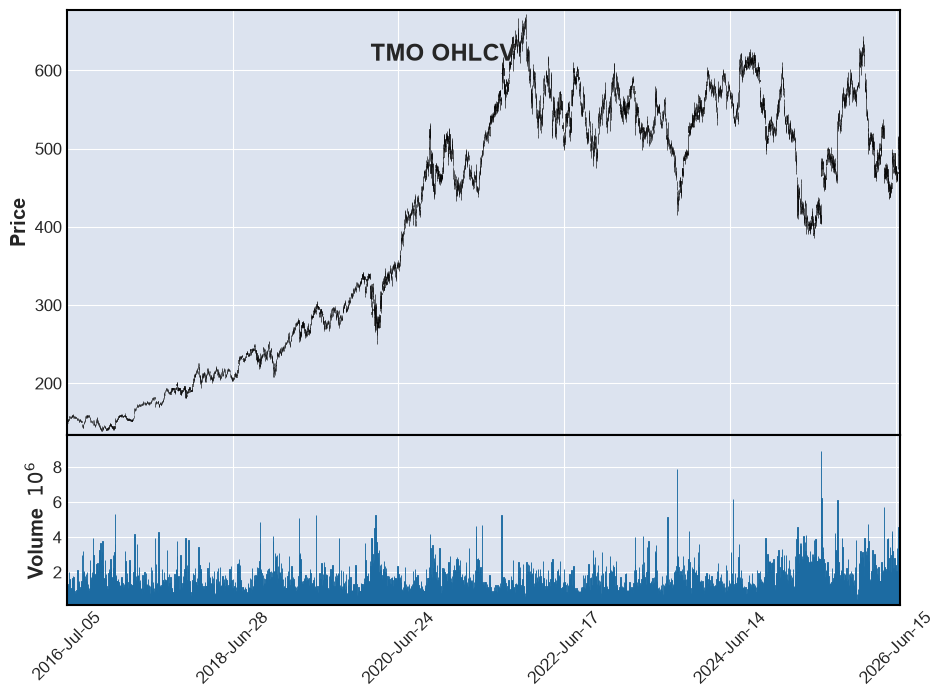}\caption{TMO}\end{subfigure}\hfill
	\begin{subfigure}[b]{0.155\textwidth}\includegraphics[width=\linewidth]{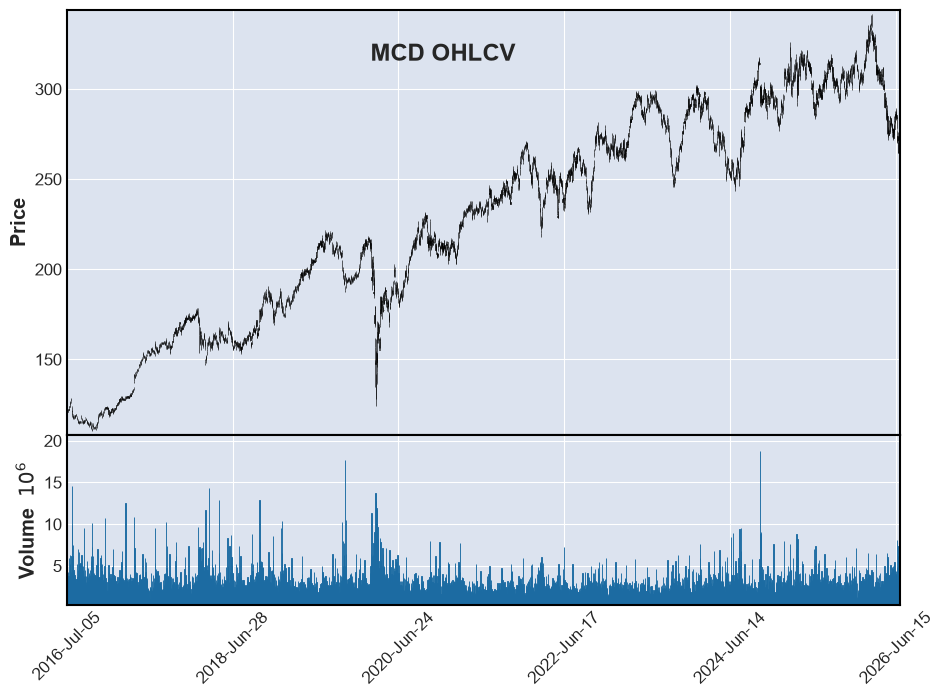}\caption{MCD}\end{subfigure}\hfill
	\begin{subfigure}[b]{0.155\textwidth}\includegraphics[width=\linewidth]{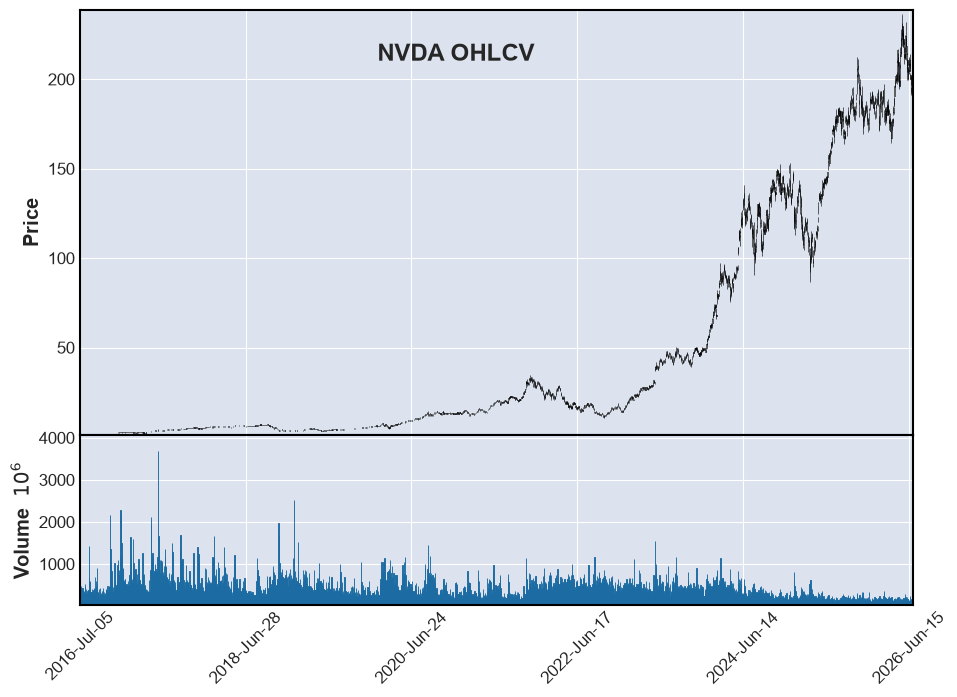}\caption{NVDA}\end{subfigure}\hfill
	\begin{subfigure}[b]{0.155\textwidth}\includegraphics[width=\linewidth]{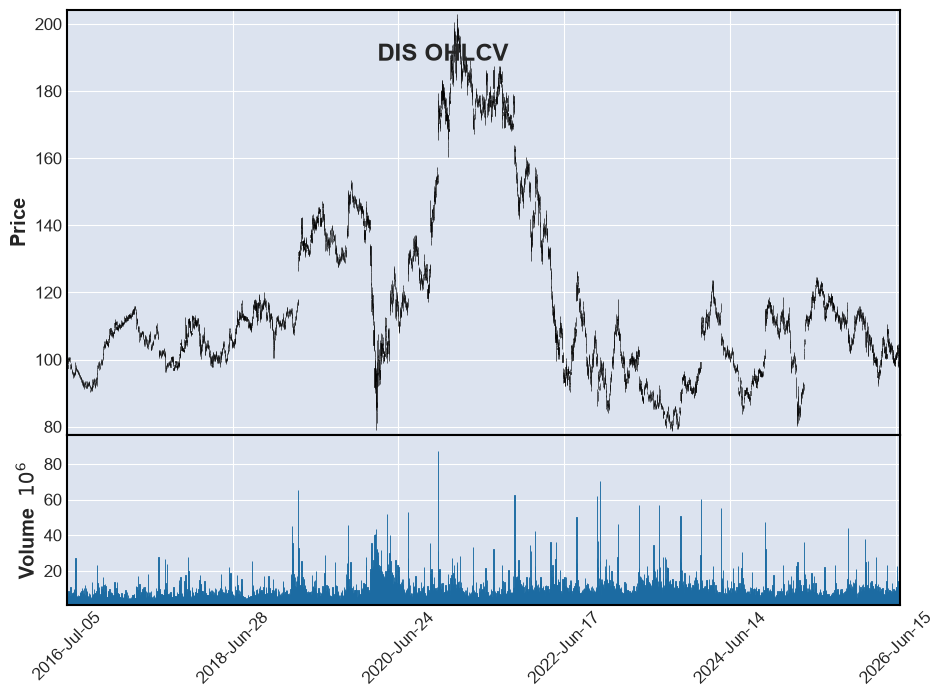}\caption{DIS}\end{subfigure}\hfill
	\begin{subfigure}[b]{0.155\textwidth}\includegraphics[width=\linewidth]{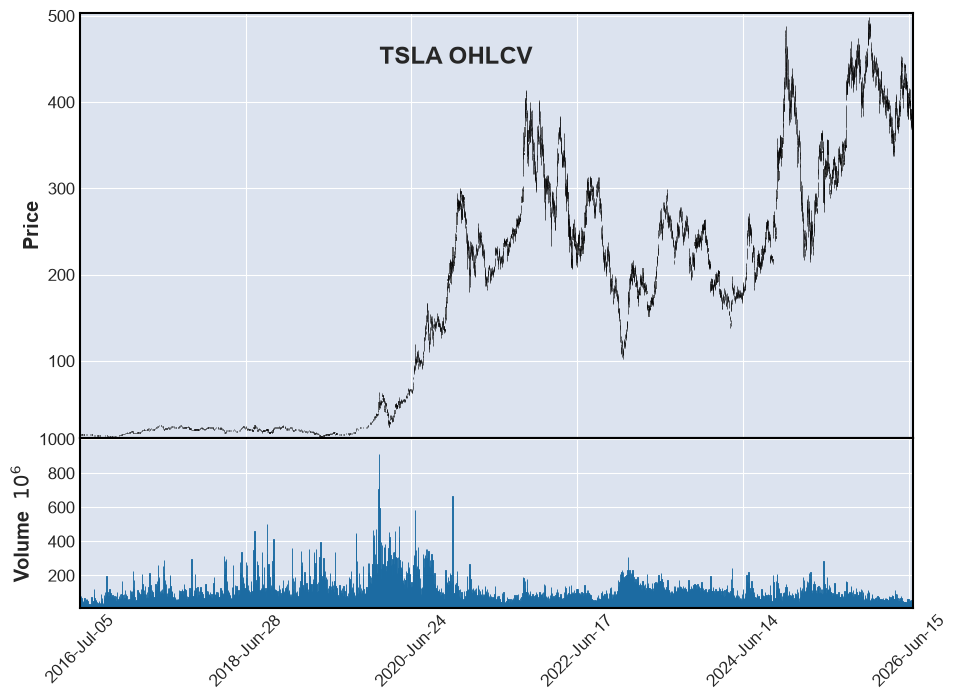}\caption{TSLA}\end{subfigure}\hfill
	\begin{subfigure}[b]{0.155\textwidth}\includegraphics[width=\linewidth]{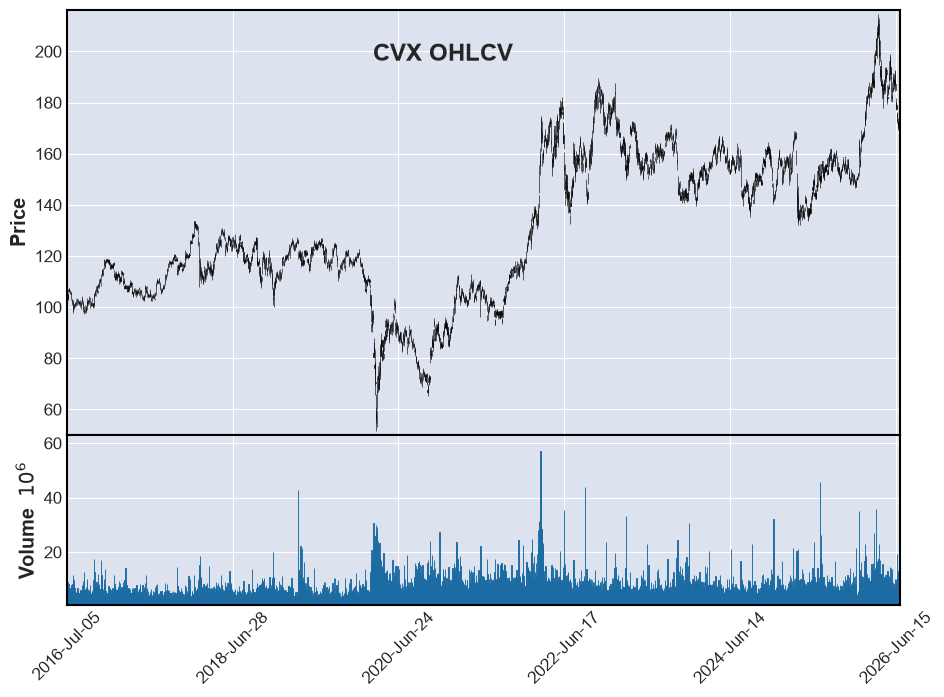}\caption{CVX}\end{subfigure}
	
	\begin{subfigure}[b]{0.155\textwidth}\includegraphics[width=\linewidth]{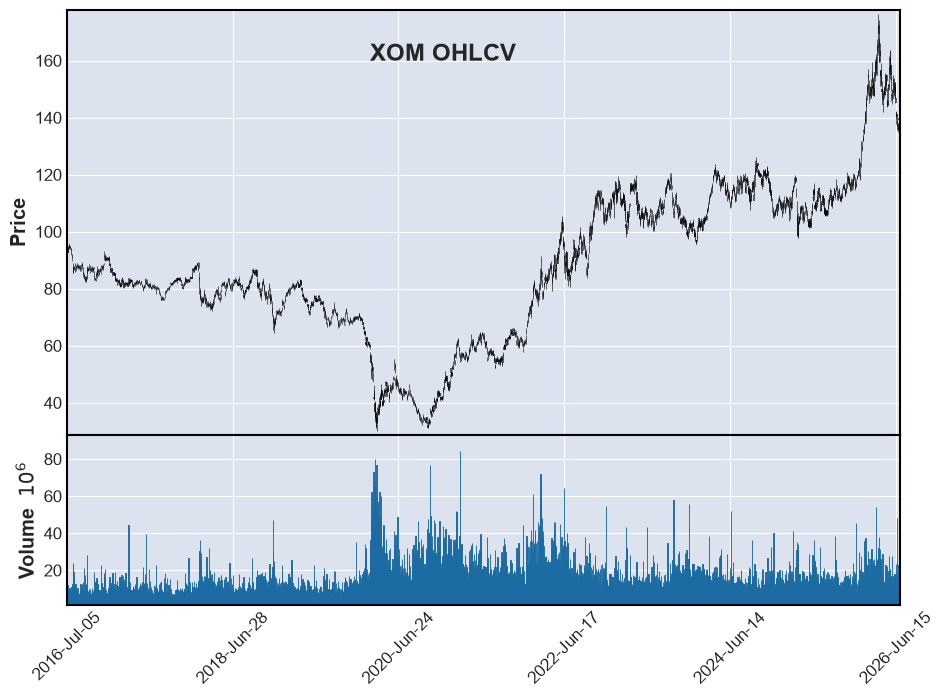}\caption{XOM}\end{subfigure}\hfill
	\begin{subfigure}[b]{0.155\textwidth}\includegraphics[width=\linewidth]{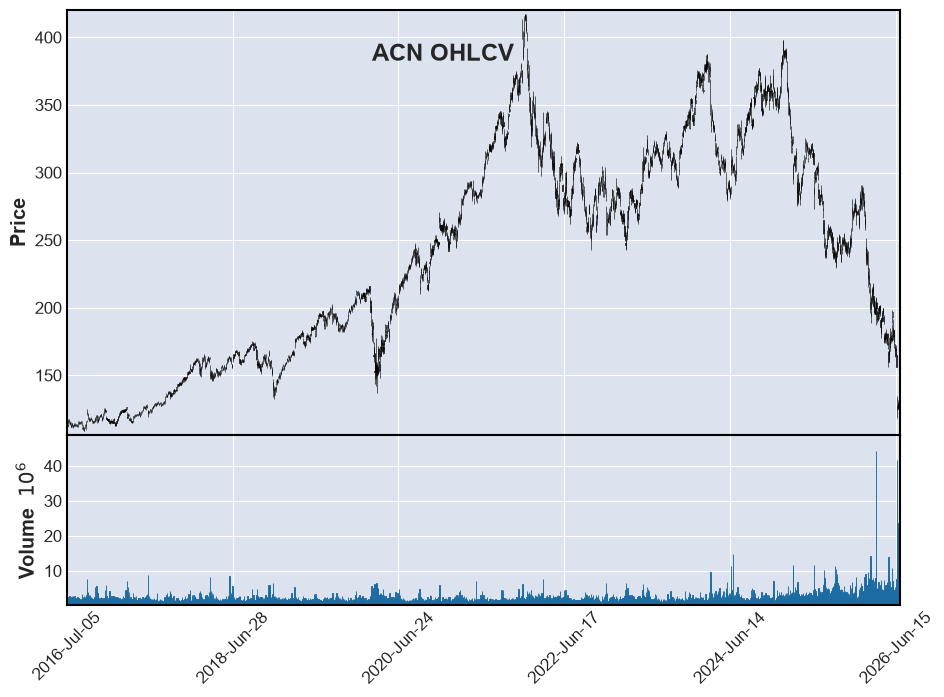}\caption{ACN}\end{subfigure}\hfill
	\begin{subfigure}[b]{0.155\textwidth}\includegraphics[width=\linewidth]{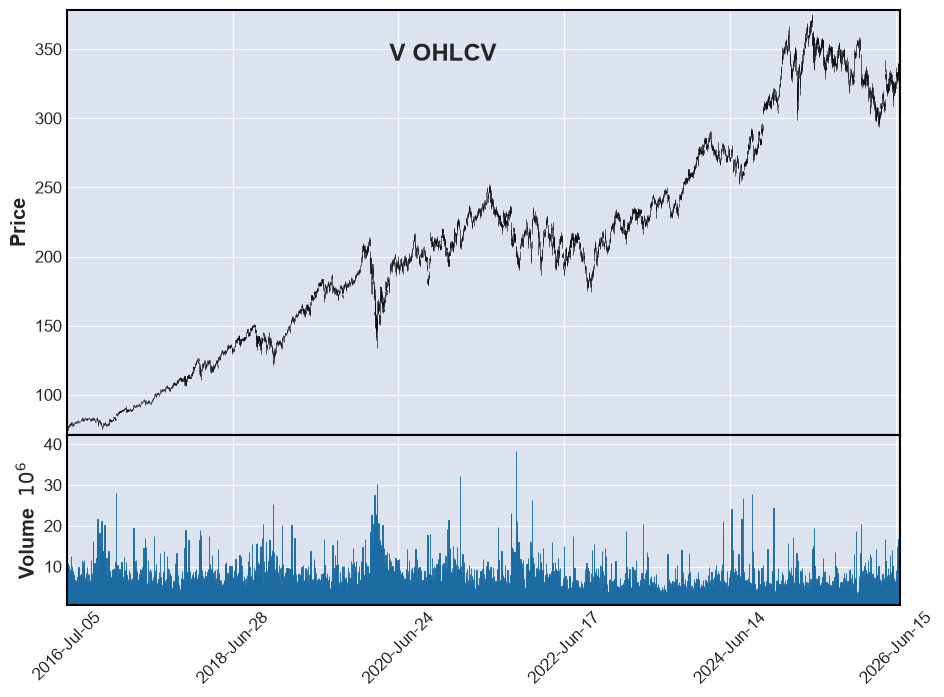}\caption{V}\end{subfigure}\hfill
	\begin{subfigure}[b]{0.155\textwidth}\includegraphics[width=\linewidth]{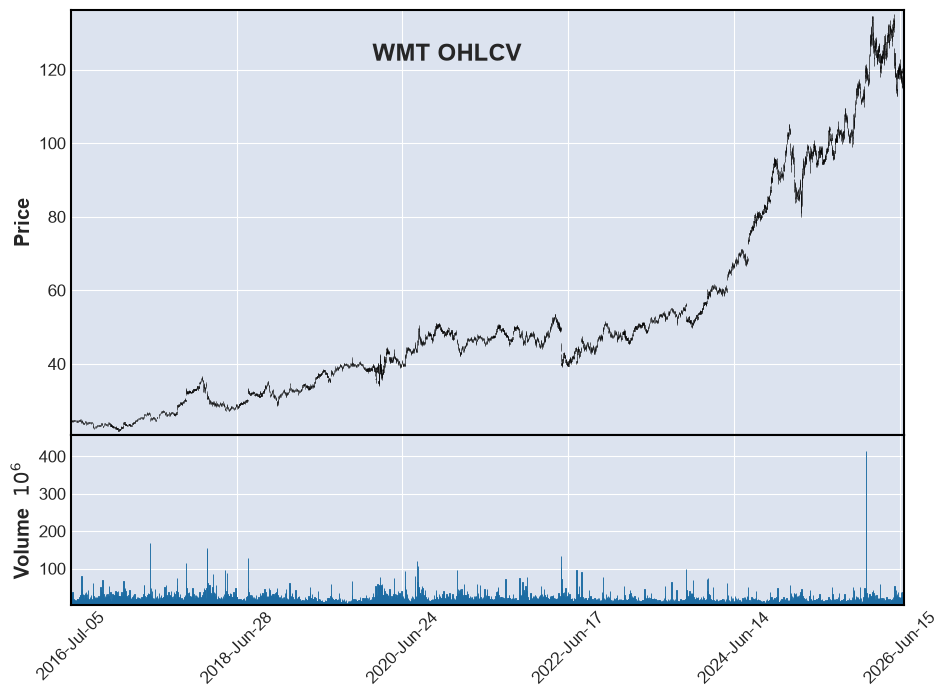}\caption{WMT}\end{subfigure}\hfill
	\begin{subfigure}[b]{0.155\textwidth}\includegraphics[width=\linewidth]{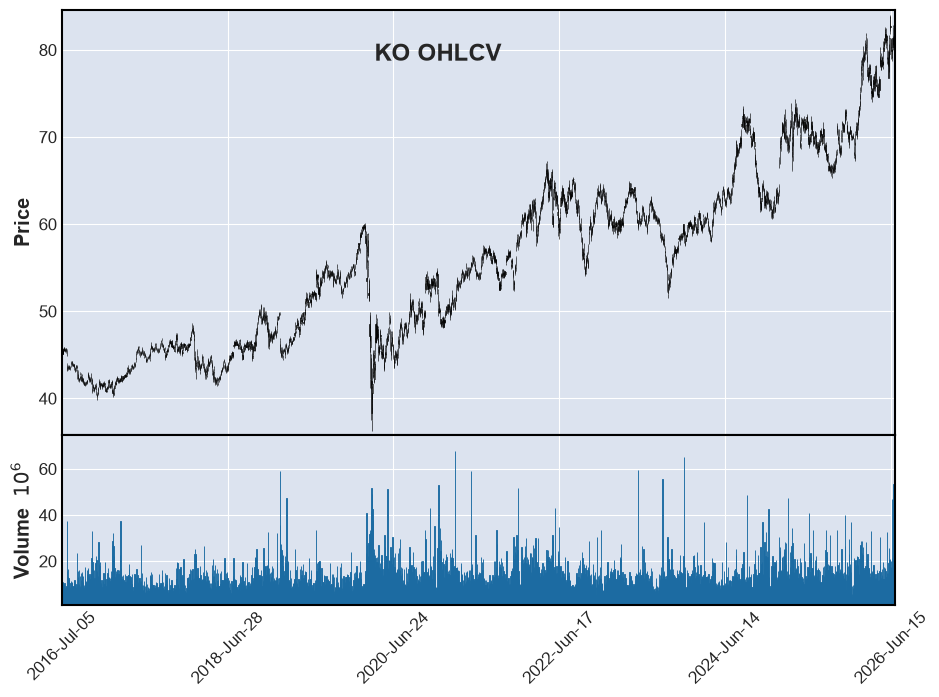}\caption{KO}\end{subfigure}\hfill
	\begin{subfigure}[b]{0.155\textwidth}\includegraphics[width=\linewidth]{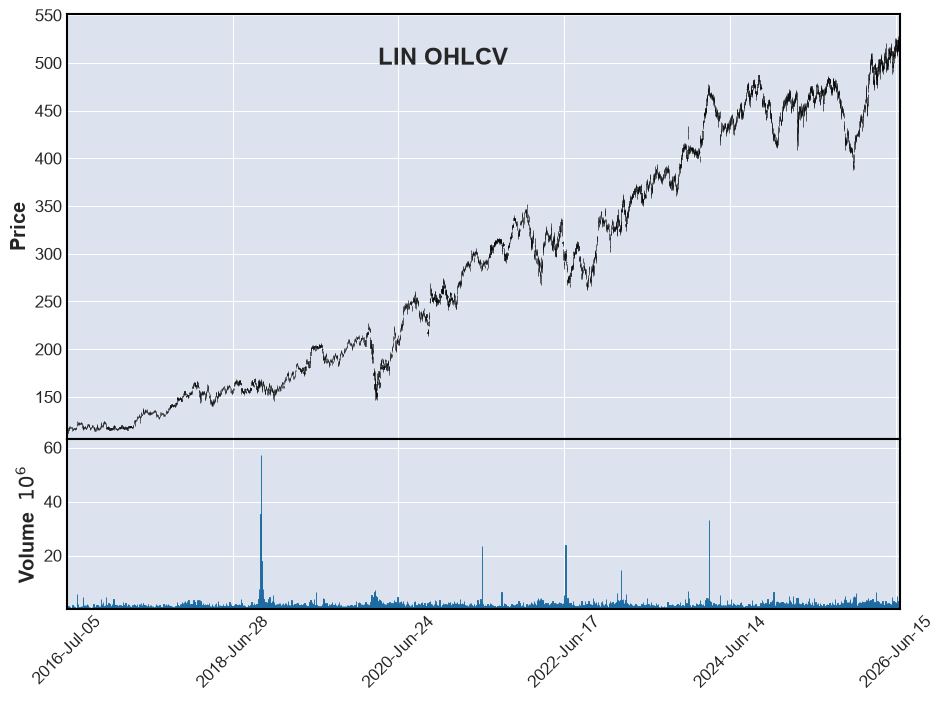}\caption{LIN}\end{subfigure}
	
	\begin{subfigure}[b]{0.155\textwidth}\includegraphics[width=\linewidth]{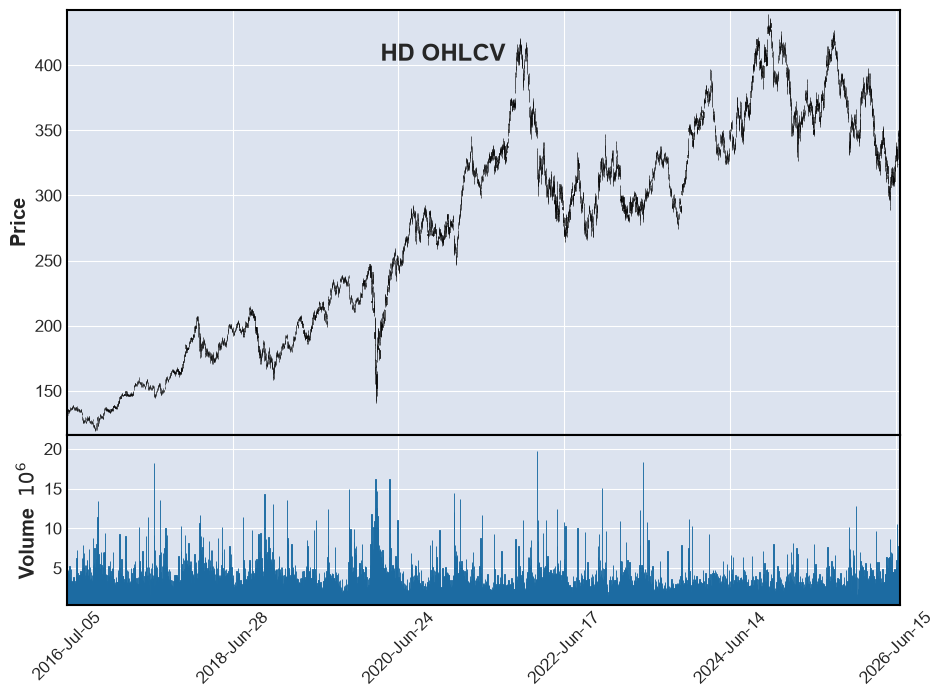}\caption{HD}\end{subfigure}\hfill
	\begin{subfigure}[b]{0.155\textwidth}\includegraphics[width=\linewidth]{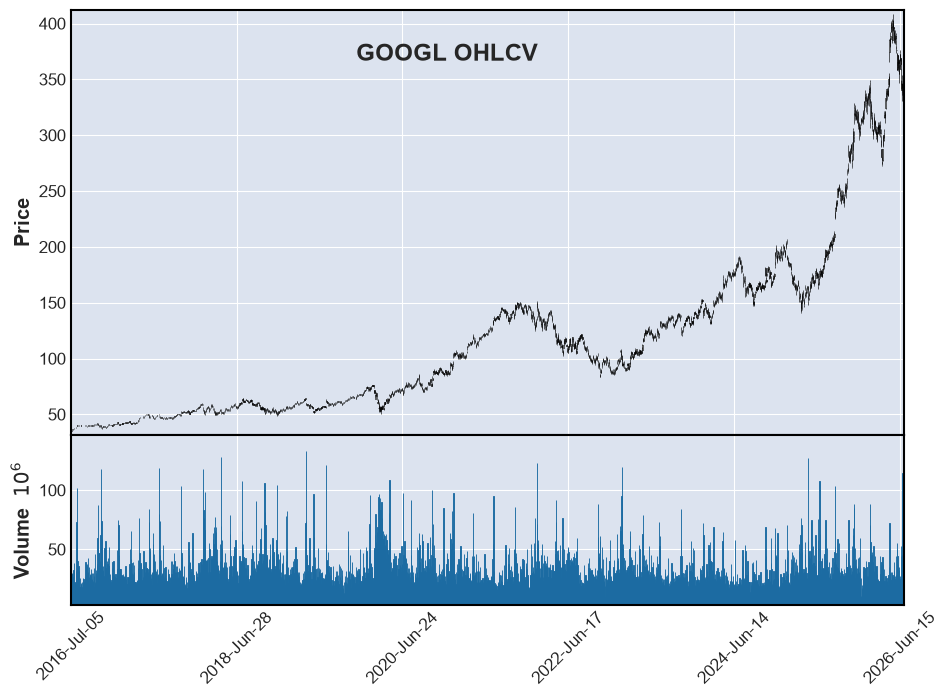}\caption{GOOGL}\end{subfigure}\hfill
	\begin{subfigure}[b]{0.155\textwidth}\includegraphics[width=\linewidth]{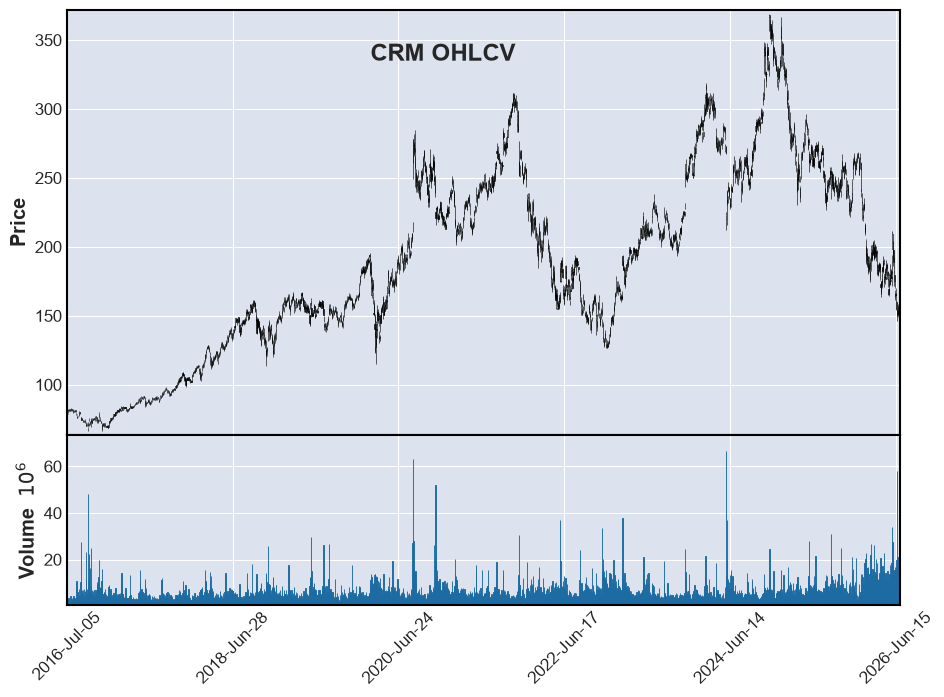}\caption{CRM}\end{subfigure}\hfill
	\begin{subfigure}[b]{0.155\textwidth}\includegraphics[width=\linewidth]{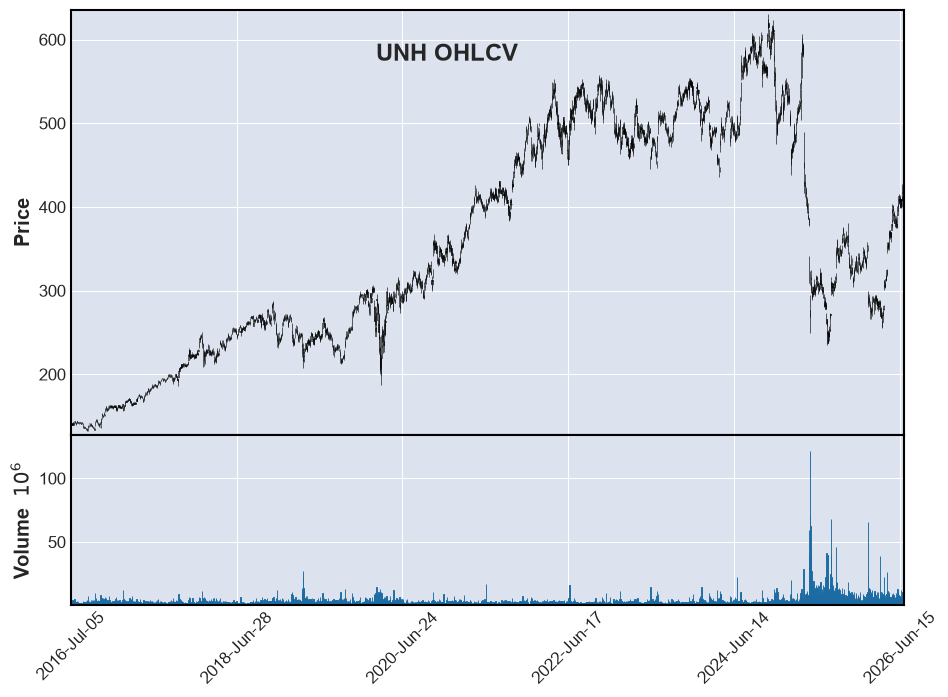}\caption{UNH}\end{subfigure}\hfill
	\begin{subfigure}[b]{0.155\textwidth}\includegraphics[width=\linewidth]{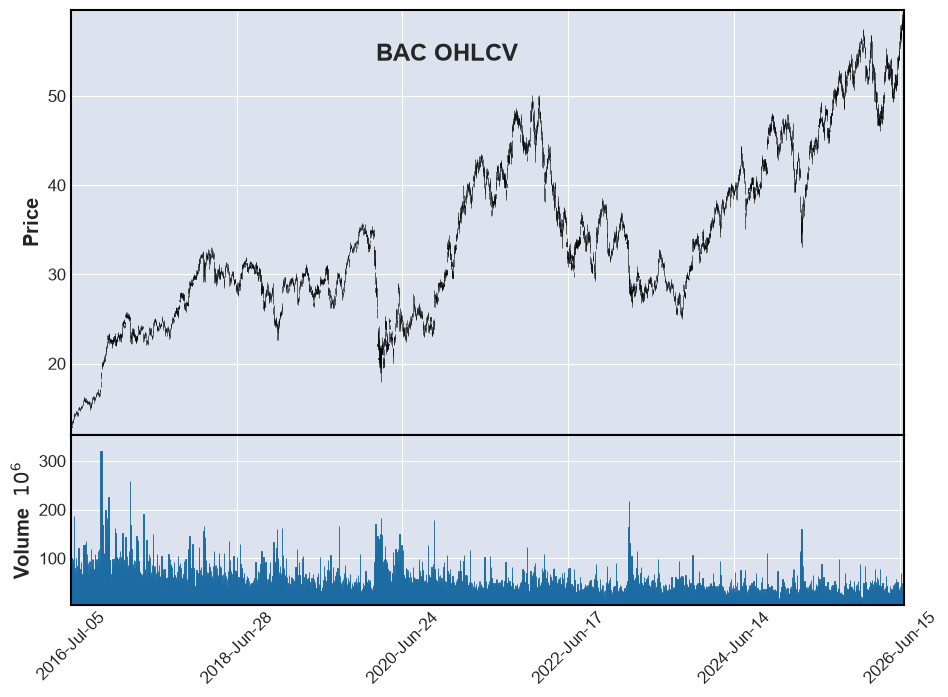}\caption{BAC}\end{subfigure}\hfill
	\begin{subfigure}[b]{0.155\textwidth}\includegraphics[width=\linewidth]{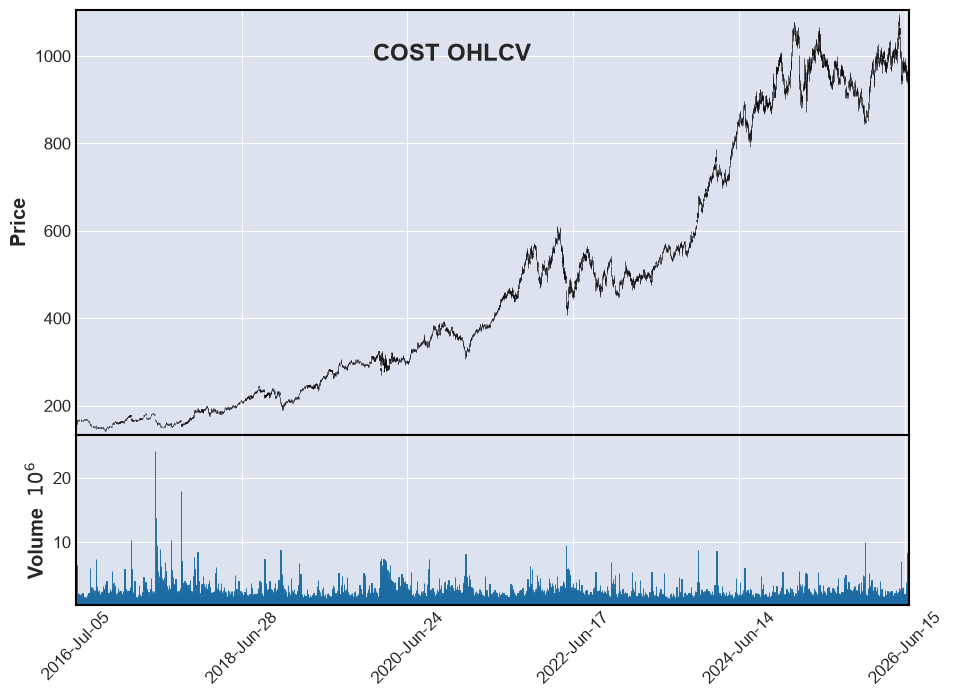}\caption{COST}\end{subfigure}
	
	\begin{subfigure}[b]{0.155\textwidth}\includegraphics[width=\linewidth]{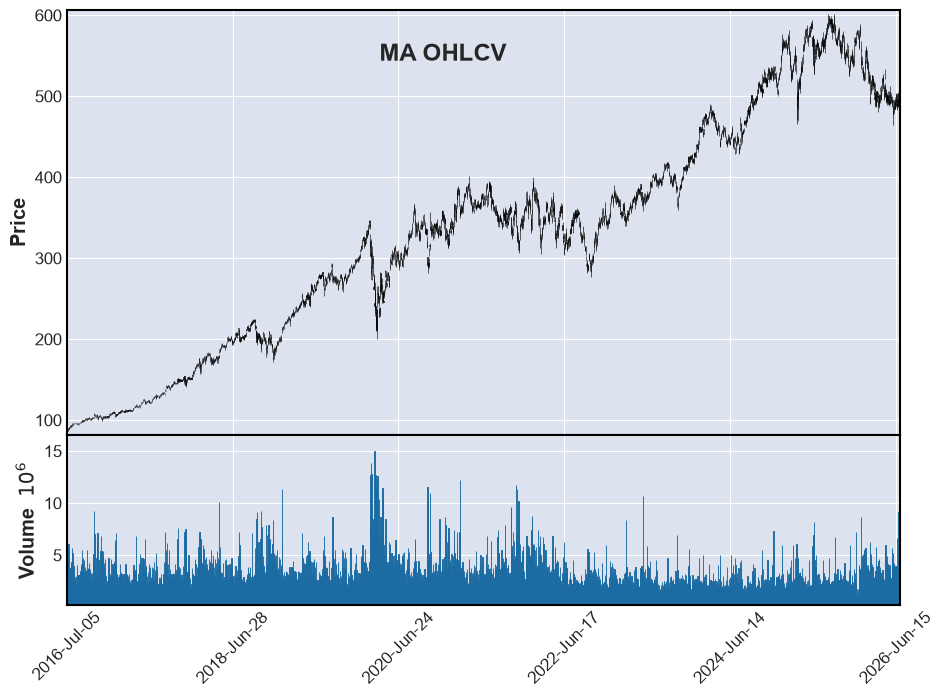}\caption{MA}\end{subfigure}\hfill
	\begin{subfigure}[b]{0.155\textwidth}\includegraphics[width=\linewidth]{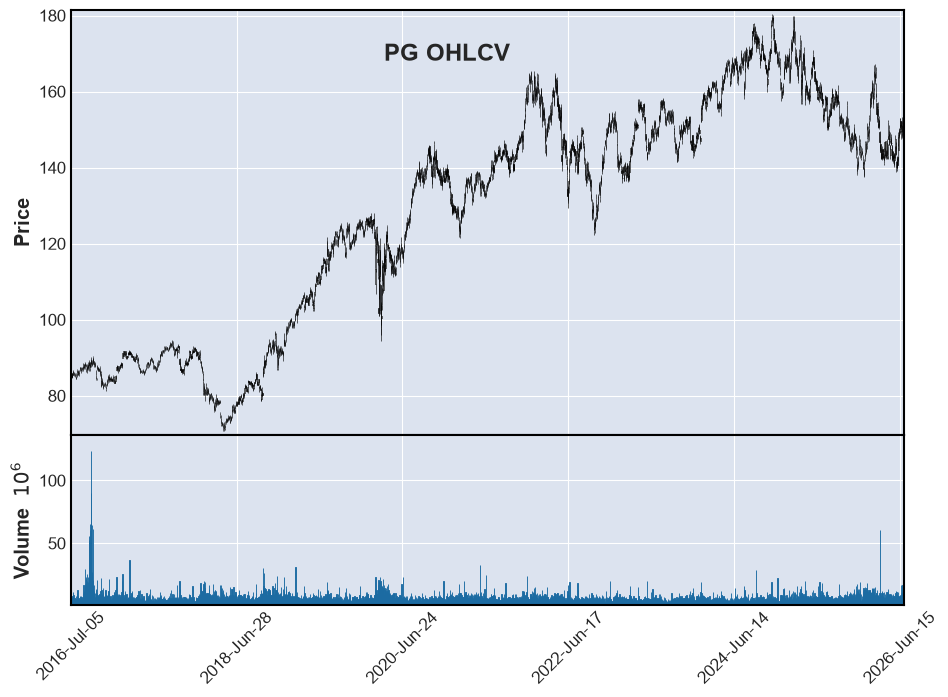}\caption{PG}\end{subfigure}\hfill
	\begin{subfigure}[b]{0.155\textwidth}\includegraphics[width=\linewidth]{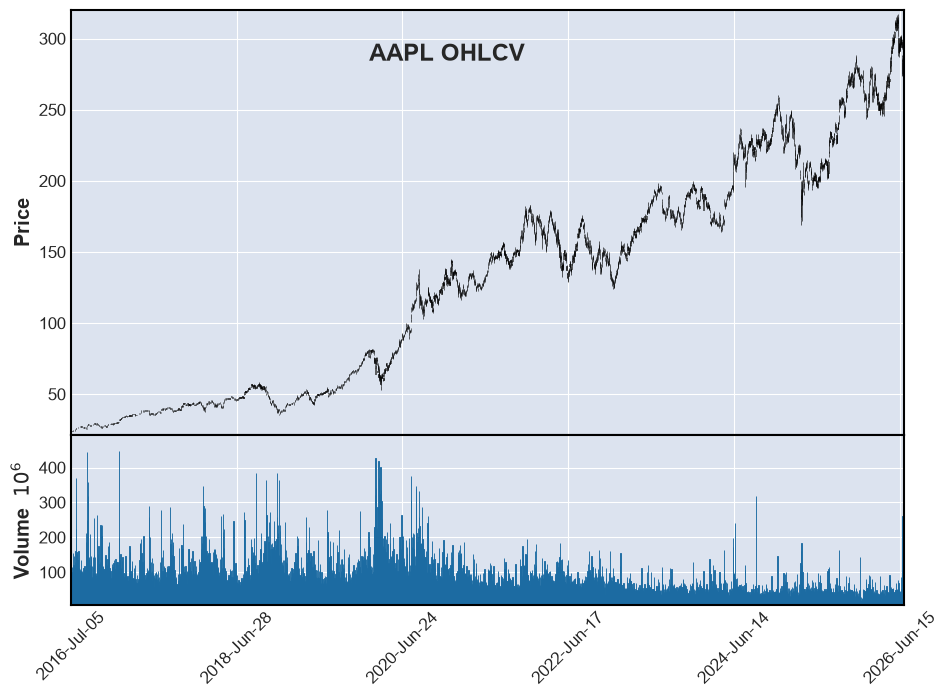}\caption{AAPL}\end{subfigure}\hfill
	\begin{subfigure}[b]{0.155\textwidth}\includegraphics[width=\linewidth]{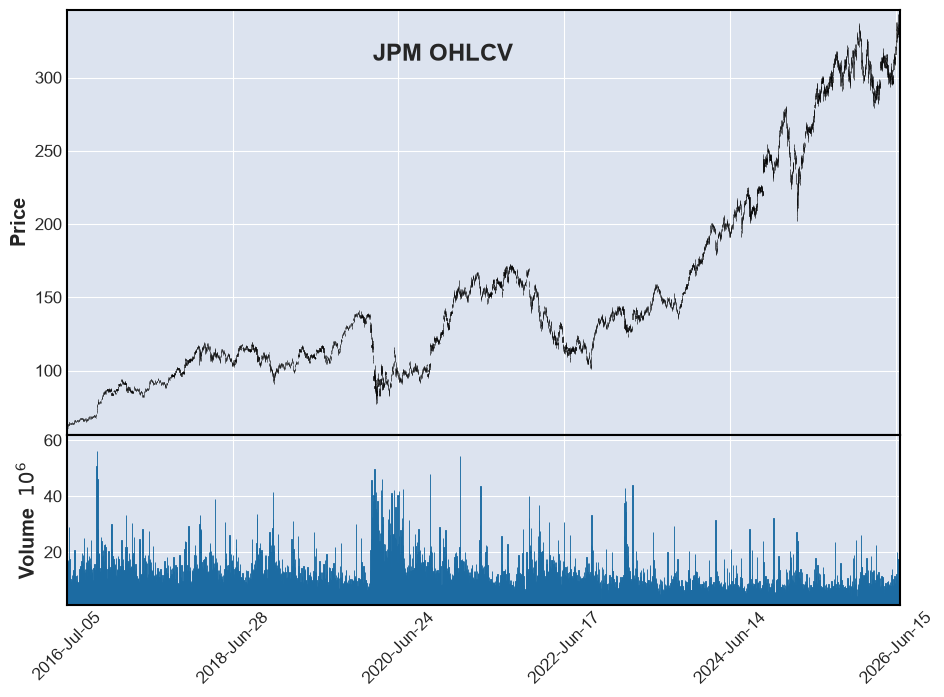}\caption{JPM}\end{subfigure}\hfill
	\begin{subfigure}[b]{0.155\textwidth}\includegraphics[width=\linewidth]{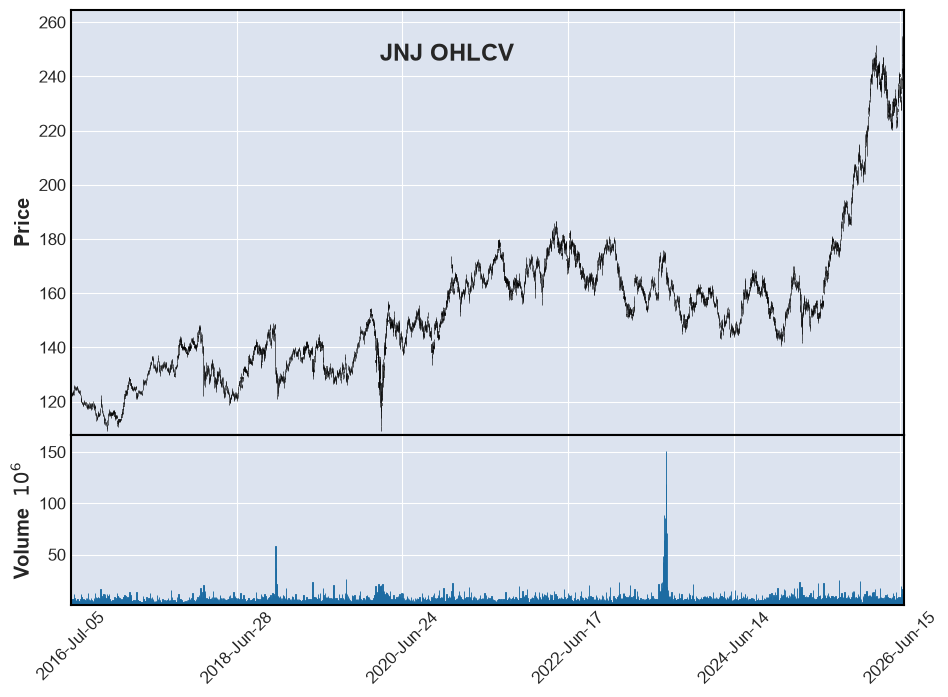}\caption{JNJ}\end{subfigure}\hfill
	\begin{subfigure}[b]{0.155\textwidth}\includegraphics[width=\linewidth]{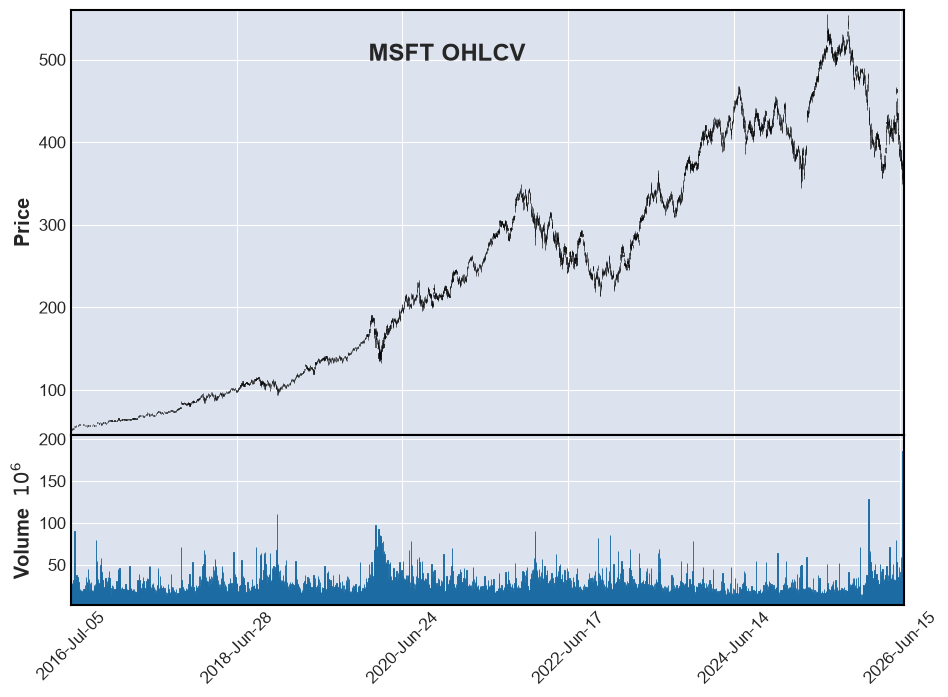}\caption{MSFT}\end{subfigure}
	
	\begin{subfigure}[b]{0.155\textwidth}\includegraphics[width=\linewidth]{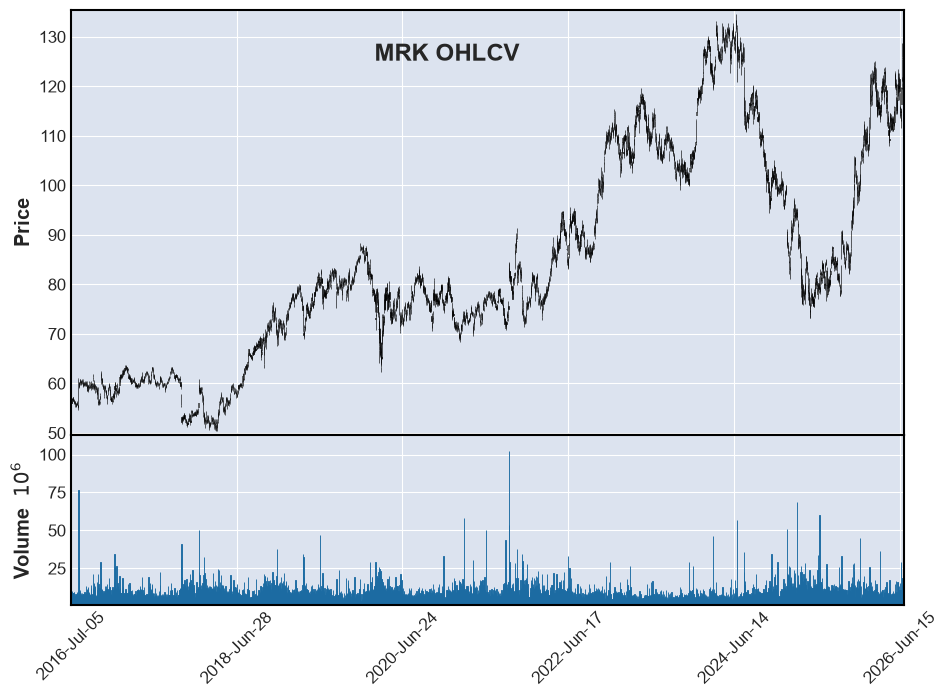}\caption{MRK}\end{subfigure}\hfill
	\begin{subfigure}[b]{0.155\textwidth}\includegraphics[width=\linewidth]{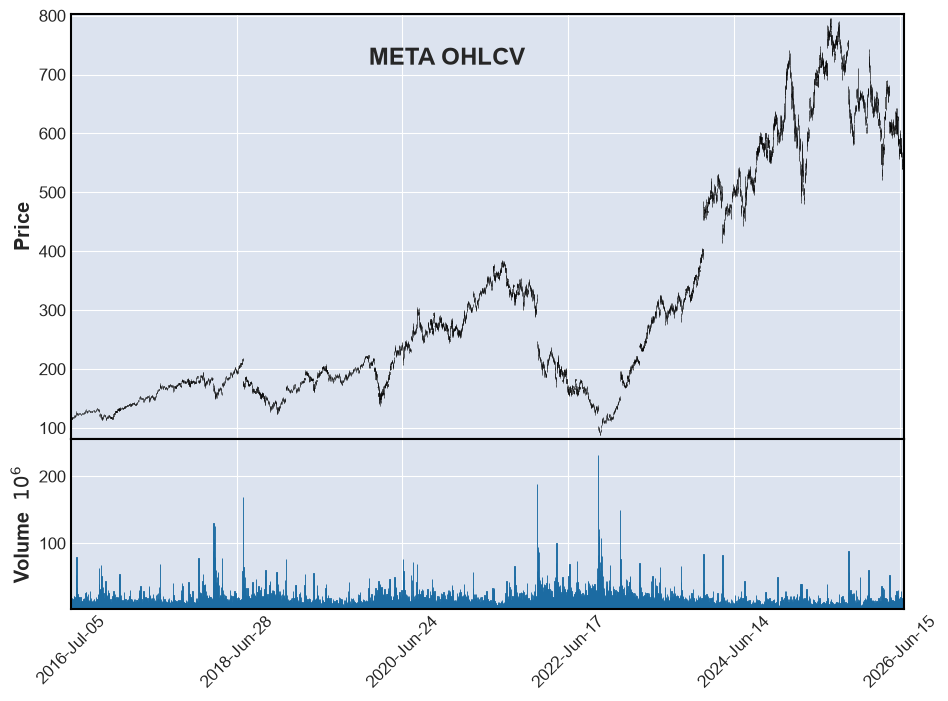}\caption{META}\end{subfigure}\hfill
	\begin{subfigure}[b]{0.155\textwidth}\includegraphics[width=\linewidth]{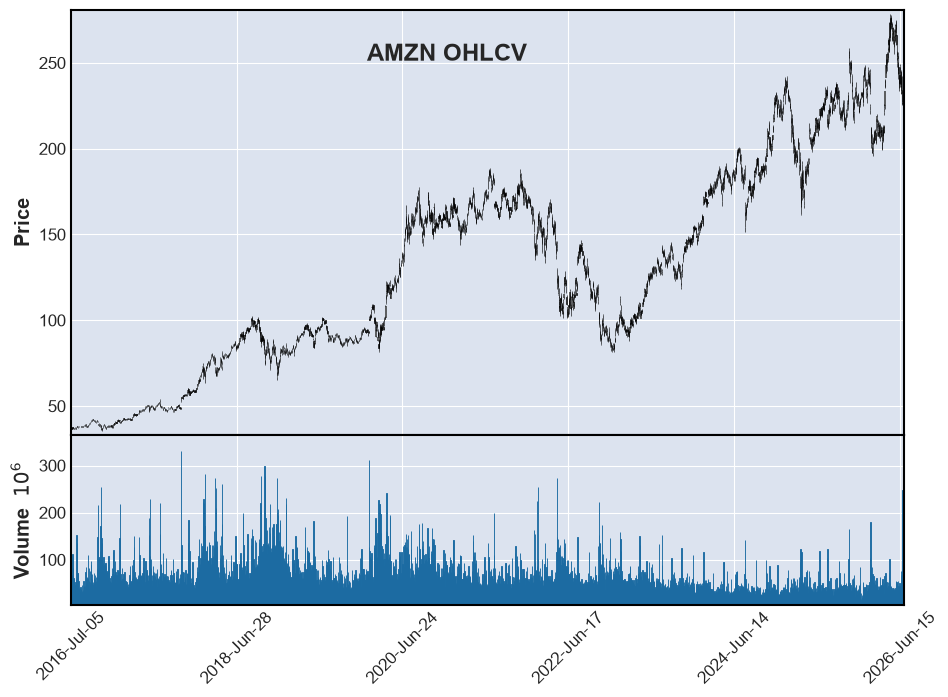}\caption{AMZN}\end{subfigure}\hfill
	\begin{subfigure}[b]{0.155\textwidth}\includegraphics[width=\linewidth]{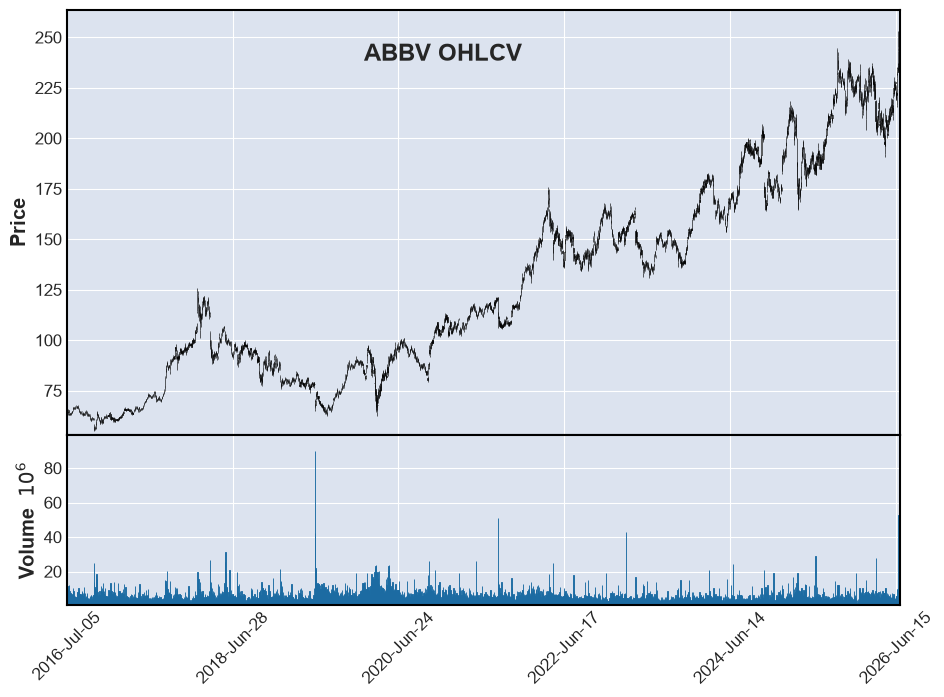}\caption{ABBV}\end{subfigure}\hfill
	\begin{subfigure}[b]{0.155\textwidth}\includegraphics[width=\linewidth]{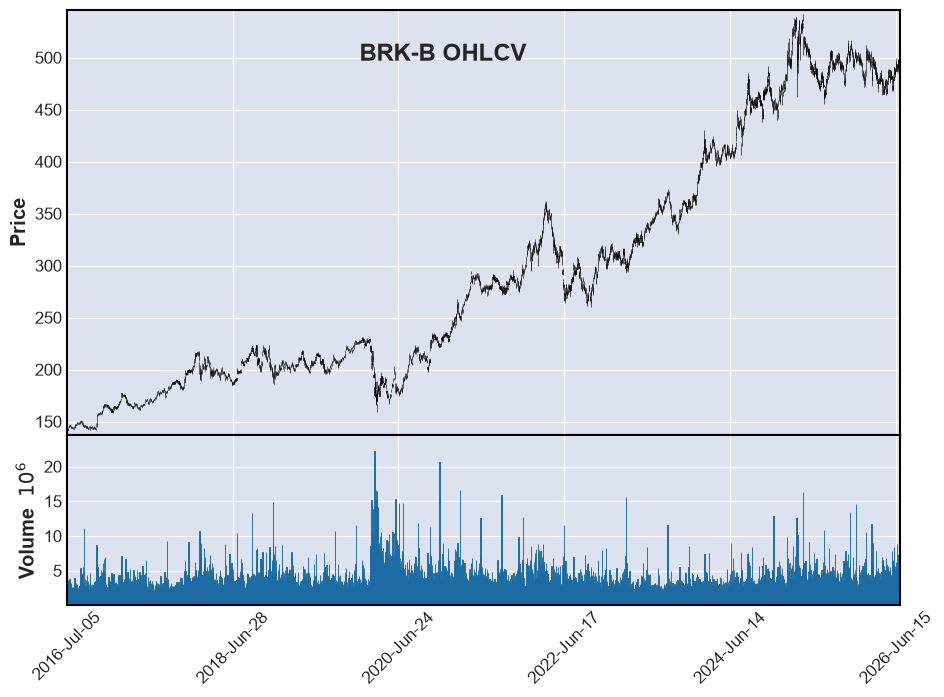}\caption{BRK-B}\end{subfigure}\hfill
	\begin{subfigure}[b]{0.155\textwidth}\includegraphics[width=\linewidth]{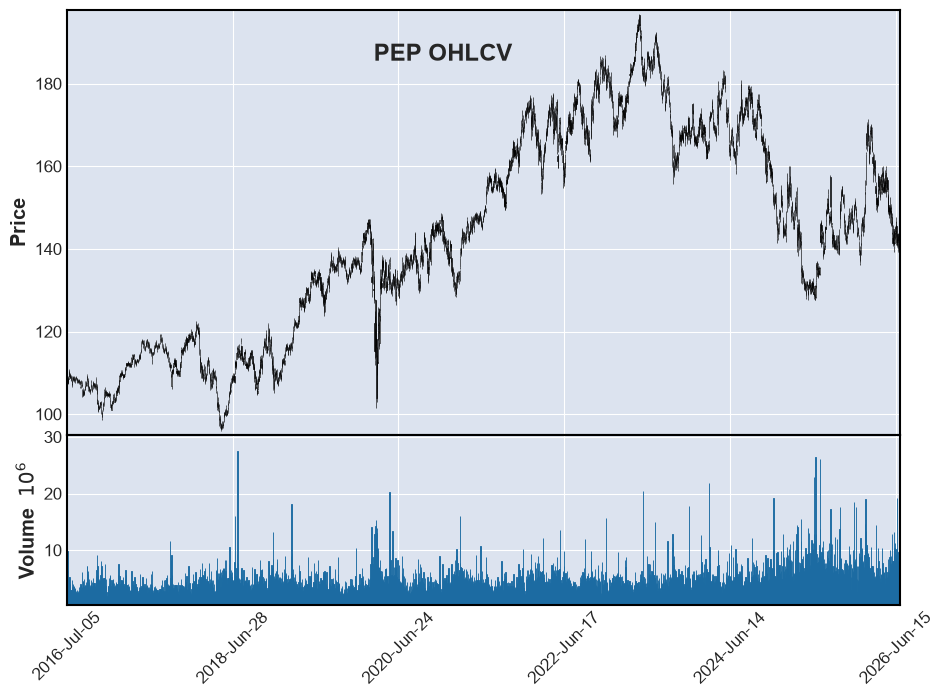}\caption{PEP}\end{subfigure}
	\caption{Daily OHLC price histories for the 30 stock symbols used in the experiments.}
	\label{fig:ohlc_stock}
\end{figure*}

\begin{figure*}[htbp]
	\centering
	\scriptsize
	\begin{subfigure}[b]{0.155\textwidth}\includegraphics[width=\linewidth]{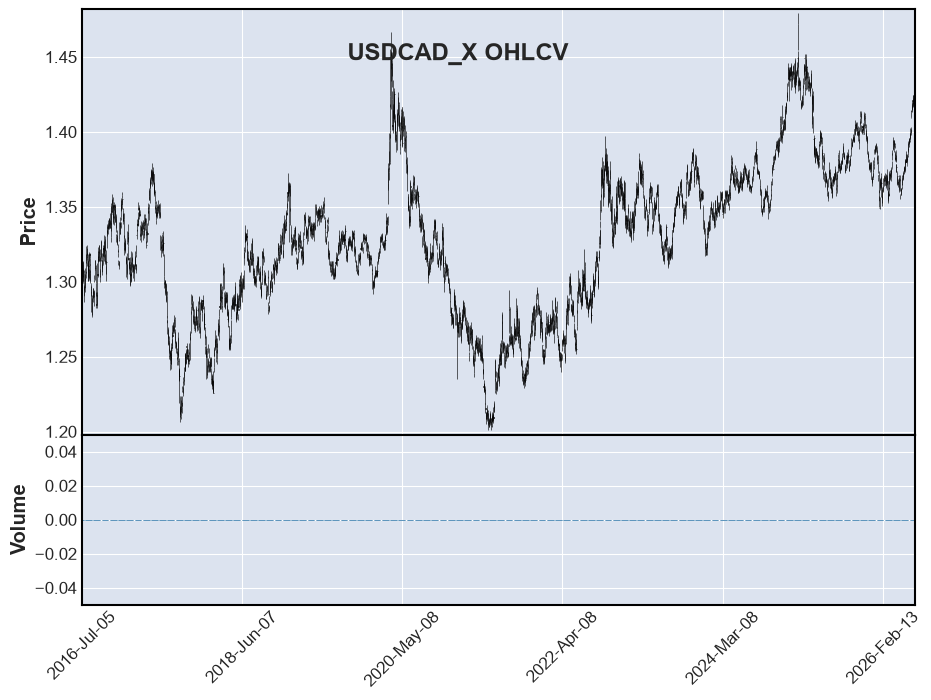}\caption{USDCAD}\end{subfigure}\hfill
	\begin{subfigure}[b]{0.155\textwidth}\includegraphics[width=\linewidth]{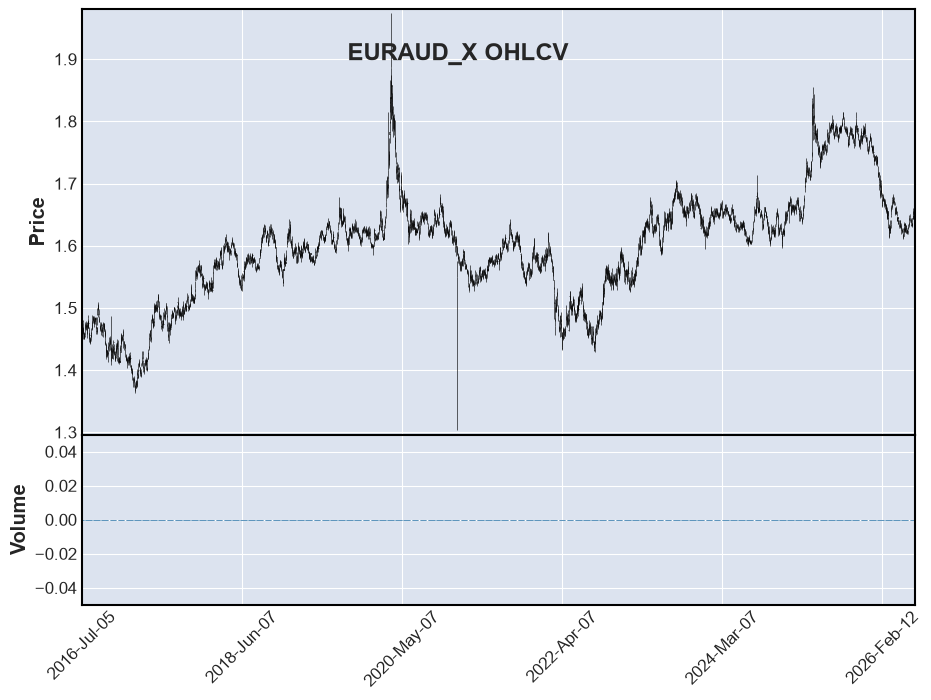}\caption{EURAUD}\end{subfigure}\hfill
	\begin{subfigure}[b]{0.155\textwidth}\includegraphics[width=\linewidth]{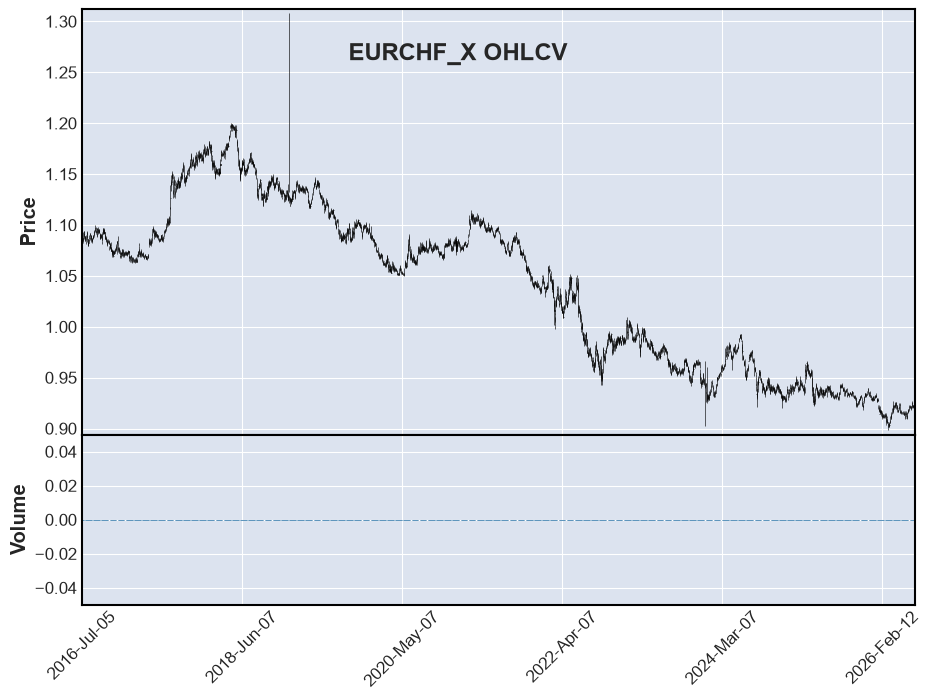}\caption{EURCHF}\end{subfigure}\hfill
	\begin{subfigure}[b]{0.155\textwidth}\includegraphics[width=\linewidth]{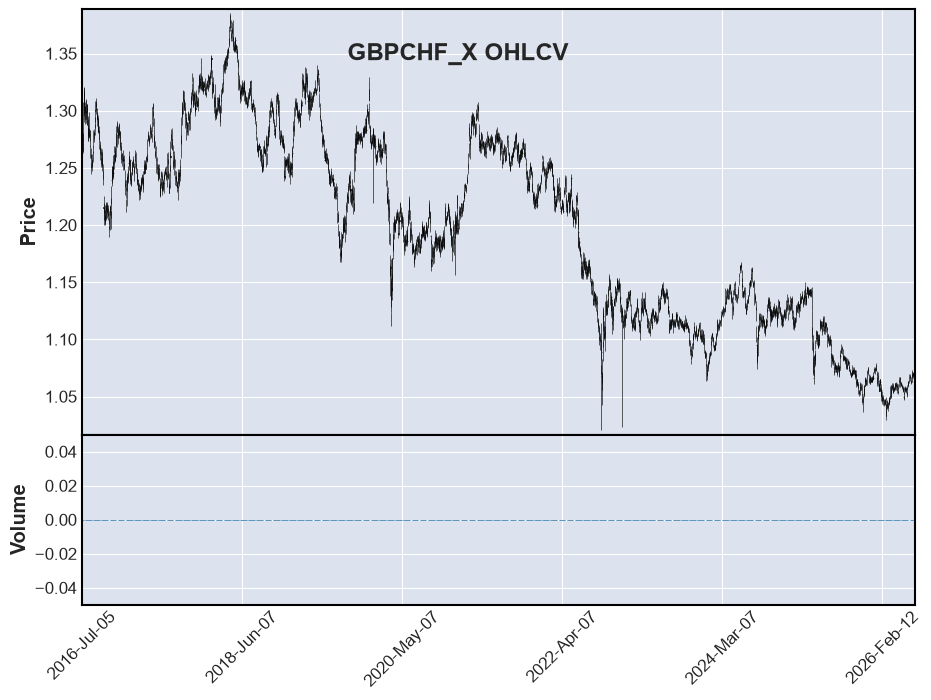}\caption{GBPCHF}\end{subfigure}\hfill
	\begin{subfigure}[b]{0.155\textwidth}\includegraphics[width=\linewidth]{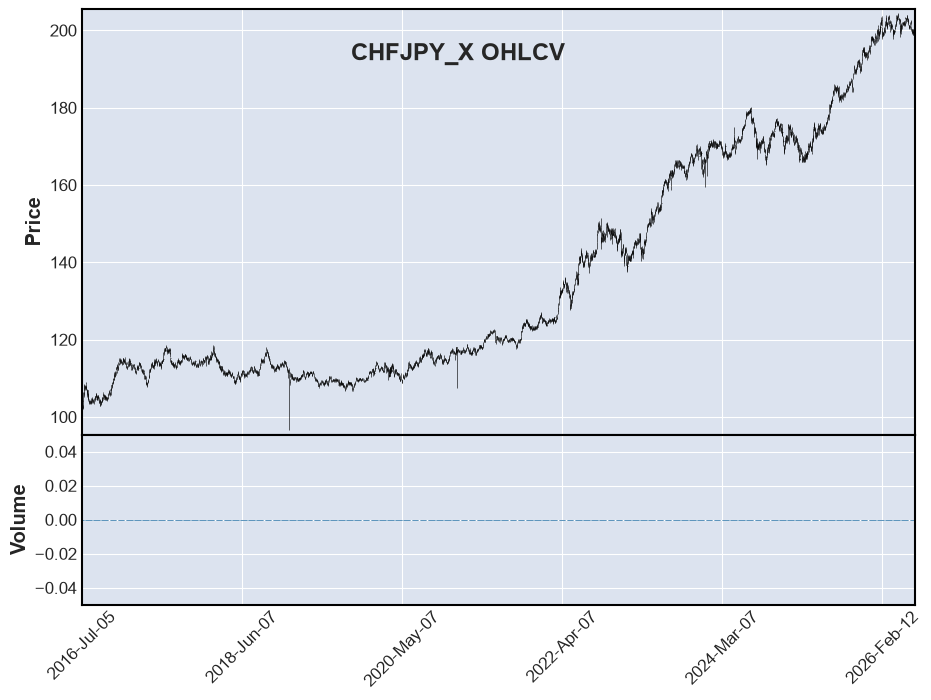}\caption{CHFJPY}\end{subfigure}\hfill
	\begin{subfigure}[b]{0.155\textwidth}\includegraphics[width=\linewidth]{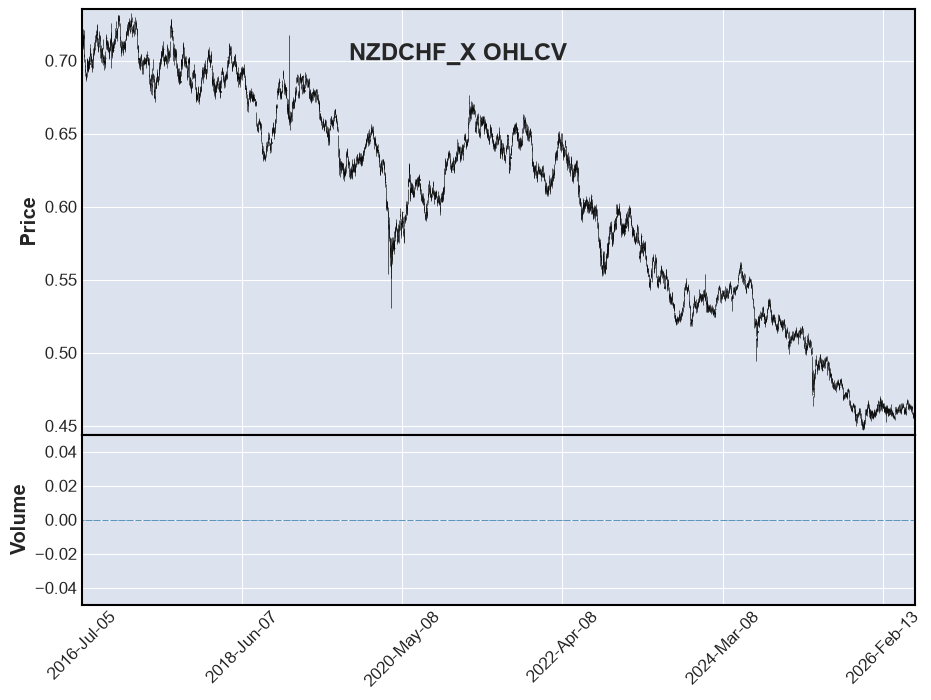}\caption{NZDCHF}\end{subfigure}
	
	\begin{subfigure}[b]{0.155\textwidth}\includegraphics[width=\linewidth]{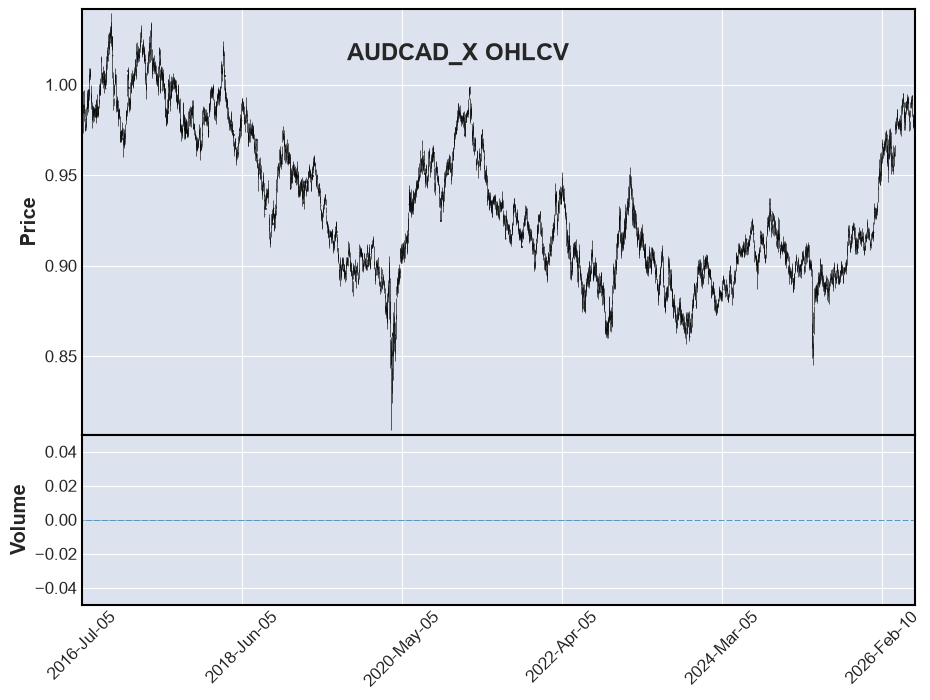}\caption{AUDCAD}\end{subfigure}\hfill
	\begin{subfigure}[b]{0.155\textwidth}\includegraphics[width=\linewidth]{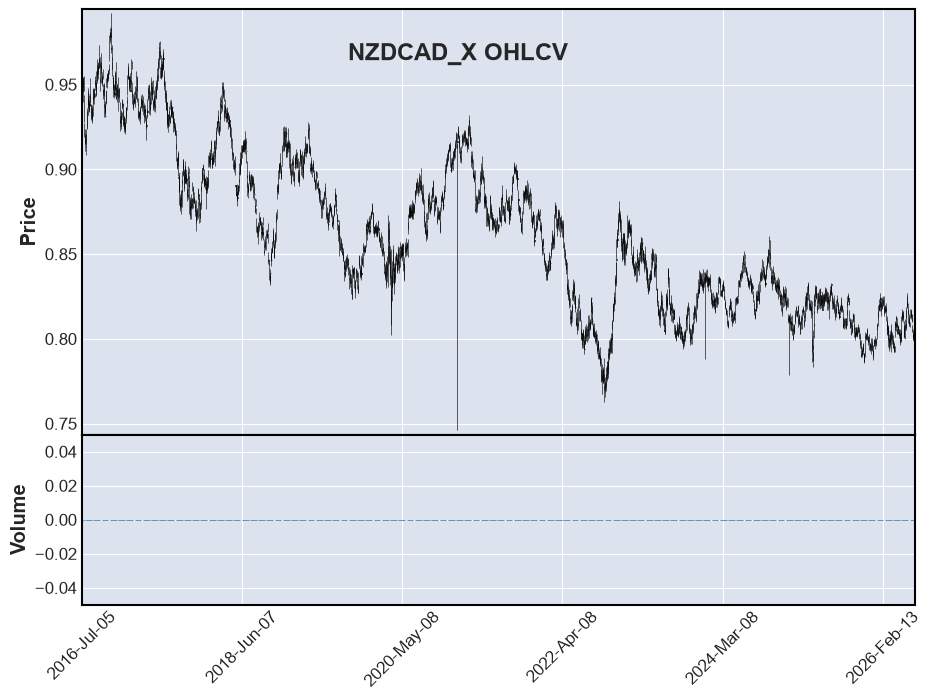}\caption{NZDCAD}\end{subfigure}\hfill
	\begin{subfigure}[b]{0.155\textwidth}\includegraphics[width=\linewidth]{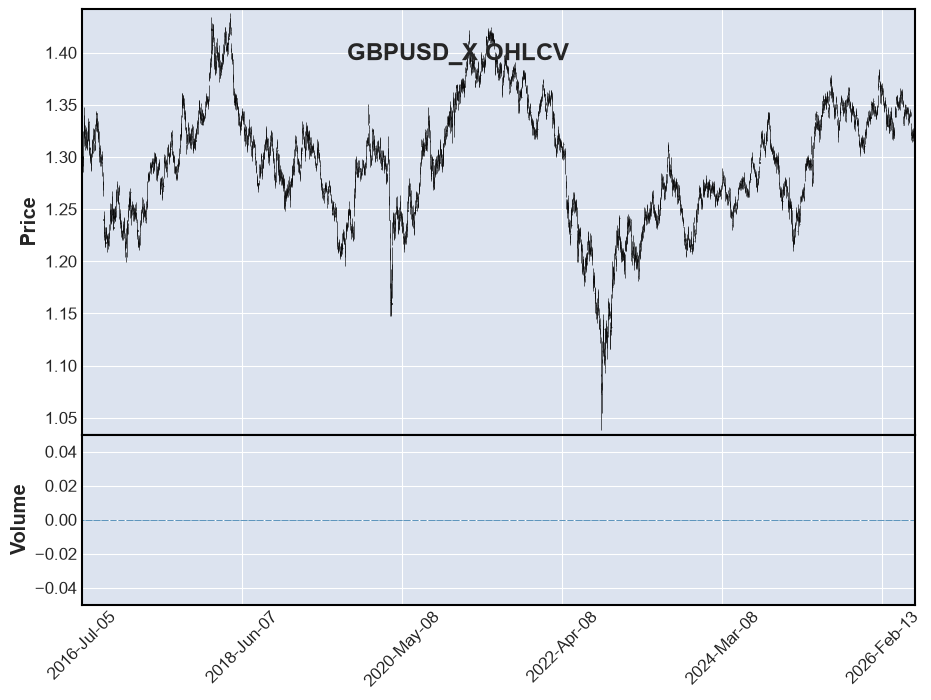}\caption{GBPUSD}\end{subfigure}\hfill
	\begin{subfigure}[b]{0.155\textwidth}\includegraphics[width=\linewidth]{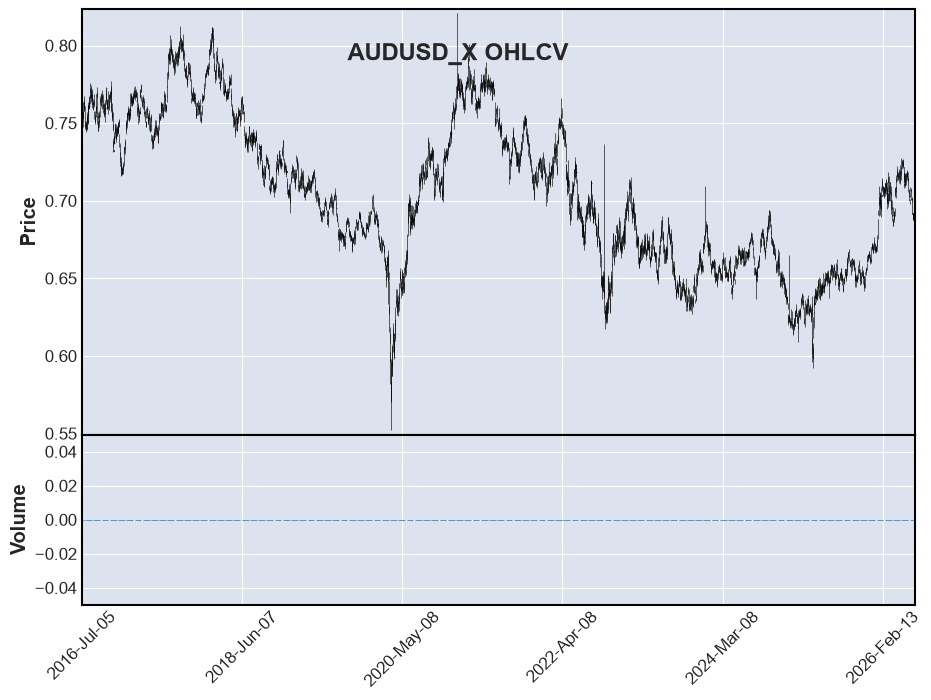}\caption{AUDUSD}\end{subfigure}\hfill
	\begin{subfigure}[b]{0.155\textwidth}\includegraphics[width=\linewidth]{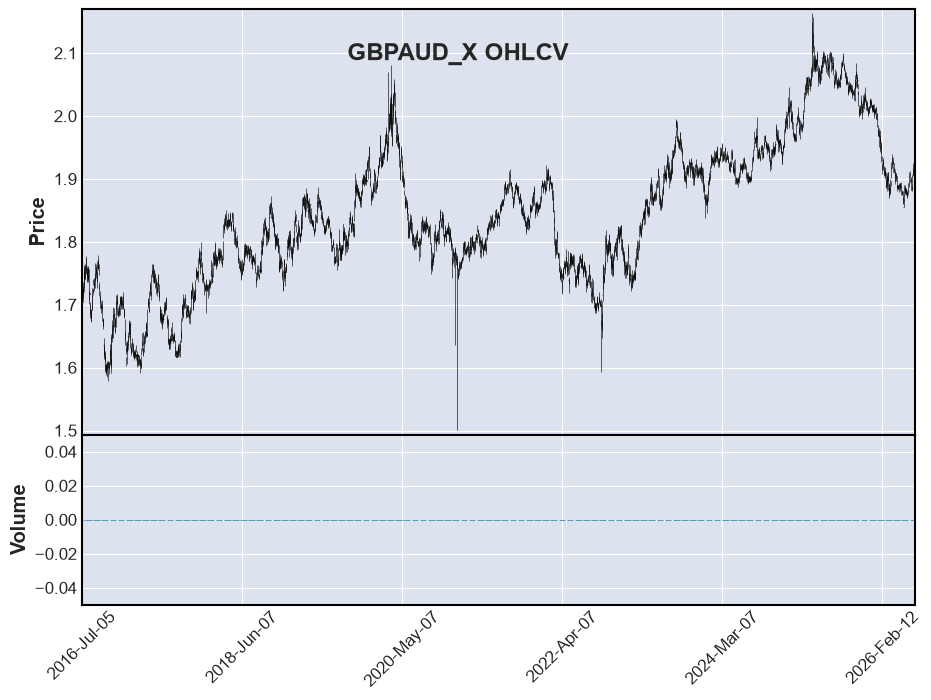}\caption{GBPAUD}\end{subfigure}\hfill
	\begin{subfigure}[b]{0.155\textwidth}\includegraphics[width=\linewidth]{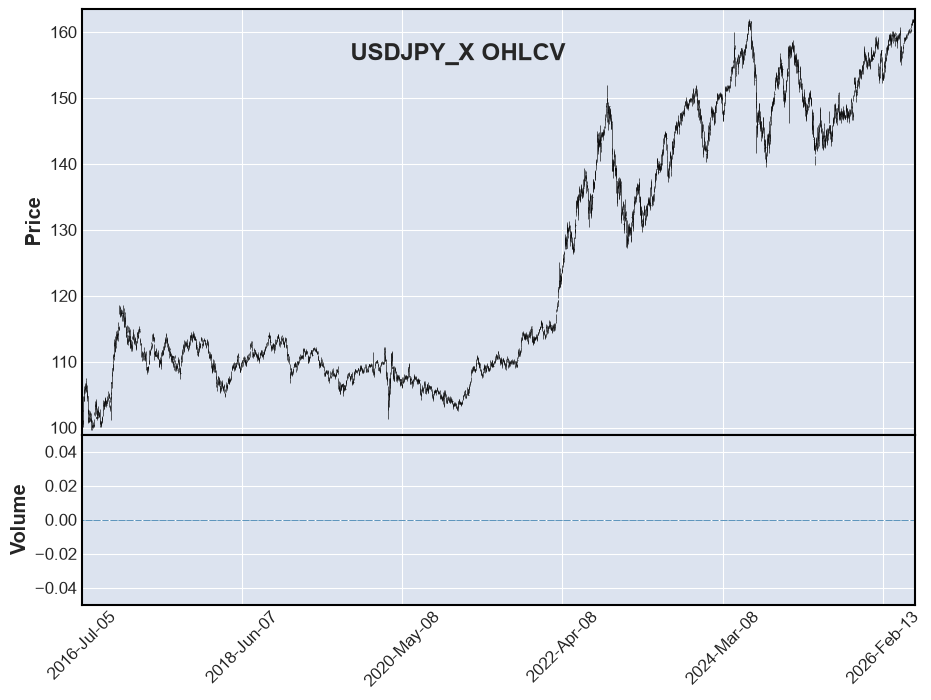}\caption{USDJPY}\end{subfigure}
	
	\begin{subfigure}[b]{0.155\textwidth}\includegraphics[width=\linewidth]{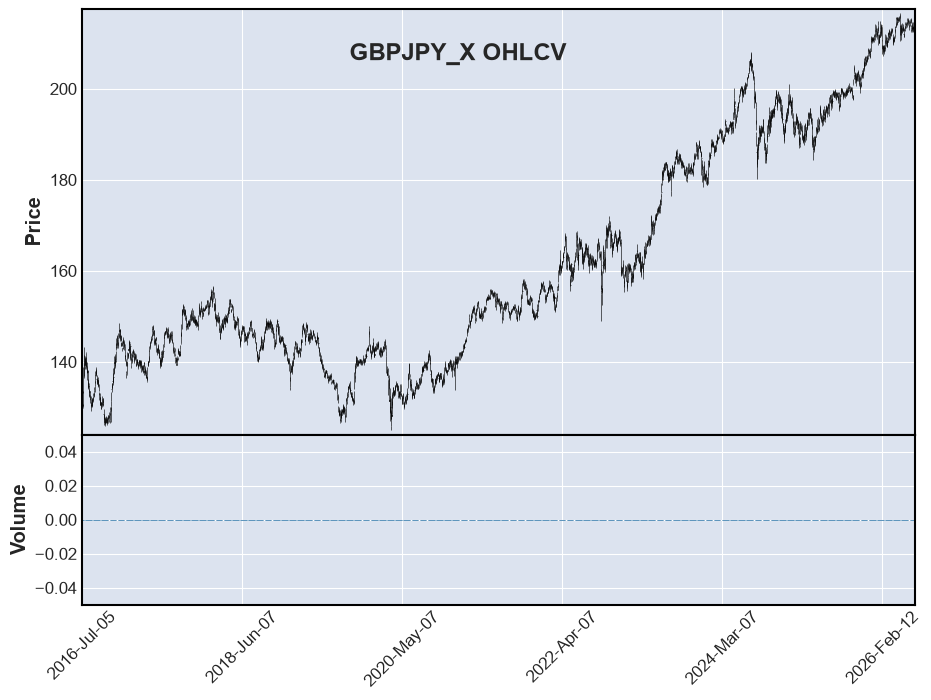}\caption{GBPJPY}\end{subfigure}\hfill
	\begin{subfigure}[b]{0.155\textwidth}\includegraphics[width=\linewidth]{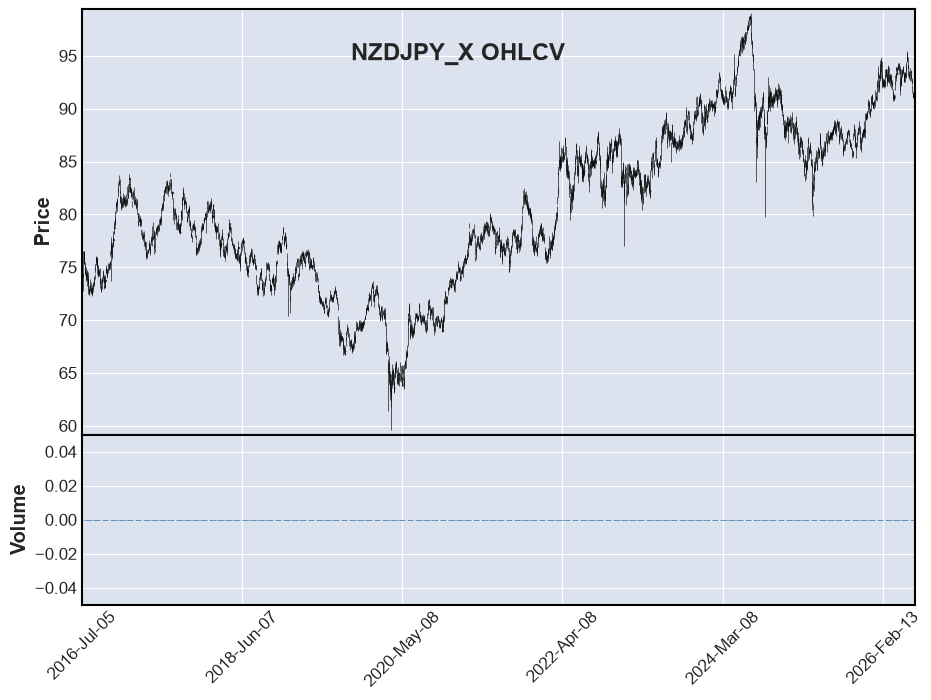}\caption{NZDJPY}\end{subfigure}\hfill
	\begin{subfigure}[b]{0.155\textwidth}\includegraphics[width=\linewidth]{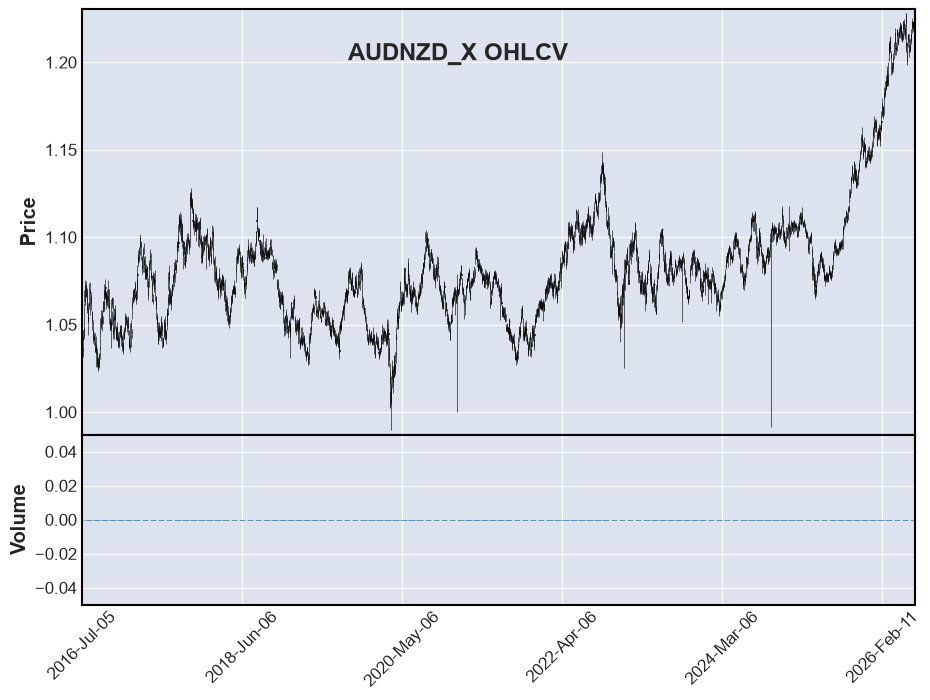}\caption{AUDNZD}\end{subfigure}\hfill
	\begin{subfigure}[b]{0.155\textwidth}\includegraphics[width=\linewidth]{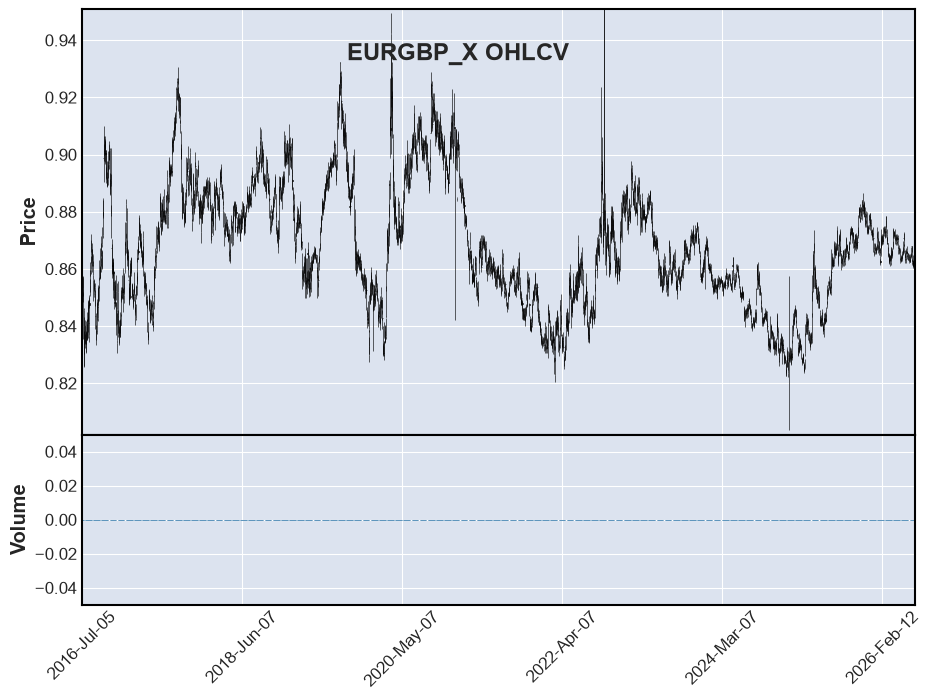}\caption{EURGBP}\end{subfigure}\hfill
	\begin{subfigure}[b]{0.155\textwidth}\includegraphics[width=\linewidth]{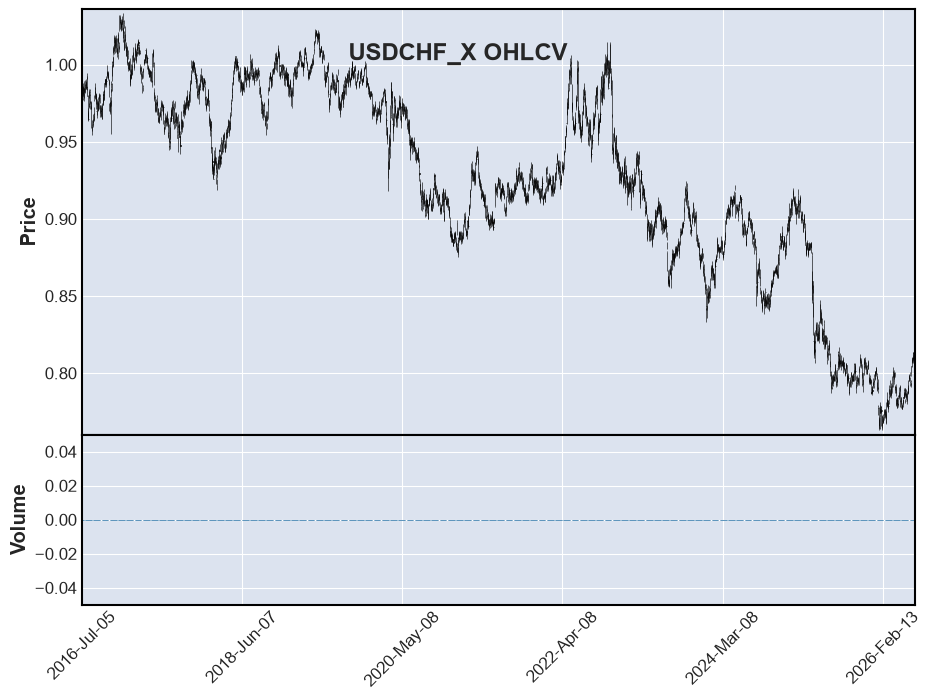}\caption{USDCHF}\end{subfigure}\hfill
	\begin{subfigure}[b]{0.155\textwidth}\includegraphics[width=\linewidth]{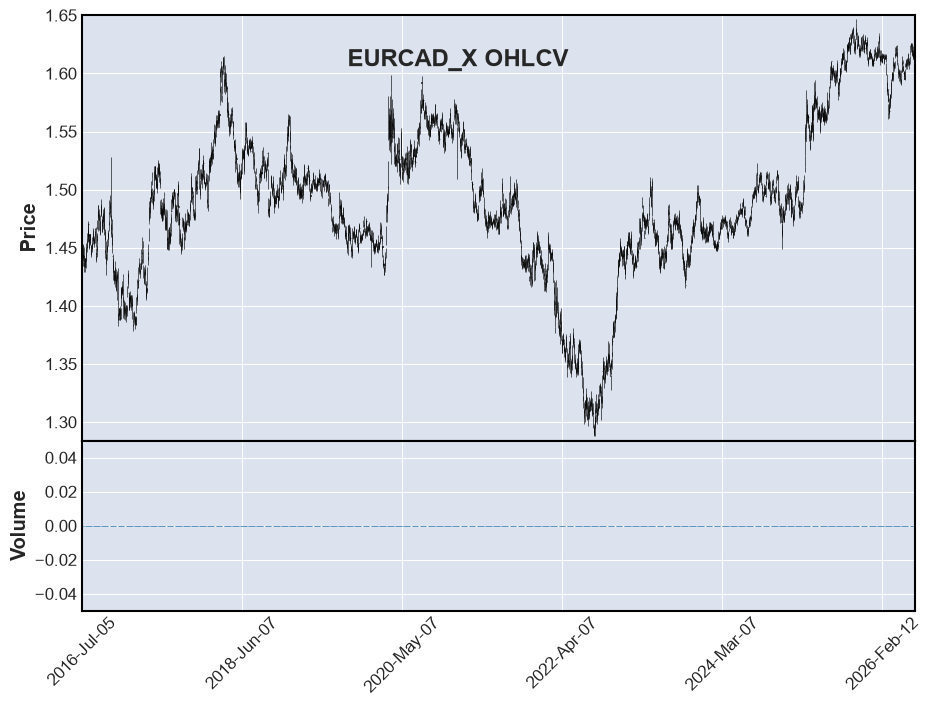}\caption{EURCAD}\end{subfigure}
	
	\begin{subfigure}[b]{0.155\textwidth}\includegraphics[width=\linewidth]{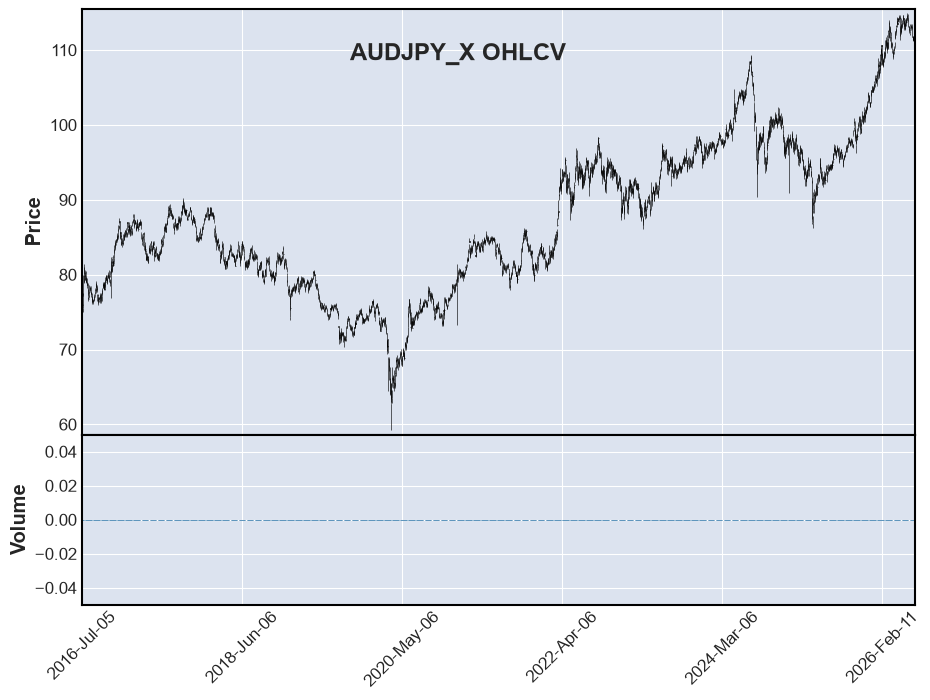}\caption{AUDJPY}\end{subfigure}\hfill
	\begin{subfigure}[b]{0.155\textwidth}\includegraphics[width=\linewidth]{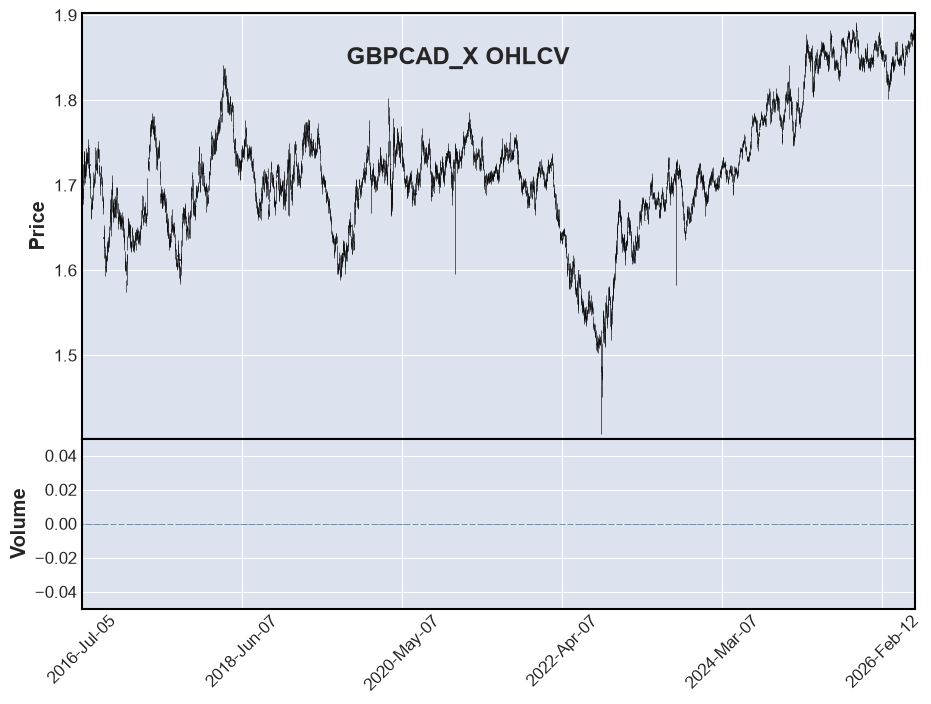}\caption{GBPCAD}\end{subfigure}\hfill
	\begin{subfigure}[b]{0.155\textwidth}\includegraphics[width=\linewidth]{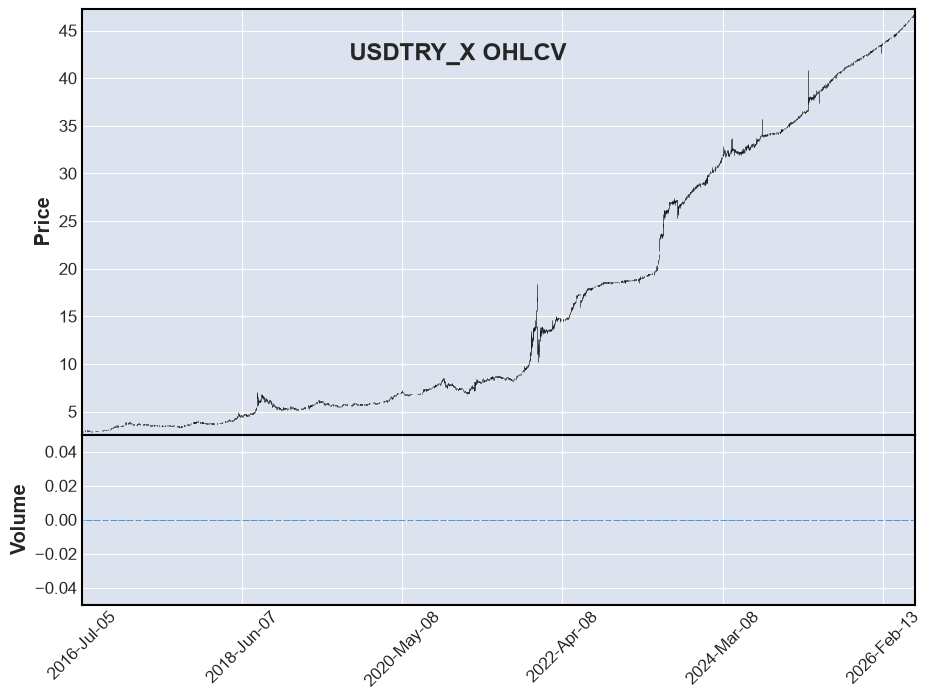}\caption{USDTRY}\end{subfigure}\hfill
	\begin{subfigure}[b]{0.155\textwidth}\includegraphics[width=\linewidth]{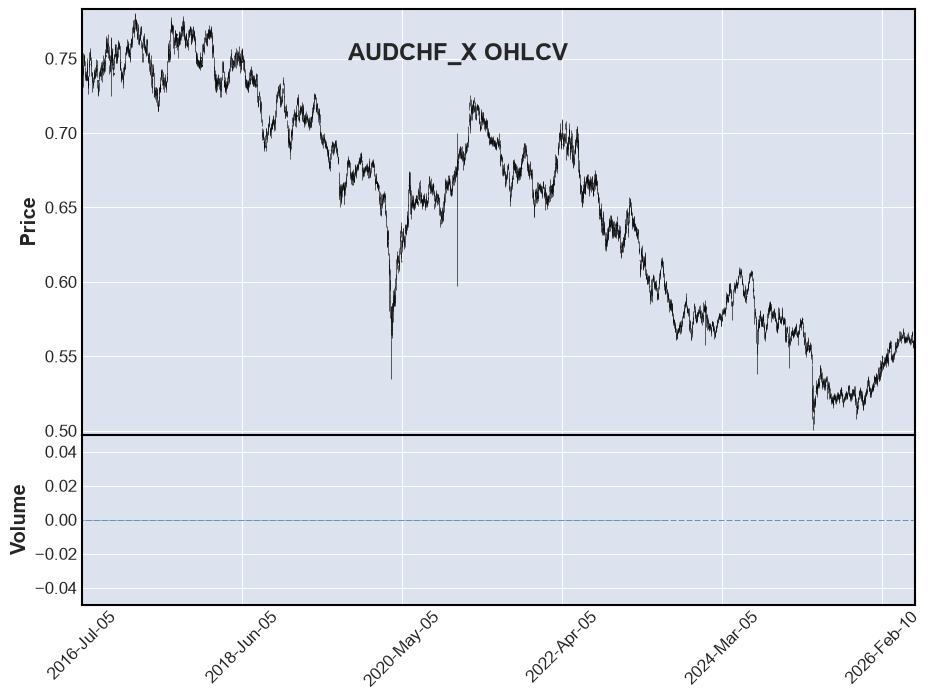}\caption{AUDCHF}\end{subfigure}\hfill
	\begin{subfigure}[b]{0.155\textwidth}\includegraphics[width=\linewidth]{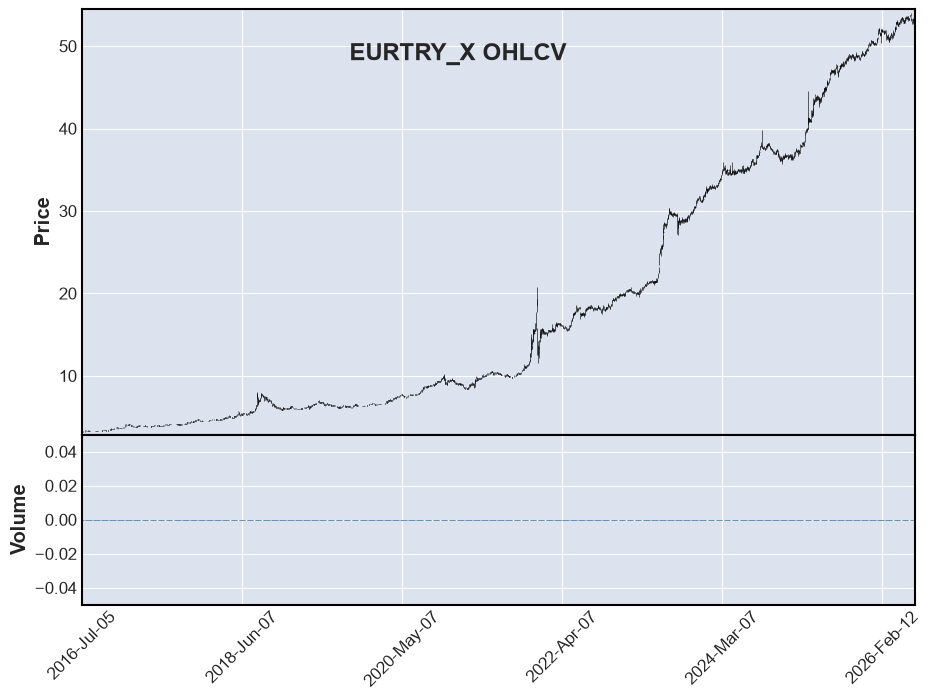}\caption{EURTRY}\end{subfigure}\hfill
	\begin{subfigure}[b]{0.155\textwidth}\includegraphics[width=\linewidth]{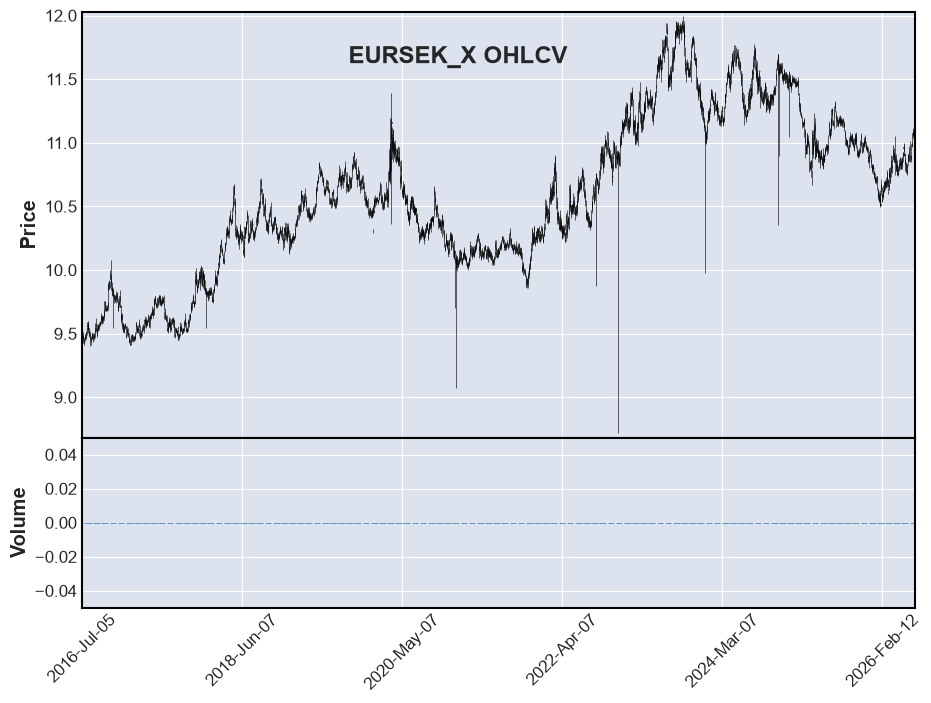}\caption{EURSEK}\end{subfigure}
	
	\begin{subfigure}[b]{0.155\textwidth}\includegraphics[width=\linewidth]{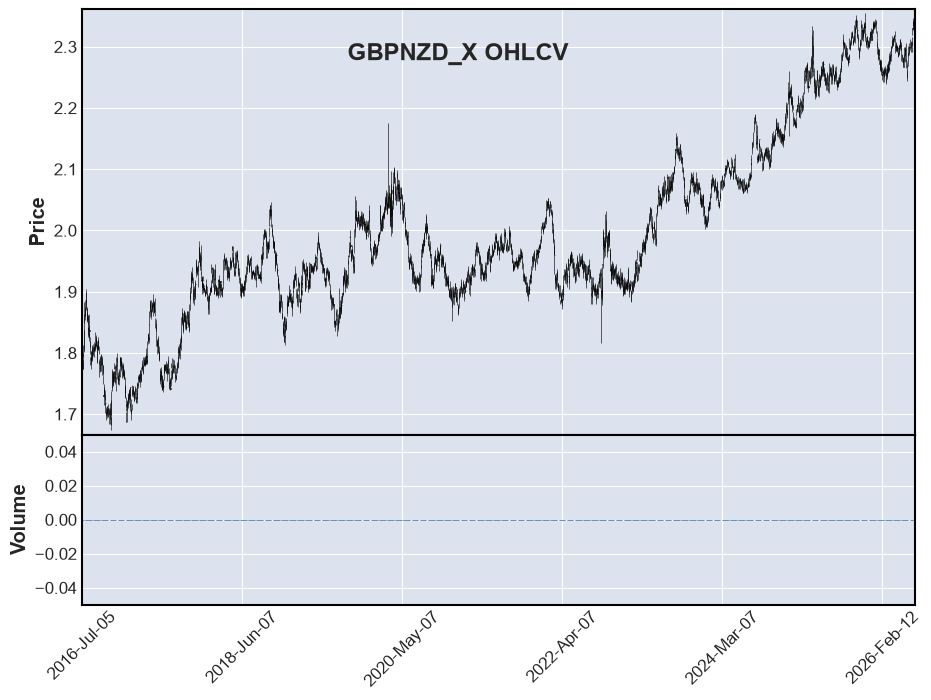}\caption{GBPNZD}\end{subfigure}\hfill
	\begin{subfigure}[b]{0.155\textwidth}\includegraphics[width=\linewidth]{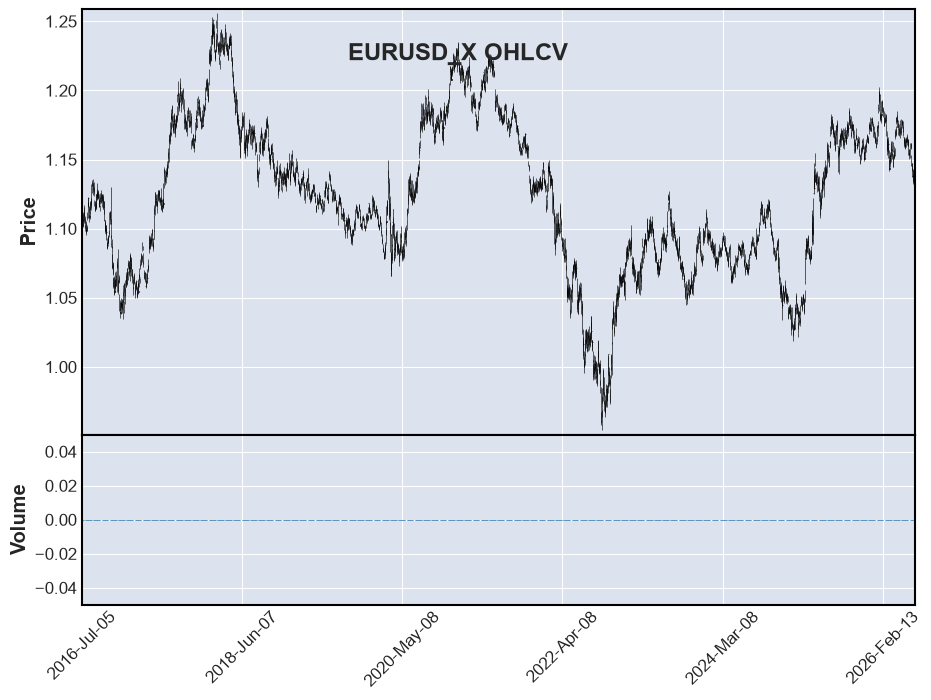}\caption{EURUSD}\end{subfigure}\hfill
	\begin{subfigure}[b]{0.155\textwidth}\includegraphics[width=\linewidth]{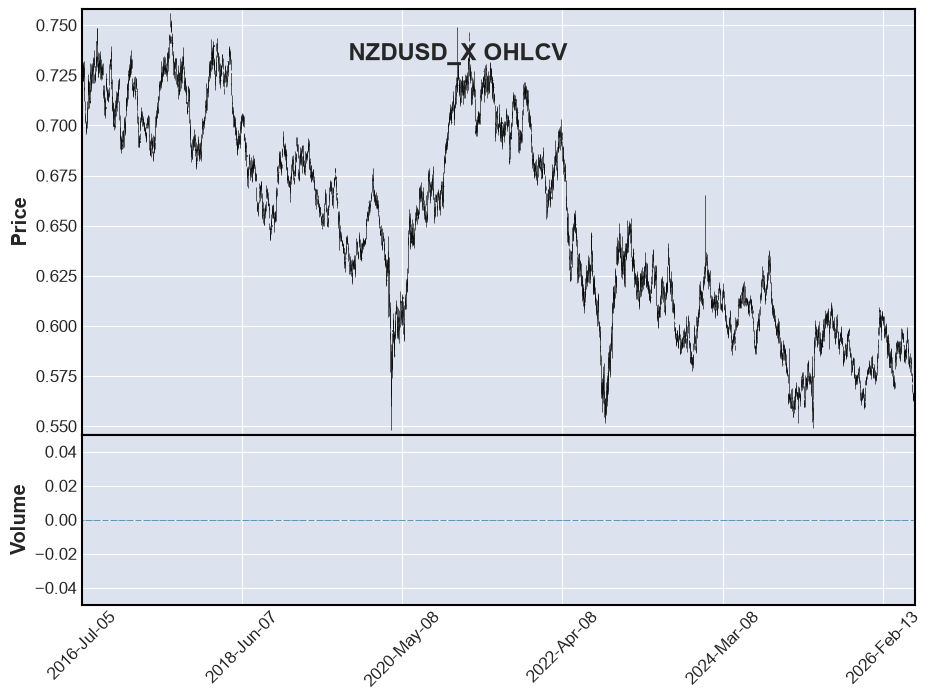}\caption{NZDUSD}\end{subfigure}\hfill
	\begin{subfigure}[b]{0.155\textwidth}\includegraphics[width=\linewidth]{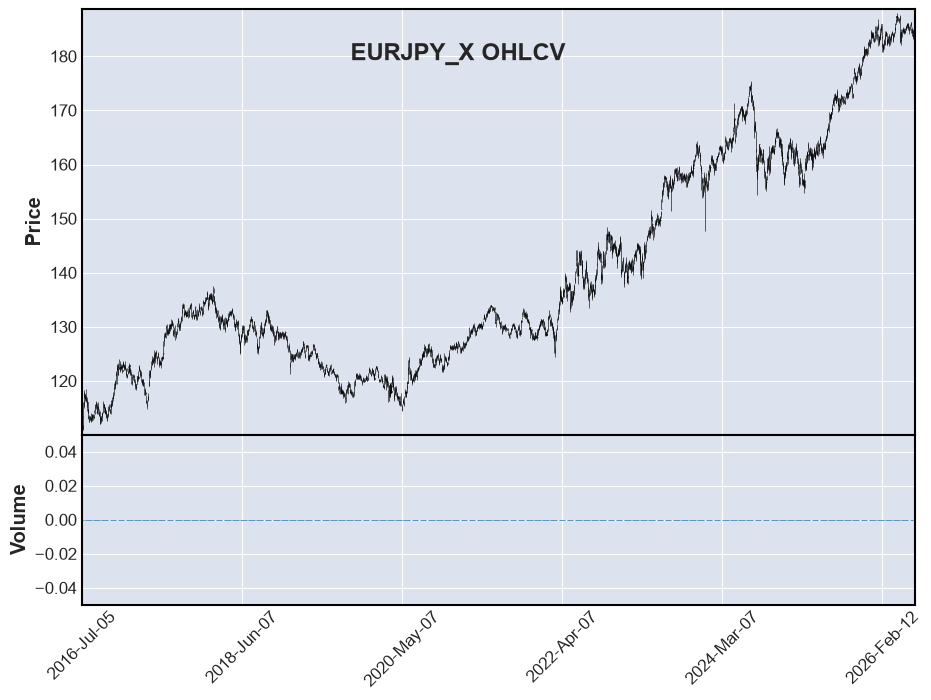}\caption{EURJPY}\end{subfigure}\hfill
	\begin{subfigure}[b]{0.155\textwidth}\includegraphics[width=\linewidth]{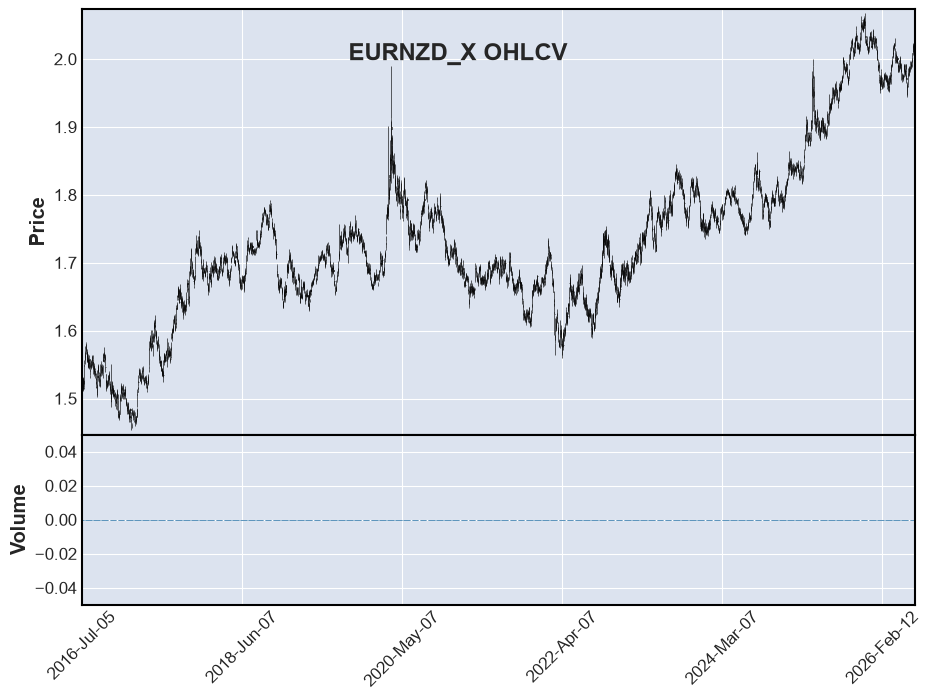}\caption{EURNZD}\end{subfigure}\hfill
	\begin{subfigure}[b]{0.155\textwidth}\includegraphics[width=\linewidth]{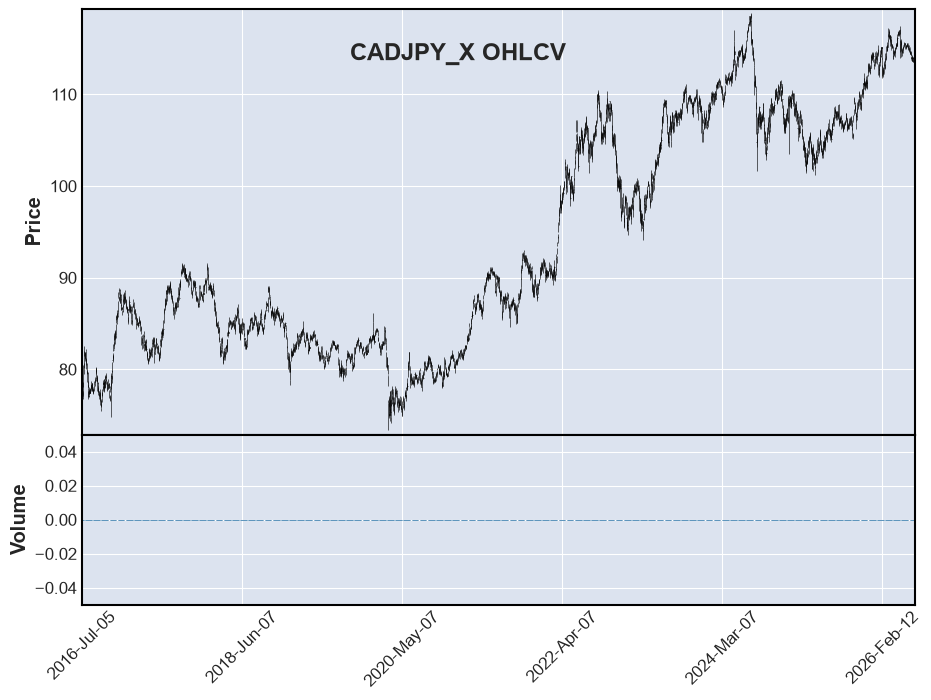}\caption{CADJPY}\end{subfigure}
	\caption{Daily OHLC price histories for the 30 foreign-exchange symbols used in the experiments.}
	\label{fig:ohlc_forex}
\end{figure*}

\paragraph{Market indices}
To provide a market-level reference, we report the market index series associated with each experiment in Figure~\ref{fig:index_crypto} and Figure~\ref{fig:index_stock}. Because the cryptocurrency market has no single standard benchmark, the cryptocurrency panel uses the two synthetic indices defined in the experimental protocol: an equal-weighted (unweighted) cumulative-return index and a volume-weighted cumulative-return index. As shown in Figure~\ref{fig:index_crypto}, the volume-weighted index grows gradually from 2020 onward and traces the characteristic cycle of the market, with a 2021 bull run, a 2022 drawdown, and a strong 2024--2025 rally followed by a partial reversal. The unweighted index instead remains close to its initial level until late 2024, after which it exhibits an extreme surge and a subsequent sharp correction, indicating that equal-weighted aggregation over raw prices is dominated by a small number of extreme constituents and motivating the use of the volume-weighted variant as the primary market-level descriptor. For the stock-market experiments, we report three broad United States equity references, namely the S\&P~500, the NASDAQ Composite, and the Dow Jones Industrial Average. All three indices trend upward over 2016--2026 and share the same major episodes, including the March 2020 pandemic crash, the 2022 bear market, and the post-2023 recovery, while the technology-heavy NASDAQ Composite displays the largest swing amplitude. No comparable composite series is reported for foreign exchange, as no standard broad-market index exists for the set of currency pairs under study; market-wide conditions in that market are instead captured by the aggregate feature stream of the state encoder.

\begin{figure}[htbp]
	\centering
	\begin{subfigure}[b]{0.48\textwidth}\includegraphics[width=\linewidth]{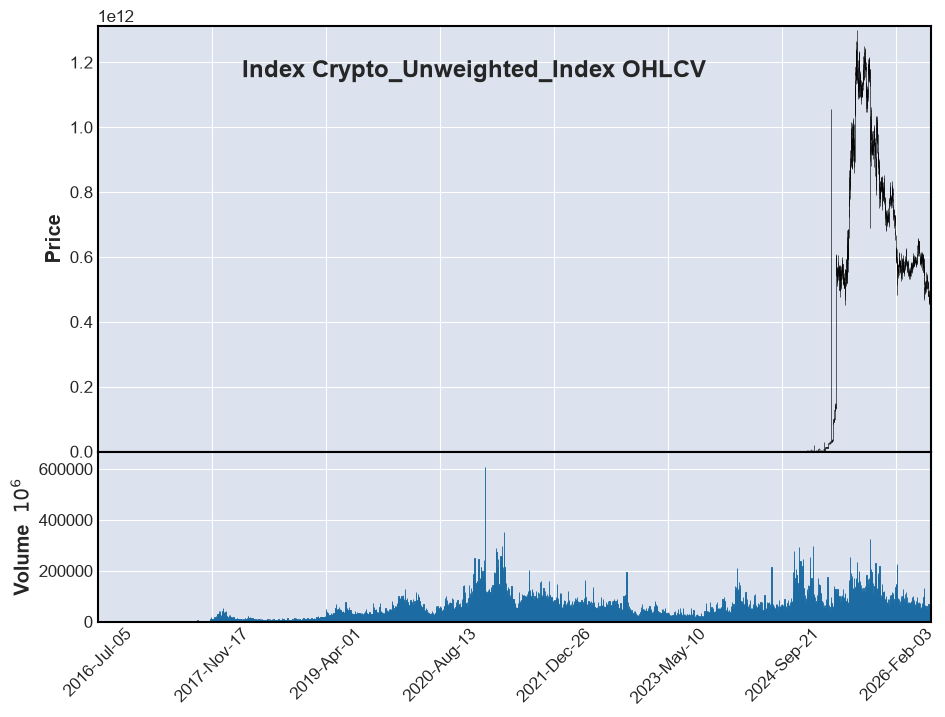}\caption{Equal-weighted (unweighted) index.}\end{subfigure}\hfill
	\begin{subfigure}[b]{0.48\textwidth}\includegraphics[width=\linewidth]{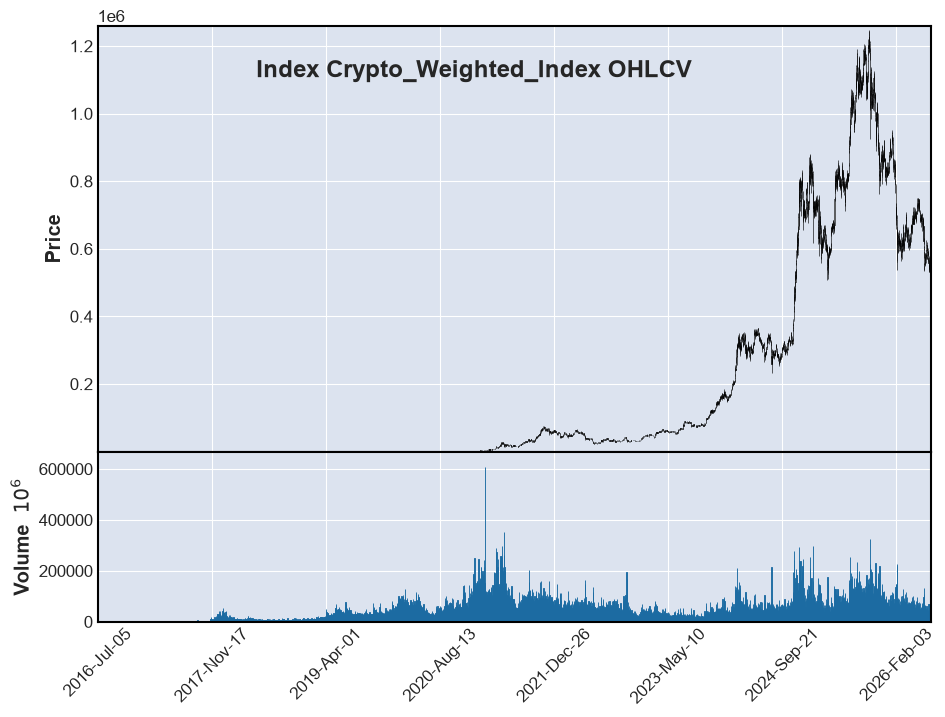}\caption{Volume-weighted index.}\end{subfigure}
	\caption{Market-level cryptocurrency indices.}
	\label{fig:index_crypto}
\end{figure}

\begin{figure}[htbp]
	\centering
	\begin{subfigure}[b]{0.32\textwidth}\includegraphics[width=\linewidth]{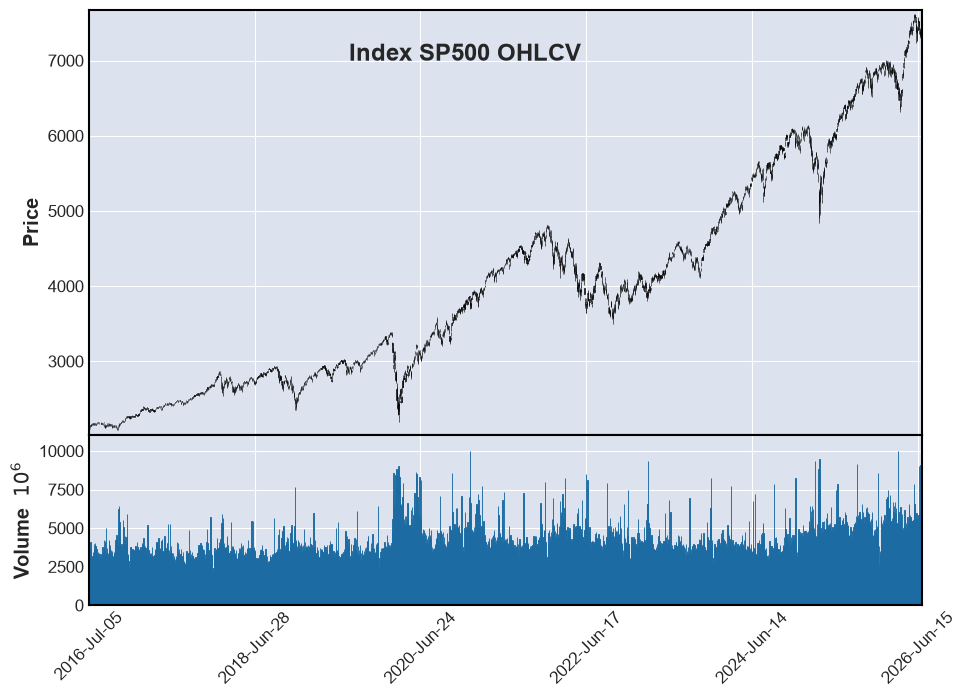}\caption{S\&P~500.}\end{subfigure}\hfill
	\begin{subfigure}[b]{0.32\textwidth}\includegraphics[width=\linewidth]{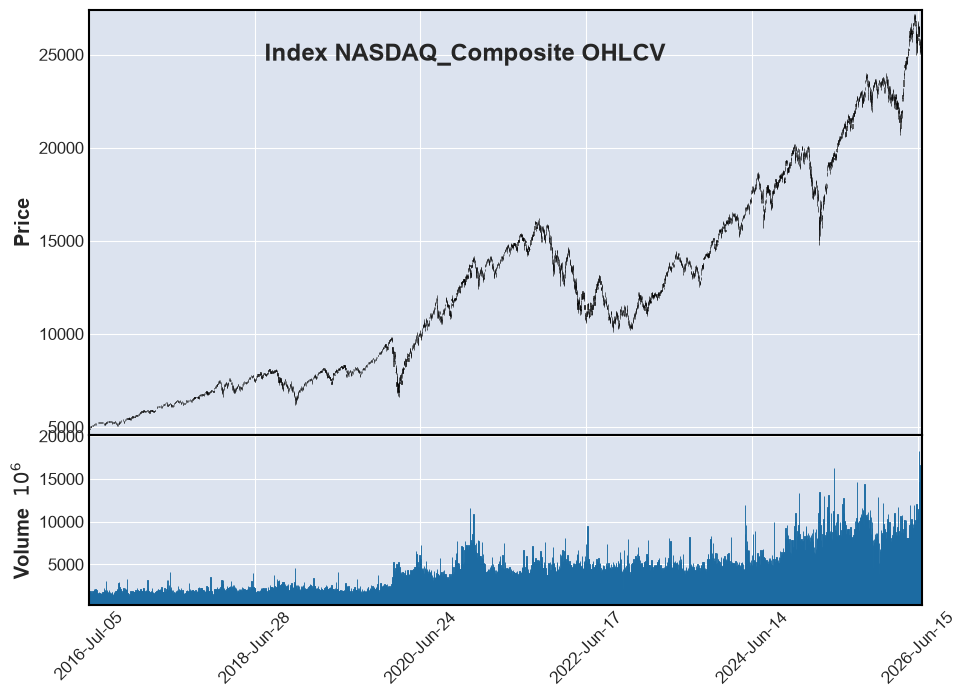}\caption{NASDAQ Composite.}\end{subfigure}\hfill
	\begin{subfigure}[b]{0.32\textwidth}\includegraphics[width=\linewidth]{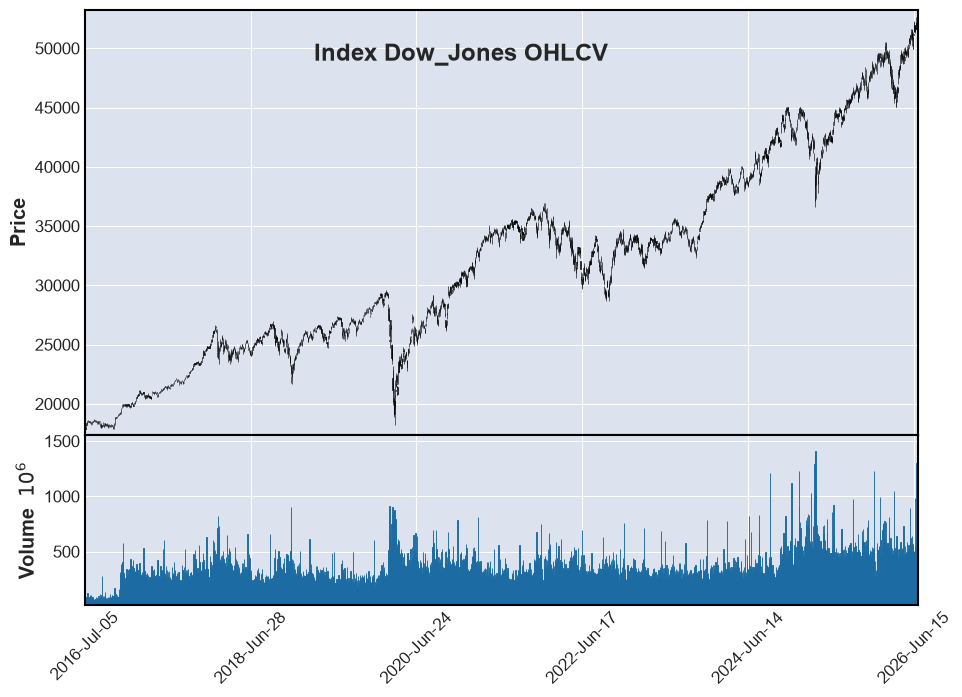}\caption{Dow Jones Average.}\end{subfigure}
	\caption{Market-level United States equity reference indices for the stock-market experiments.}
	\label{fig:index_stock}
\end{figure}

\paragraph{Expert pool}
We construct a pool of five heterogeneous experts. The purpose of the pool is to expose the switching mechanism to strategies with different inductive biases rather than to create a collection of near-identical predictors. The pool contains policy-gradient, actor-critic, distributional, and off-policy reinforcement-learning strategies, providing complementary behaviors for the adaptive selector.

\begin{enumerate}
	\item \textbf{A2C (Advantage Actor-Critic).}
	A synchronous actor-critic policy-gradient method that provides a relatively simple policy-learning baseline.
	\item \textbf{DDPG (Deep Deterministic Policy Gradient).}
	An off-policy actor-critic method designed for continuous action spaces.
	\item \textbf{PPO (Proximal Policy Optimization).}
	A policy-gradient method that constrains policy updates to improve training stability.
	\item \textbf{TQC (Truncated Quantile Critics).}
	A distributional reinforcement-learning method that models the reward distribution through quantile critics and provides a risk-sensitive candidate policy.
	\item \textbf{CrossQ.}
	An off-policy reinforcement-learning expert included to increase behavioral diversity within the candidate pool.
\end{enumerate}

\paragraph{Metrics}
We report cumulative return, annualized return, annualized volatility, Sharpe ratio, Sortino ratio, and maximum drawdown (MDD). Sharpe ratio measures return relative to total variability, whereas Sortino ratio focuses on downside variability and is therefore particularly informative when downside risk is of primary interest. MDD measures the largest peak-to-trough decline and is used to assess robustness under adverse market regimes. Because the three markets exhibit different statistical properties and volatility scales, the discussion emphasizes risk-adjusted measures together with cumulative return rather than interpreting any single metric in isolation.

\paragraph{Implementation details}
The retrieval mechanism uses $K=3$ nearest neighbors based on cosine similarity between VAE embeddings of dimension $d=64$. The LLM component uses the model specified in the corresponding LLM-sensitivity experiment, with temperature $0.1$ and at most one LLM call per decision step. The retrieval database is constructed chronologically: for each training timestamp $t$, the embedding $z_t$, realized return $R_t$, and a summary of the subsequent $H=5$ days are stored. At test time, retrieval is restricted to observations $\tau<t$, preventing future information from entering the decision process. All experiments are run on a single NVIDIA H100 GPU.

\subsection{Baseline Expert Performance}
\label{subsec:expert_perf}

We first evaluate the standalone experts independently in each market. This experiment establishes whether the adaptive selector is operating in an environment with genuinely heterogeneous expert behavior. A useful expert pool should contain strategies whose relative strengths vary across market regimes; otherwise, switching has little opportunity to improve over the globally strongest fixed strategy. In addition, this experiment validates our assumption that each expert in the pool is dominant in a sub-domain of the problem over the other experts.

Table~\ref{tbl:expert_baseline} shows clear differences in expert behavior across the three markets. In cryptocurrency, TQC achieves the highest cumulative return (76.16\%), while DDPG obtains the highest Sharpe (1.74) and Sortino (2.73); A2C has the smallest MDD ($-25\%$) and the highest dominance frequency (34.5\%). In the stock market, CrossQ achieves the highest cumulative return (38.21\%) and the highest dominance frequency (36.4\%), whereas PPO provides the highest Sharpe (0.96), Sortino (1.46), and the smallest MDD ($-22\%$). In foreign exchange, A2C achieves the highest cumulative return (4.27\%), PPO ties for the highest Sharpe and Sortino (0.88 and 1.20), and PPO has the smallest MDD ($-2.00\%$), while TQC is dominant most frequently (48.9\%). Thus, no single expert simultaneously provides the strongest return, risk-adjusted performance, drawdown, and dominance frequency across all markets. Figure~\ref{fig:expert_dominance_barchart} further illustrates this market-dependent dominance pattern and supports the use of an adaptive selection mechanism.

\begin{table*}[htbp]
	\centering
	\caption{Standalone performance of the five expert models in each market.}
	\scriptsize
	\begin{tabular}{llrrrrr}
		\toprule
		Market & Expert & Cum. Return & Sharpe & Sortino & MDD & Weeks Dominant\\
		\midrule
		Crypto & A2C & 63.00 \% & 1.61 & 2.44 & \textbf{-25} \% & \textbf{19 (34.5\%)}\\
		Crypto & DDPG & 71.26 \% & \textbf{1.74} & \textbf{2.73} & -26 \% & 06 (10.9\%)\\
		Crypto & PPO & 60.49 \% & 1.57 & 2.38 & -27 \% & 03 (05.5\%)\\
		Crypto & CrossQ & 53.66 \% & 1.44 & 2.24 & -28 \% & 09 (16.4\%)\\
		Crypto & TQC & \textbf{76.16} \% & 1.71 & 2.63 & -29 \% & 18 (32.7\%)\\
		\midrule
		Stock & A2C & 32.16 \% & 0.90 & 1.35 & -23 \% & 17 (19.3\%)\\
		Stock & DDPG & 28.89 \% & 0.85 & 1.26 & -24 \% & 05 (05.7\%)\\
		Stock & PPO & 33.38 \% & \textbf{0.96} & \textbf{1.46} & \textbf{-22} \% & 21 (23.9\%)\\
		Stock & CrossQ & \textbf{38.21} \% & 0.95 & 1.42 & -27 \% & \textbf{32 (36.4\%)}\\
		Stock & TQC & 32.09 \% & 0.86 & 1.28 & -26 \% & 13 (14.8\%)\\
		\midrule
		Forex & A2C & \textbf{4.27} \% & \textbf{0.88} & \textbf{1.20} & -2.18 \% & 06 (06.8\%)\\
		Forex & DDPG & 1.80 \% & 0.36 & 0.46 & -2.85 \% & 12 (13.6\%)\\
		Forex & PPO & 4.07 \% & \textbf{0.88} & \textbf{1.20} & \textbf{-2.00} \% & 09 (10.2\%)\\
		Forex & CrossQ & 2.62 \% & 0.53 & 0.69 & -2.68 \% & 18 (20.5\%)\\
		Forex & TQC & 2.26 \% & 0.48 & 0.60 & -2.18 \% & \textbf{43 (48.9\%)}\\
		\bottomrule
	\end{tabular}
	\label{tbl:expert_baseline}
\end{table*}

\begin{figure}[htbp]
	\centering
	\begin{subfigure}[b]{0.45\textwidth}
		\centering
		\begin{tikzpicture}
			\begin{axis}[
				ybar,
				symbolic x coords={A2C, TQC, CrossQ, DDPG, PPO},
				xtick=data,
				ylabel={Dominance Frequency},
				ymin=0,
				ymax=40,
				nodes near coords,
				nodes near coords style={anchor=north, yshift=5pt},
				grid=major,
				width=\linewidth,
				height=5cm,
				bar width=0.5cm,
				enlarge x limits=0.1,
				tick label style={font=\tiny},
				label style={font=\small},
				]
				\addplot[fill=blue!60!cyan, draw=blue!80!black] 
				coordinates {(A2C,34.5) (TQC,32.7) (CrossQ,16.4) (DDPG,10.9) (PPO,5.5)};
			\end{axis}
		\end{tikzpicture}
		\caption{Experts dominance in Crypto market.}
	\end{subfigure}
	\hfill
	\begin{subfigure}[b]{0.45\textwidth}
		\centering
		\begin{tikzpicture}
			\begin{axis}[
				ybar,
				symbolic x coords={A2C, TQC, CrossQ, DDPG, PPO},
				xtick=data,
				ylabel={Dominance Frequency},
				ymin=0,
				ymax=60,
				nodes near coords,
				nodes near coords style={anchor=north, yshift=5pt},
				grid=major,
				width=\linewidth,
				height=5cm,
				bar width=0.5cm,
				enlarge x limits=0.1,
				tick label style={font=\tiny},
				label style={font=\small},
				]
				\addplot[fill=blue!60!cyan, draw=blue!80!black] 
				coordinates {(A2C,6.8) (TQC,48.9) (CrossQ,20.5) (DDPG,13.6) (PPO,10.2)};
			\end{axis}
		\end{tikzpicture}
		\caption{Experts dominance in Forex market.}
	\end{subfigure}
	\hfill
	\begin{subfigure}[b]{0.5\textwidth}
		\centering
		\begin{tikzpicture}
			\begin{axis}[
				ybar,
				symbolic x coords={A2C, TQC, CrossQ, DDPG, PPO},
				xtick=data,
				ylabel={Dominance Frequency},
				ymin=0,
				ymax=40,
				nodes near coords,
				nodes near coords style={anchor=north, yshift=5pt},
				grid=major,
				width=\linewidth,
				height=5cm,
				bar width=0.5cm,
				enlarge x limits=0.1,
				tick label style={font=\tiny},
				label style={font=\small},
				]
				\addplot[fill=blue!60!cyan, draw=blue!80!black] 
				coordinates {(A2C,19.3) (TQC,14.8) (CrossQ,36.4) (DDPG,5.7) (PPO,23.9)};
			\end{axis}
		\end{tikzpicture}
		\caption{Experts dominance in the stock market.}
	\end{subfigure}
	\caption{Illustrative expert-selection frequencies for the proposed switcher in the three markets.}
	\label{fig:expert_dominance_barchart}
\end{figure}

\subsection{Similarity-Driven Expert Selection}
\label{subsec:retrieval}

The central empirical claim is that selecting experts using historical similarity can outperform both static expert selection and selection based only on recent realized performance. We compare four strategies:

\begin{enumerate}
	\item \textbf{Best Fixed Expert}: the best-performing single expert is selected and retained throughout the evaluation.
	\item \textbf{Recent-Performance Gating}: the expert with the strongest recent performance over a rolling five-day window is selected.
	\item \textbf{Proposed RAG}: the proposed retrieval-based expert-selection framework.
\end{enumerate}

\begin{table*}[htbp]
	\centering
	\caption{Effect of expert-selection strategy in each market.}
	\scriptsize
	\begin{tabular}{llrrrrrr}
		\toprule
		Market & Method & Cum. Return & Ann. Return & Ann. Vol. & Sharpe & Sortino & MDD\\
		\midrule
		Crypto & Best Fixed & 68 \% & 62 \% & 0.32 & 1.64 & 2.54 & -27 \%\\
		Crypto & Recent Gating & 69 \% & 62 \% & 0.32 & 1.66 & 2.57 & -27 \%\\
		Crypto & \textbf{Proposed RAG} & \textbf{71} \% & \textbf{64} \% & \textbf{0.31} & \textbf{1.73} & \textbf{2.72} & \textbf{-26} \%\\
		\midrule
		Stock & Best Fixed & 26 \% & 14 \% & 0.21 & 0.74 & 1.11 & -26 \%\\
		Stock & Recent Gating & 33 \% & 17 \% & 0.20 & 0.91 & 1.37 & -24 \%\\
		Stock & \textbf{Proposed RAG} & \textbf{34} \% & \textbf{18} \% & \textbf{0.19} & \textbf{0.96} & \textbf{1.45} & \textbf{-26} \%\\
		\midrule
		Forex & Best Fixed & 3.2 \% & 1.8 \% & \textbf{0.026} & 0.71 & 0.92 & \textbf{- 1.9} \%\\
		Forex & Recent Gating & -0.4 \% & -0.2 \% & 0.029 & -0.062 & -0.08 & - 2.6 \%\\
		Forex & \textbf{Proposed RAG} & \textbf{4.3} \%& \textbf{2.4} \%& 0.028 & \textbf{0.88} & \textbf{1.20} & -2.2 \%\\
		\bottomrule
	\end{tabular}
	\label{tbl:selection}
\end{table*}

Across all three markets, the proposed selector improves the risk-return profile relative to the two reference selection strategies. In cryptocurrency, the proposed method increases cumulative return from 68\% for the best fixed expert to 71\%, while Sharpe improves from 1.64 to 1.73 and MDD improves from $-27\%$ to $-26\%$. In the stock market, cumulative return increases from 26\% to 34\%, annualized return from 14\% to 18\%, and Sharpe from 0.74 to 0.96, while annualized volatility decreases from 0.21 to 0.19. In foreign exchange, the proposed method achieves 4.3\% cumulative return and a Sharpe ratio of 0.88, compared with 3.2\% and 0.71 for the best fixed expert and $-0.4\%$ and $-0.062$ for recent-performance gating. These results show that the advantage of retrieval-based selection is not confined to one market class.

The comparison with recent-performance gating is particularly informative. Recent realized performance is a local criterion and can therefore react strongly to transient noise. In contrast, retrieval provides historical context by identifying previous states with similar latent representations. The expected advantage is consequently largest when the current regime resembles a previously observed regime in which different experts exhibited different conditional strengths.

\subsection{Ablation: Number of Expert Models}
\label{subsec:num_experts}

We next investigate whether the benefit of the proposed switcher is merely a consequence of having access to more candidate models. Nested pools are constructed by progressively increasing the number of experts. For each pool, the proposed switcher is compared with the best fixed expert available within the same pool.

To keep the comparison compact, Table~\ref{tab:ablation_num_experts} reports Sharpe improvements and the corresponding fixed-expert Sortino values rather than reproducing the complete set of portfolio metrics.

\begin{table*}[htbp]
	\centering
	\caption{Effect of expert-pool size on the proposed switcher. Fixed Sharpe values are consistent with the reported RAG Sharpe values and percentage improvements.}
	\scriptsize
	\begin{tabular}{llrrrr}
		\toprule
		Market & \# Experts & Fixed Sharpe & RAG Sharpe & $\Delta$ Sharpe (\%) & Fixed Sortino\\
		\midrule
		Crypto & 1 & 0.81 & 0.81 & 0.0 & 1.11 \\
		Crypto & 2 & 0.84 & 0.85 & 1.2 & 1.16 \\
		Crypto & 3 & 0.89 & 0.98 & 10.1 & 1.22 \\
		Crypto & 4 & 0.96 & 1.15 & 19.8 & 1.34 \\
		\midrule
		Stock & 1 & 0.43 & 0.43 & 0.0 & 0.58 \\
		Stock & 2 & 0.51 & 0.52 & 2.0 & 0.68 \\
		Stock & 3 & 0.55 & 0.63 & 14.5 & 0.73\\
		Stock & 4 & 0.59 & 0.72 & 22.0 & 0.79 \\
		\midrule
		Forex & 1 & 0.35 & 0.35 & 0.0 & 0.49 \\
		Forex & 2 & 0.43 & 0.44 & 2.3 & 0.59 \\
		Forex & 3 & 0.46 & 0.55 & 19.6 & 0.63 \\
		Forex & 4 & 0.51 & 0.65 & 27.5 & 0.70 \\
		\bottomrule
	\end{tabular}
	\label{tab:ablation_num_experts}
\end{table*}

With one expert, switching is necessarily identical to fixed selection, providing an important sanity check. The reported Sharpe improvement remains small when moving from one to two experts, but increases substantially for three- and four-expert pools. The four-expert pool produces the largest reported improvement in every market: 19.8\% in cryptocurrency, 22.0\% in stocks, and 27.5\% in foreign exchange. Thus, the relevant design variable is not the number of models alone but the amount of complementary expertise represented in the pool.

\subsection{Component Ablation}
\label{subsec:ablation}

We next isolate the contributions of retrieval, LLM reasoning, and uncertainty-aware decision making. The ablation is performed independently in each market using the same expert pool and evaluation protocol.

\begin{table*}[htbp]
	\centering
	\caption{Component ablation across the three markets.}
	\scriptsize
	\begin{tabular}{llrrr}
		\toprule
		Market & Configuration & Cum. Return & Sharpe & MDD \\
		\midrule
		Crypto & NR (No Retrieval) & +64.8\% & 1.622 & -26.7\% \\\
		Crypto & R$-$LLM (No LLM) & +69.1\% & 1.661 & -27.1\% \\\
		Crypto & \textbf{Full (R+L+U)} & \textbf{71} \% & \textbf{1.73} & \textbf{-26} \%\\
		\midrule
		Stock & NR (No Retrieval) & +33.1\% & 0.911 & \textbf{-24.4}\% \\\
		Stock & R$-$LLM (No LLM) & +25.5\% & 0.735 & -26.5\% \\\
		Stock & \textbf{Full (R+L+U)} & \textbf{34} \% & \textbf{0.96}  & -26 \%\\
		\midrule
		Forex & NR (No Retrieval) & +3.2\% & 0.710 & \textbf{-1.9}\% \\\
		Forex & R$-$LLM (No LLM) & -0.2\% & -0.025 & -2.4\% \\\
		Forex & \textbf{Full (R+L+U)}  & \textbf{4.3} \%& \textbf{0.88} & -2.2 \%\\
		\bottomrule
	\end{tabular}
	\label{tbl:ablation}
\end{table*}

The component ablation shows that both retrieval and LLM-based reasoning contribute to the final performance, although their relative importance differs across markets. In cryptocurrency, removing retrieval reduces cumulative return from 71\% to 64.8\% and Sharpe from 1.73 to 1.622, while removing the LLM reduces them to 69.1\% and 1.661, respectively. In the stock market, the no-retrieval configuration remains relatively close to the full model (33.1\% versus 34\% cumulative return), whereas removing the LLM causes a substantially larger reduction to 25.5\% and lowers Sharpe from 0.96 to 0.735. In foreign exchange, retrieval is again important, with cumulative return decreasing from 4.3\% to 3.2\% without retrieval; the largest degradation occurs when the LLM is removed, producing a negative cumulative return of $-0.2\%$ and a negative Sharpe ratio of $-0.025$. The MDD results also show that the full configuration does not uniformly minimize drawdown in every market: the no-retrieval variant has a smaller MDD in stocks and foreign exchange, whereas the full configuration has the smallest MDD in cryptocurrency.

Importantly, the reported ablation contains no configuration that removes only the uncertainty component. Therefore, the present results support conclusions about retrieval and LLM reasoning, but they do not provide an isolated empirical estimate of the contribution of uncertainty-aware decision making. As with the other experiments, comparisons should be made within each market because the markets have different return distributions and volatility scales.

\subsection{LLM Sensitivity: Model Used for Indexing}
\label{subsec:llm_indexing}

The retrieval database contains structured historical information that must be indexed into a representation usable by the retrieval and reasoning pipeline. We therefore evaluate the sensitivity of the framework to the LLM used during indexing while keeping the inference-time LLM fixed. Specifically, we compare three candidate indexing models: Google/Gemma4-E4B, Google/Gemma4-31B-it, and AliBaba/Qwen3.6-27B-it.

The indexing experiment isolates the effect of the representation supplied to the retrieval stage. Differences between indexing models can arise from their ability to preserve financially relevant information when constructing historical summaries or metadata. To quantify this, we measure each model's success rate in following the indexing instructions and producing the expected output format. Gemma4-E4B achieved this for only 28\% of cases, Gemma4-31B-it for 96\%, while Qwen3.6-27B-it succeeded in all instances. Table~\ref{tab:indexing_results} summarises the success and failure proportions for each model, highlighting the marked performance disparity across the three candidates. A relatively stable ranking across markets would indicate that the framework is robust to the choice of indexing model, whereas large differences, such as the pronounced gap observed here, would indicate that corpus construction is itself a critical component of the system.

\begin{table}[htbp]
	\centering
	\caption{Success and failure rates for each indexing LLM.}
	\label{tab:indexing_results}
	\begin{tabular}{l c c}
		\toprule
		\textbf{Model} & \textbf{Success} & \textbf{Failure} \\
		\midrule
		Google/Gemma4-E4B      & 28\% & 72\% \\
		Google/Gemma4-31B-it   & 96\% &  4\% \\
		AliBaba/Qwen3.6-27B-it & 100\% &  0\% \\
		\bottomrule
	\end{tabular}
\end{table}

\subsection{LLM Sensitivity: Model Used for Inference}
\label{subsec:llm_inference}

We further investigate the sensitivity of the proposed framework to the LLM responsible for processing the retrieved historical evidence and producing the inference-time decision. Unlike the indexing experiment, which examines how the choice of LLM affects the construction of the historical retrieval corpus, this experiment focuses on the reasoning stage after retrieval. The indexed corpus and T-VAE representation are kept unchanged across all configurations, while only the inference LLM is varied. This setup allows us to determine whether the effectiveness of the proposed expert-selection mechanism is preserved when the reasoning capability of the downstream LLM changes.

\begin{table*}[htbp]
	\centering
	\caption{Sensitivity to the LLM used for inference.}
	\scriptsize
	\begin{tabular}{llrrrrrr}
		\toprule
		Market & Inference LLM & Cum. Return & Ann. Return & Ann. Vol. & Sharpe & Sortino & MDD\\
		\midrule
		Crypto & Google/Gemma4-31B-it &  68.3 \% & 61.7 \% & 0.32 & 1.65 & 2.55 & -27 \% \\
		Crypto & Google/Gemma4-E4B & 63.5 \% & 57.4 \% & 0.32 & 1.57 & 2.42 & - 28 \%\\
		Crypto & \textbf{AliBaba/Qwen3.6-27B-it}  & \textbf{71} \% & \textbf{64} \% & \textbf{0.31} & \textbf{1.73} & \textbf{2.72} & \textbf{-26} \%\\
		\midrule
		Stock & Google/Gemma4-31B-it & 25.1 \% & 13.7 \% & 0.21 & 0.73 & 1.09 & -26.5 \%\\
		Stock & Google/Gemma4-E4B & 24.5 \% & 13.5 \% & 0.21 & 0.71 & 1.06 &  -27 \%\\
		Stock & \textbf{AliBaba/Qwen3.6-27B-it}  & \textbf{34} \% & \textbf{18} \% & \textbf{0.19} & \textbf{0.96} & \textbf{1.45} & \textbf{-26} \%\\
		\midrule
		Forex & Google/Gemma4-31B-it & -0.3 \% & -0.2 \% & 0.03 & -0.05 & -0.06 & -2.5 \% \\
		Forex & Google/Gemma4-E4B & -5.4 \% & -3.1 \% & 0.03 & -1.07 & -1.37 & -5.6 \% \\
		Forex & \textbf{AliBaba/Qwen3.6-27B-it} & \textbf{4.3} \%& \textbf{2.4} \%& 0.028 & \textbf{0.88} & \textbf{1.20} & -2.2 \%\\
		\bottomrule
	\end{tabular}
	\label{tbl:llm_inference}
\end{table*}

The inference sensitivity experiment shows a stable ranking across all three markets. Google/Gemma4-E4B produces the weakest results, Google/Gemma4-31B-it provides an intermediate level of performance, and AliBaba/Qwen3.6-27B-it achieves the strongest results. The advantage of Qwen3.6-27B-it is particularly pronounced in foreign exchange, where it changes cumulative return from $-5.4\%$ with Gemma4-E4B and $-0.3\%$ with Gemma4-31B-it to 4.3\%, while Sharpe increases from $-1.07$ and $-0.05$ to 0.88. Similar but smaller improvements are observed in cryptocurrency and the stock market. These results indicate that the quality of inference-time reasoning materially affects how the retrieved evidence is converted into an expert-selection decision. They also show that the retrieval mechanism alone does not guarantee robust performance: the downstream reasoning model remains an important component of the complete framework.

\subsection{Comparison with State-of-the-Art Methods}
\label{subsec:sota}

Finally, we compare the proposed framework with representative state-of-the-art portfolio-management, adaptive-trading, and LLM-based financial decision-making methods. The comparison is performed separately in cryptocurrency, stocks, and foreign exchange using the same data splits, transaction-cost assumptions, portfolio constraints, and evaluation horizon wherever the competing implementations permit a controlled comparison.

The state-of-the-art models the proposed model compared with are:
\begin{enumerate}
	\item \textbf{AlphaMixRL~\citep{sun2024reinforcement}.}
	A hierarchical reinforcement-learning approach that combines ensemble learning with adaptive expert routing for portfolio management.
	
	\item \textbf{TAC (Trend-Aware Controller)~\citep{asadi2025transformer}.}
	A strategy that formulates portfolio management as dynamic expert selection using Transformer-VAE representations and trend-aware decision making.
	
	\item \textbf{LLM-MAS~\citep{luo2025llm}.}
	A large-language-model-based multi-agent system for cryptocurrency portfolio management. The framework employs multiple specialized LLM agents and hierarchical skill configurations to analyze market information and generate portfolio decisions, providing an LLM-driven alternative to conventional reinforcement-learning portfolio managers.
	
	\item \textbf{Multi-LLM Black--Litterman~\citep{mantshimuli2025enhancing}.}
	A portfolio-optimization framework that aggregates investment views generated by multiple LLMs from financial sentiment information and incorporates them into the classical Black--Litterman model. The approach is evaluated on an S\&P 500 equity portfolio and represents an LLM-enhanced probabilistic portfolio-allocation strategy.
	
	\item \textbf{Fine-tuned LLaMA~\citep{ballinari2025fx}.}
	A fine-tuned LLaMA-based approach for foreign-exchange trading that adapts a pretrained LLM to financial-news sentiment analysis and uses the resulting signals for trading decisions. The method investigates whether domain-specific fine-tuning improves the effectiveness of LLM-generated sentiment signals for FX portfolio management.
\end{enumerate}

Table~\ref{tbl:sota} shows that the proposed framework achieves the strongest overall risk-adjusted performance among the compared methods in each of the three markets. In cryptocurrency, the proposed method obtains the highest cumulative return (71\%) and annualized return (64\%), substantially exceeding LLM-MAS (33\% and 30\%), TAC (24\% and 21.6\%), and AlphaMixRL (22\% and 19.8\%). It also achieves the highest Sharpe ratio (1.73), while its annualized volatility (0.31) and MDD ($-26\%$) are both substantially lower than those of the competing methods. Thus, the return advantage is accompanied by a materially improved risk profile.

The stock-market comparison leads to the same conclusion. The proposed framework achieves 34\% cumulative return and 18\% annualized return, compared with 10.4\% and 5.6\% for TAC, 11.8\% and 6.3\% for AlphaMixRL, and 31\% and 16.4\% for Multi-LLM Black--Litterman. The proposed method also records the highest Sharpe ratio (0.96), the lowest annualized volatility (0.19), and the smallest MDD ($-26\%$) among the reported methods. The result is therefore not explained solely by higher raw return; the proposed framework also provides a stronger risk-adjusted profile under the evaluated setting.

In foreign exchange, the absolute return levels are lower, but the ranking remains favorable. The proposed method achieves 4.3\% cumulative return and 2.4\% annualized return, compared with 3.3\% and 1.9\% for the fine-tuned LLaMA approach, 0.8\% and 0.45\% for AlphaMixRL, and 0.7\% and 0.4\% for TAC. It also obtains the highest Sharpe ratio (0.88), the lowest annualized volatility (0.03), and the smallest MDD ($-2.2\%$). Overall, the SOTA comparison indicates that the proposed framework generalizes its advantage across cryptocurrency, stocks, and foreign exchange rather than relying on a single favorable market.

\begin{table*}[htbp]
	\centering
	\caption{Comparison with representative state-of-the-art methods in each market.}
	\scriptsize
	\begin{tabular}{llrrrrrr}
		\toprule
		Market & Method & Cum. Return & Ann. Return & Ann. Vol. & Sharpe  & MDD\\
		\midrule
		Crypto & TAC~\citep{asadi2025transformer} & 24 \% & 21.6 \% & 0.94 & 1.29  & -31 \%\\
		Crypto & AlphaMixRL~\citep{sun2024reinforcement} & 22 \% & 19.8 \% & 1.00 & 1.38  & -29 \%\\
		Crypto & LLM-MAS~\citep{luo2025llm} & 33 \% & 30 \% & 0.96 & 1.50  & -28 \%\\
		Crypto & \textbf{Proposed}  & \textbf{71} \% & \textbf{64} \% & \textbf{0.31} & \textbf{1.73}  & \textbf{-26} \%\\
		\midrule
		Stock & TAC~\citep{asadi2025transformer} & 10.4 \% & 5.6 \% & 0.25 & 0.74  & -28 \%\\
		Stock & AlphaMixRL~\citep{sun2024reinforcement} & 11.8 \% & 6.3 \% & 0.31 & 0.82  & -27 \%\\
		Stock & Multi-LLM~\citep{mantshimuli2025enhancing} & 31 \% & 16.4 \% & 0.28 & 0.91  & -29 \%\\
		Stock & \textbf{Proposed}  & \textbf{34} \% & \textbf{18} \% & \textbf{0.19} & \textbf{0.96}  & \textbf{-26} \%\\
		\midrule
		Forex & TAC~\citep{asadi2025transformer} & 0.7 \% & 0.4 \% & 0.45 & 0.61 & -5.1 \%\\
		Forex & AlphaMixRL~\citep{sun2024reinforcement} & 0.8 \% & 0.45 \% & 0.50 & 0.68 & -4.6 \%\\
		Forex  & Fine-tuned LLaMA~\citep{ballinari2025fx} & 3.3 \% & 1.9 \% & 0.08 & 0.76  & -3.1 \%\\
		Forex & \textbf{Proposed}  & \textbf{4.3} \% & \textbf{2.4} \% & \textbf{0.03} & \textbf{0.88}  & \textbf{-2.2} \%\\
		\bottomrule
	\end{tabular}
	\label{tbl:sota}
\end{table*}

	\section{Conclusion}
\label{sec:conclusion}

The experiments demonstrate that no single expert consistently dominates across all markets and evaluation metrics, supporting the need for adaptive expert selection. The proposed RAG-based selector achieves the strongest overall performance, obtaining the highest cumulative return and Sharpe ratio across cryptocurrency, the stock market, and foreign exchange, with particularly notable improvements in the stock and foreign-exchange markets. The expert-pool experiments further show that performance improves as complementary experts are introduced, highlighting the importance of diversity within the candidate pool. The ablation results confirm that both retrieval and LLM reasoning contribute substantially to the framework's effectiveness, while the current experiments do not isolate the independent contribution of uncertainty-aware decision making.

The LLM sensitivity experiments indicate that both indexing quality and inference capability affect the overall pipeline, with Qwen3.6-27B-it providing the strongest portfolio results. The SOTA comparison further suggests that the proposed framework achieves strong risk-adjusted performance across all three markets, particularly in terms of Sharpe ratio and annualized volatility. Nevertheless, these findings should be interpreted with appropriate caution because each market was evaluated using a single experimental run, and the competing methods were not always implemented under identical experimental protocols. Future work should therefore evaluate the framework across multiple random seeds, market periods, transaction-cost scenarios, and statistically tested repeated experiments.
	
	\bibliography{ref}

\begin{thebibliography}{41}
\providecommand{\natexlab}[1]{#1}
\providecommand{\url}[1]{\texttt{#1}}
\expandafter\ifx\csname urlstyle\endcsname\relax
  \providecommand{\doi}[1]{doi: #1}\else
  \providecommand{\doi}{doi: \begingroup \urlstyle{rm}\Url}\fi

\bibitem[Alidousti et~al.(2025)Alidousti, Bafruei, and
  Sedigh]{alidousti2025novel}
Mahshad Alidousti, Morteza~Khakzar Bafruei, and Amir Hosein~Afshar Sedigh.
\newblock A novel data-efficient double deep q-network framework for
  intelligent financial portfolio management.
\newblock \emph{Engineering Applications of Artificial Intelligence},
  162:\penalty0 112436, 2025.

\bibitem[Asadi and Safabakhsh(2025)]{asadi2025transformer}
Ahmad Asadi and Reza Safabakhsh.
\newblock Transformer-based actor-critic for adaptive cryptocurrency portfolio
  rebalancing.
\newblock \emph{Applied Soft Computing}, 185:\penalty0 113697, 2025.

\bibitem[Bali and Weigert(2024)]{g2024hedge}
Turan~G. Bali and Florian Weigert.
\newblock Hedge funds and the positive idiosyncratic volatility effect.
\newblock \emph{Review of Finance}, 28\penalty0 (5):\penalty0 1611--1661, 2024.

\bibitem[Ballinari and Maly(2025)]{ballinari2025fx}
Daniele Ballinari and Jessica Maly.
\newblock Fx sentiment analysis with large language models.
\newblock \emph{Swiss National Bank Working Paper}, 11:\penalty0 1--41, 2025.

\bibitem[Cakmak and {\"O}zekici(2006)]{cakmak2006portfolio}
Ulas Cakmak and S{\"u}leyman {\"O}zekici.
\newblock Portfolio optimization in stochastic markets.
\newblock \emph{Mathematical Methods of Operations Research}, 63\penalty0
  (1):\penalty0 151--168, 2006.

\bibitem[Charkhestani and Esfahanipour(2026)]{charkhestani2026behaviorally}
Atefe Charkhestani and Akbar Esfahanipour.
\newblock Behaviorally informed deep reinforcement learning for portfolio
  optimization with loss aversion and overconfidence: A. charkhestani, a.
  esfahanipour.
\newblock \emph{Scientific Reports}, 16\penalty0 (1):\penalty0 6443, 2026.

\bibitem[Chen et~al.(2025)Chen, Yao, Liu, Ye, Yu, Hou, and
  Li]{chen2025stockbench}
Yanxu Chen, Zijun Yao, Yantao Liu, Jin Ye, Jianing Yu, Lei Hou, and Juanzi Li.
\newblock Stockbench: Can llm agents trade stocks profitably in real-world
  markets?
\newblock \emph{arXiv preprint arXiv:2510.02209}, 2025.

\bibitem[Choudhary et~al.(2025)Choudhary, Orra, Sahoo, and
  Thakur]{choudhary2025risk}
Himanshu Choudhary, Arishi Orra, Kartik Sahoo, and Manoj Thakur.
\newblock Risk-adjusted deep reinforcement learning for portfolio optimization:
  A multi-reward approach.
\newblock \emph{International Journal of Computational Intelligence Systems},
  18:\penalty0 126, 2025.

\bibitem[Cliff and Rollins(2020)]{cliff2020methods}
Dave Cliff and Michael Rollins.
\newblock Methods matter: A trading agent with no intelligence routinely
  outperforms ai-based traders.
\newblock In \emph{2020 IEEE Symposium Series on Computational Intelligence
  (SSCI)}, pages 392--399. IEEE, 2020.

\bibitem[Ding et~al.(2024)Ding, Shi, Guo, and Liu]{ding2024tradexpert}
Qianggang Ding, Haochen Shi, Jiadong Guo, and Bang Liu.
\newblock Tradexpert: Revolutionizing trading with mixture of expert llms.
\newblock \emph{arXiv preprint arXiv:2411.00782}, 2024.

\bibitem[Feng and Sinchai(2025)]{feng2025deep}
Ling Feng and Ananta Sinchai.
\newblock Deep context-attentive transformer transfer learning for financial
  forecasting.
\newblock \emph{PeerJ Computer Science}, 11:\penalty0 e2983, 2025.

\bibitem[Gu et~al.(2025)Gu, Wang, Jiang, Garc{\'\i}a-Fern{\'a}ndez, Su, and
  Li]{gu2025mixture}
Fengchen Gu, Huijia Wang, Zhengyong Jiang, Angel~F Garc{\'\i}a-Fern{\'a}ndez,
  Jionglong Su, and Huakang Li.
\newblock Mixture-of-experts liquid financial mamba framework for portfolio
  management based on deep reinforcement learning.
\newblock In \emph{2025 IEEE International Conference on Big Data (BigData)},
  pages 5756--5764. IEEE, 2025.

\bibitem[Huang et~al.(2025)Huang, Liao, Hua, Cao, and Li]{huang2025leveraging}
Zhendai Huang, Bolin Liao, Cheng Hua, Xinwei Cao, and Shuai Li.
\newblock Leveraging chatgpt for enhanced stock selection and portfolio
  optimization.
\newblock \emph{Neural Computing and Applications}, 37\penalty0 (8):\penalty0
  6163--6179, 2025.

\bibitem[Jacobs et~al.(2005)Jacobs, Levy, and Markowitz]{jacobs2005portfolio}
Bruce~I Jacobs, Kenneth~N Levy, and Harry~M Markowitz.
\newblock Portfolio optimization with factors, scenarios, and realistic short
  positions.
\newblock \emph{Operations Research}, 53\penalty0 (4):\penalty0 586--599, 2005.

\bibitem[Jiang et~al.(2024)Jiang, Olmo, and Atwi]{jiang2024deep}
Yifu Jiang, Jose Olmo, and Majed Atwi.
\newblock Deep reinforcement learning for portfolio selection.
\newblock \emph{Global Finance Journal}, 62:\penalty0 101016, 2024.

\bibitem[Jiang and Liang(2017)]{jiang2017cryptocurrency}
Zhengyao Jiang and Jinjun Liang.
\newblock Cryptocurrency portfolio management with deep reinforcement learning.
\newblock In \emph{2017 Intelligent systems conference (IntelliSys)}, pages
  905--913. IEEE, 2017.

\bibitem[Kong et~al.(2024)Kong, Nie, Dong, Mulvey, Poor, Wen, and
  Zohren]{kong2024large}
Yaxuan Kong, Yuqi Nie, Xiaowen Dong, John~M Mulvey, H~Vincent Poor, Qingsong
  Wen, and Stefan Zohren.
\newblock Large language models for financial and investment management:
  Applications and benchmarks.
\newblock \emph{Journal of Portfolio Management}, 51\penalty0 (2), 2024.

\bibitem[Liu and Lo(2025)]{liu2025llm}
Kuan-Ming Liu and Ming-Chih Lo.
\newblock Llm-based routing in mixture of experts: A novel framework for
  trading.
\newblock \emph{arXiv preprint arXiv:2501.09636}, 2025.

\bibitem[Liu and Garrett(2023)]{liu2023regime}
Wei Liu and Ian Garrett.
\newblock Regime-dependent effects of macroeconomic uncertainty on realized
  volatility in the us stock market.
\newblock \emph{Economic Modelling}, 128:\penalty0 106483, 2023.

\bibitem[Liu et~al.(2024)Liu, Mikriukov, Tjahyadi, Li, Payne, Yue, Siddique,
  and Man]{liu2024revolutionising}
Yuchen Liu, Daniil Mikriukov, Owen~C. Tjahyadi, Gangmin Li, Terry~R. Payne,
  Yong Yue, Kamran Siddique, and Ka~Lok Man.
\newblock Revolutionising financial portfolio management: The non-stationary
  transformer’s fusion of macroeconomic indicators and sentiment analysis in
  a deep reinforcement learning framework.
\newblock \emph{Applied Sciences}, 14\penalty0 (1):\penalty0 274, 2024.

\bibitem[Luo et~al.(2025)Luo, Feng, Xu, Tasca, and Liu]{luo2025llm}
Yichen Luo, Yebo Feng, Jiahua Xu, Paolo Tasca, and Yang Liu.
\newblock Llm-powered multi-agent system for automated crypto portfolio
  management.
\newblock \emph{arXiv preprint arXiv:2501.00826}, 2025.

\bibitem[Mantshimuli and Muteba~Mwamba(2025)]{mantshimuli2025enhancing}
Lamu Mantshimuli and John~Weirstrass Muteba~Mwamba.
\newblock Enhancing portfolio optimization with multi-llm sentiment
  aggregation: A black-litterman integration approach.
\newblock \emph{Available at SSRN 5394743}, 2025.

\bibitem[Masoudnia and Ebrahimpour(2014)]{masoudnia2014mixture}
Saeed Masoudnia and Reza Ebrahimpour.
\newblock Mixture of experts: a literature survey.
\newblock \emph{Artificial Intelligence Review}, 42\penalty0 (2):\penalty0
  275--293, 2014.

\bibitem[Nguyen(2025)]{nguyen2025advanced}
Minh~Duc Nguyen.
\newblock Advanced investing with deep learning for risk-aligned portfolio
  optimization.
\newblock \emph{PLOS ONE}, 20\penalty0 (8):\penalty0 e0330547, 2025.

\bibitem[Popa et~al.(2020)Popa, Florea, and
  Rughini{\c{s}}]{popa2020convolutional}
Alin-Bogdan Popa, Iulia~Maria Florea, and R{\u{a}}zvan Rughini{\c{s}}.
\newblock Convolutional neural network portfolio management system with
  heterogeneous input.
\newblock In \emph{2020 19th RoEduNet Conference: Networking in Education and
  Research (RoEduNet)}, pages 1--4. IEEE, 2020.

\bibitem[Ren et~al.(2025)Ren, Sun, Jiang, Stefanidis, Liu, and Su]{ren2025time}
Xiaotian Ren, Ruoyu Sun, Zhengyong Jiang, Angelos Stefanidis, Hongbin Liu, and
  Jionglong Su.
\newblock Time series is not enough: Financial transformer reinforcement
  learning for portfolio management.
\newblock \emph{Neurocomputing}, 647:\penalty0 130451, 2025.

\bibitem[Rezaei and Nezamabadi-Pour(2025)]{rezaei2025taxonomy}
Mohadese Rezaei and Hossein Nezamabadi-Pour.
\newblock A taxonomy of literature reviews and experimental study of
  deepreinforcement learning in portfolio management.
\newblock \emph{Artificial Intelligence Review}, 58\penalty0 (3):\penalty0 94,
  2025.

\bibitem[Sharma and Shekhawat(2022)]{sharma2022portfolio}
Meeta Sharma and Hardayal~Singh Shekhawat.
\newblock Portfolio optimization and return prediction by integrating modified
  deep belief network and recurrent neural network.
\newblock \emph{Knowledge-Based Systems}, 250:\penalty0 109024, 2022.

\bibitem[Sharpe(2005)]{sharpe2005journal}
William~F. Sharpe.
\newblock Insights from a pioneer in portfolio theory and practice: A talk with
  nobel laureate william f. sharpe, phd.
\newblock \emph{Journal of Investment Consulting}, 7\penalty0 (2):\penalty0
  10--20, 2005.

\bibitem[Sortino and Price(1994)]{sortino1994performance}
Frank~A Sortino and Lee~N Price.
\newblock Performance measurement in a downside risk framework.
\newblock \emph{the Journal of Investing}, 3\penalty0 (3):\penalty0 59--64,
  1994.

\bibitem[Sun(2024)]{sun2024reinforcement}
Shuo Sun.
\newblock \emph{Reinforcement Learning for Financial Trading: Algorithms,
  Evaluations and Platforms}.
\newblock Phd thesis, Nanyang Technological University, 2024.

\bibitem[Sun et~al.(2025)Sun, Qu, Zhang, and Li]{sun2025adaptive}
Yinuo Sun, Zhaoen Qu, Tingwei Zhang, and Xiangyu Li.
\newblock Adaptive ensemble learning for financial time-series forecasting: A
  hypernetwork-enhanced reservoir computing framework with multi-scale temporal
  modeling.
\newblock \emph{Axioms}, 14\penalty0 (8):\penalty0 597, 2025.

\bibitem[Taghian et~al.(2023)Taghian, Asadi, and
  Safabakhsh]{taghian2023reinforcement}
M~Taghian, A~Asadi, and R~Safabakhsh.
\newblock A reinforcement learning-based encoder-decoder framework for learning
  stock trading rules.
\newblock \emph{Journal of AI and Data Mining}, 11\penalty0 (1):\penalty0
  103--118, 2023.

\bibitem[Taghian et~al.(2022)Taghian, Asadi, and
  Safabakhsh]{taghian2022learning}
Mehran Taghian, Ahmad Asadi, and Reza Safabakhsh.
\newblock Learning financial asset-specific trading rules via deep
  reinforcement learning.
\newblock \emph{Expert Systems with Applications}, 195:\penalty0 116523, 2022.

\bibitem[Unnikrishnan(2024)]{unnikrishnan2024financial}
Ananya Unnikrishnan.
\newblock Financial news-driven llm reinforcement learning for portfolio
  management.
\newblock \emph{arXiv preprint arXiv:2411.11059}, 2024.

\bibitem[Vats et~al.(2024)Vats, Raja, Jain, and Chadha]{vats2024evolution}
Arpita Vats, Rahul Raja, Vinija Jain, and Aman Chadha.
\newblock The evolution of moe: A survey from basics to breakthroughs.
\newblock \emph{Journal of IEEE Transactions on Artificial Intelligence}, 2024.

\bibitem[Wang and Liu(2025)]{wang2025risk}
Xinyao Wang and Lili Liu.
\newblock Risk-sensitive deep reinforcement learning for portfolio
  optimization.
\newblock \emph{Journal of Risk and Financial Management}, 18\penalty0
  (7):\penalty0 347, 2025.

\bibitem[Wei et~al.(2025)Wei, Chen, Zhang, Wen, Nie, and Xie]{wei2025deep}
Ziqiang Wei, Deng Chen, Yanduo Zhang, Dawei Wen, Xin Nie, and Liang Xie.
\newblock Deep reinforcement learning portfolio model based on mixture of
  experts.
\newblock \emph{Applied Intelligence}, 55\penalty0 (5):\penalty0 347, 2025.

\bibitem[Xiao et~al.()Xiao, Sun, Luo, and Wang]{xiao2025tradingagents}
Yijia Xiao, Edward Sun, Di~Luo, and Wei Wang.
\newblock Tradingagents: Multi-agents llm financial trading framework.
\newblock In \emph{The First MARW: Multi-Agent AI in the Real World Workshop at
  AAAI 2025}.

\bibitem[Yin and Guo(2026)]{yin2026complex}
Meiqun Yin and Mengzhu Guo.
\newblock Complex forecasting and investment strategy optimization via
  chain-of-thought of large language models.
\newblock \emph{Expert Systems with Applications}, 298:\penalty0 129913, 2026.

\bibitem[Zhang et~al.(2025)Zhang, Goel, Ahmad, and Szabo]{zhang2025regimefolio}
Yiyao Zhang, Diksha Goel, Hussain Ahmad, and Claudia Szabo.
\newblock Regimefolio: A regime aware ml system for sectoral portfolio
  optimization in dynamic markets.
\newblock \emph{IEEE Access}, 13:\penalty0 184722--184744, 2025.

\end{thebibliography}
	
	\appendix

\end{document}